\documentclass{article}

\usepackage[main, preprint]{neurips_2026}

\usepackage{etoc}
\usepackage[utf8]{inputenc} 
\usepackage[T1]{fontenc}    
\usepackage{hyperref}       
\usepackage{url}  
\usepackage{booktabs}       
\usepackage{nicefrac}       
\usepackage{microtype}      
\usepackage{xcolor}         
\RequirePackage{amsthm, amsmath, amsfonts, amssymb}
\RequirePackage{mathtools}
\usepackage{algorithm}
\usepackage{algpseudocode}
\usepackage{array}
\usepackage{tabularx}
\usepackage{multirow}
\usepackage{ragged2e}
\usepackage{xurl}
\newcommand{\bfm}[1]{\ensuremath{\mathbf{#1}}}
\def\ba{\bfm a}     \def\bA{\bfm A}          
\def\bb{\bfm b}     \def\bB{\bfm B}          
     \def\bC{\bfm C}          
     \def\bD{\bfm D}     \def\cD{{\cal  D}}     
\def\be{\bfm e}     \def\bE{\bfm E}     \def\cE{{\cal  E}}     
\def\bff{\bfm f}    \def\bF{\bfm F}     \def\cF{{\cal  F}}     
     \def\bG{\bfm G}          
\def\bh{\bfm h}     \def\bH{\bfm H}     \def\cH{{\cal  H}}     
     \def\bI{\bfm I}          
          \def\cJ{{\cal  J}}     
          \def\cK{{\cal  K}}     
          \def\cL{{\cal  L}}     
\def\bm{\bfm m}     \def\bM{\bfm M}          
          \def\cN{{\cal  N}}     
               
     \def\bP{\bfm P}          
\def\bq{\bfm q}     \def\bQ{\bfm Q}          
     \def\bR{\bfm R}          
     \def\bS{\bfm S}          
\def\bt{\bfm t}          \def\cT{{\cal  T}}     
\def\bu{{\bfm u}}     \def\bU{\bfm U}          
\def\bv{\bfm v}     \def\bV{\bfm V}          
\def\bw{\bfm w}     \def\bW{\bfm W}          
\def\bx{\bfm x}     \def\bX{\bfm X}          
               
\def\bz{\bfm z}     \def\bZ{\bfm Z}          
\def\bzero{\bfm 0}

\newcommand{\bfsym}[1]{\ensuremath{\boldsymbol{#1}}}
\def \balpha   {\bfsym{\alpha}}       
\def \bgamma   {\bfsym{\gamma}}       
\def \bepsilon {\bfsym{\epsilon}}     
         \def \btheta   {\bfsym{\theta}}
        
      \def \bmu      {\bfsym{\mu}}
          \def \bxi      {\bfsym{\xi}}
     \def \bpi      {\bfsym{\pi}}

\def \bphi     {\bfsym{\phi}}         
\def \bpsi     {\bfsym{\psi}}

\def \bGamma   {\bfsym{\Gamma}}       \def \bDelta   {\bfsym{\Delta}}
\def \bTheta   {\bfsym{\Theta}}       \def \bLambda  {\bfsym{\Lambda}}
\def \bXi      {\bfsym{\Xi}}          \def \bPi      {\bfsym{\Pi}}
\def \bSigma   {\bfsym{\Sigma}}

\DeclareMathOperator*{\argmin}{argmin}
\DeclareMathOperator*{\argmax}{argmax}

\DeclareMathOperator{\diag}{diag}
\newcommand{\E}{\mathbb{E}}
\DeclareMathOperator{\rank}{rank}

\DeclareMathOperator \Tr {\mathrm{tr}}

\def \bone   {\bfsym{1}}

\def \RR	{\mathbb{R}}

\def \PP {\mathbb{P}}

\def \cN {\mathcal{N}}
\def \cL {\mathcal{L}}

\def \cT {\mathcal{T}}

\def \cJ {\mathcal{J}}

\def \tilde{\widetilde}

\newcommand{\beq}  {\begin{equation}}
\newcommand{\eeq}  {\end{equation}}
\newcommand{\beqn} {\begin{eqnarray}}
\newcommand{\eeqn} {\end{eqnarray}}
\newcommand{\beqnn}{\begin{eqnarray*}}
\newcommand{\eeqnn}{\end{eqnarray*}}

\newtheorem{theorem}{Theorem}[section]
\newtheorem{assumption}[theorem]{Assumption}

\theoremstyle{remark}

\providecommand{\mc}[2]{#1\,(#2)}
\newcommand{\diseaselist}[1]{%
\begin{tabular}{@{}l@{}}
#1
\end{tabular}
}

\title{Hypergraph Representation Learning with Hyperlink Random Effects}

\author{%
  Zimeng Li
  \\
  Department of Statistics, Harvard University\\
  Cambridge, MA 02138 \\
  \texttt{zimengli@g.harvard.edu} \\
   \And
  Shihao Wu \\
  Department of Statistics and Data Science, University of California, Davis \\
  Davis, CA 95616 \\
  \texttt{shaowu@ucdavis.edu} \\
  \AND
  Gongjun Xu \\
  Department of Statistics, University of Michigan \\
  Ann Arbor, MI 48109 \\
  \texttt{gongjun@umich.edu} \\
  \And
  Ji Zhu \\
  Department of Statistics, University of Michigan \\
  Ann Arbor, MI 48109 \\
  \texttt{jizhu@umich.edu} \\
}

\begin{document}

\maketitle

\begin{abstract}
Hypergraphs record multi-way interactions among entities. Extracting information
from the combinatorial structure underlying observed multi-way
interactions is a central task in many real-world problems. Existing methods
face several limitations. First, many deep architectures for hypergraphs do not
explicitly exploit the potential low-rank structure, which can sacrifice
parsimony and interpretability in the learned representations. Second, many
low-rank-based methods operate on tensor representations, which typically
require hyperlinks to have uniform sizes and thus limit their applicability to general
hypergraphs with non-uniform hyperlink sizes. Third, many methods ignore the
fact that hyperlinks often arise from heterogeneous mechanisms. For example,
medical symptoms may co-occur in the profiles of patients with very different
conditions, and such heterogeneity should be incorporated into the learning
process.
In this work, we develop a general framework for hypergraph representation
learning using hyperlink random effects while exploiting the low-rank structure
in hypergraphs. The proposed framework accommodates latent heterogeneity in
hyperlink formation while preserving entity interaction patterns. We establish identifiability of the model parameters and theoretical guarantees of  representation-level recovery under this framework. The framework allows flexible specifications for the
hyperlink random effects; in this paper, we study three choices: categorical, Gaussian mixture, and score-based effects, and develop corresponding estimation algorithms. Through simulation studies, we
demonstrate the effectiveness of the proposed method in recovering latent
structure and capturing heterogeneous interaction patterns. Empirical studies on
 real-world hypergraph datasets further
illustrate the practical utility of our approach.
\end{abstract}

\etocdepthtag.toc{main}

\section{Introduction}

Hypergraphs provide a flexible representation for relational data in which
interactions can occur among any subsets of entities. Unlike traditional
graphs, where each edge connects two nodes, a hyperlink can connect any number
of nodes and record the corresponding multi-way interaction. Such set-valued
interactions arise in many domains, including recommender systems
\citep{xia2021self}, biological networks
\citep{klamt2009hypergraphs,feng2021hypergraph}, medical records
\citep{johnson2016mimic,johnson2023mimic}, and social systems
\citep{benson2018simplicial,cencetti2021temporal}. A central task in these applications
is hypergraph representation learning, which seeks low-dimensional embeddings of
entities that summarize the combinatorial structure of observed hyperlinks and
support downstream tasks such as hyperlink prediction, clustering,
classification, and interpretation.

Existing methods face several limitations. First, many deep architectures for
hypergraphs do not explicitly exploit the potential low-rank structure
underlying multi-way interactions. This can lead to representations that are
less parsimonious and harder to interpret, with potential higher computational demand. 
Second, many low-rank-based methods
operate on tensor representations, which typically require hyperlinks to have
uniform sizes and can become computationally challenging for high-order tensors.
This limits their applicability to general hypergraphs with hyperlinks of
varying sizes. Third, many methods assume, either explicitly or implicitly, that
hyperlinks are formed through a homogeneous mechanism conditional on node-level
information or parameters. In practice, however, hyperlinks often arise from
heterogeneous mechanisms. For example, in medical records, sets of
symptoms may appear in patients with very different underlying conditions. A
practical representation learning method should therefore capture such latent
heterogeneity while preserving shared entity-level interaction patterns.

To address these concerns, we introduce a general framework for hypergraph representation
learning with flexible hyperlink random effects. The framework assigns each node
a latent embedding and a degree parameter. The latent embeddings characterize
entity interactions in a shared low-dimensional space, while the degree
parameters capture individual entity popularity. To model heterogeneity in
hyperlink formation, each hyperlink is associated with a random-effect embedding
that represents its latent formation mechanism. The node-level parameters
preserve common interaction patterns across the hypergraph, whereas the
hyperlink-level random effects allow different hyperlinks to arise from
different latent mechanisms. Figure~\ref{fig:model} illustrates this idea using
a symptom co-occurrence hypergraph from medical record applications. A formal
introduction to the model is given in Section~\ref{sec:model}.

\begin{figure}[t]
    \centering
    \includegraphics[width=1\linewidth]{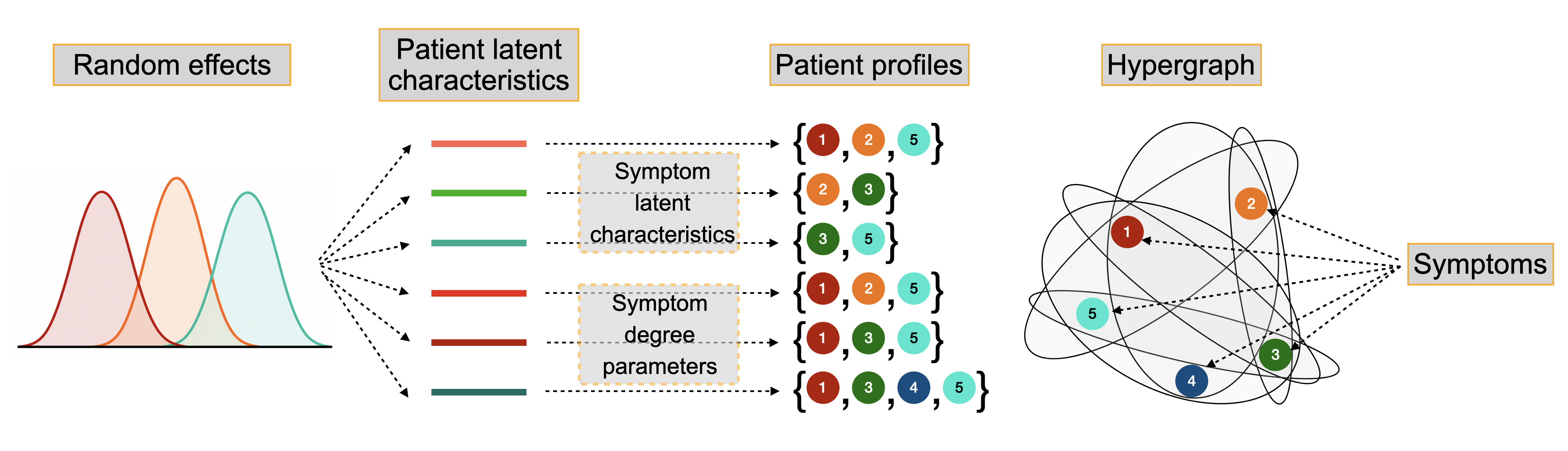}
    \caption{Illustration of hypergraph representation learning with hyperlink
    random effects using a patient-profile toy example. The data are represented
    as a symptom co-occurrence hypergraph with five nodes, corresponding to
    symptoms, and six observed hyperlinks, corresponding to symptom sets observed
    in individual patient profiles. Each symptom is assigned a latent embedding
    and degree parameter. To capture patient-level heterogeneity, each profile
    is assigned a hyperlink random-effect embedding representing latent patient
    characteristics that drive profile formation.\vspace*{-0.5em}}
    \label{fig:model}
\end{figure}

\subsection{Related work}

\paragraph{Projecting hypergraphs onto graphs.}
A common strategy when analyzing hypergraphs is to first project them into
weighted graphs, enabling the use of graph-based tools such as modularity
maximization \citep{kumar2018hypergraph}, graph-likelihood-based methods
\citep{lee2020robust}, and spectral methods
\citep{ghoshdastidar2015provable}. This projection strategy is convenient and have
been used in applications including academic collaboration networks
\citep{newman2001structure,ji2016coauthorship} and congressional voting
relations \citep{fowler2006connecting,lee2017time}. However, projecting a
hypergraph into pairwise edges can cause information loss
and lead to suboptimal performance in tasks such as community detection
\citep{ke2019community}. Our framework instead models hyperlinks directly as
multi-way interactions and preserves the original higher-order structure.

\paragraph{Tensor-based methods.}
Tensor-based methods directly model higher-order structure and are especially
natural for uniform hypergraphs, where all hyperlinks have the same size. They
have been used for tasks such as community detection
\citep{ghoshdastidar2017consistency,ke2019community}, embedding learning
\citep{yuan2021high}, and inference \citep{yuan2022testing}.
However, fixed-order tensor representations are less suited to general
hypergraphs with non-uniform hyperlink sizes and can become computationally
challenging when the tensor order is large. Existing extensions to non-uniform
hypergraphs often require preprocessing or augmentation, such as adding null
vertices to pad hyperlinks to a common size \citep{zhen2023community}. In
contrast, our framework works directly with observed hyperlinks of varying sizes
and exploits low-rank structure through entity embeddings and hyperlink random
effects, avoiding the need to represent the hypergraph as a fixed-order tensor.

\paragraph{Deep architectures for hypergraph learning.}
Recent neural architectures extend graph learning to hypergraphs in several directions. For instance, 
\cite{feng2019hypergraph} define spectral convolution using a hypergraph Laplacian,
\cite{yadati2019hypergcn} train graph convolutional networks on sparse graph approximations of hyperedges,
\cite{zhang2020hypersagnn} use self-attention for hyperedges,
\cite{dong2020hnhn} update both node and hyperedge representations through nonlinear hypergraph convolutions,
\cite{huang2021unignn} give a unified message-passing view of graph and hypergraph neural networks,
and \cite{chien2022allset} implement hypergraph layers through learnable multiset functions.
These models are flexible and powerful for prediction, but they are not primarily designed to impose a compact low-rank structure, while
the low-rank hypothesis motivates parsimonious representations of complex interactions \citep{thibeault2024low}.
Our approach instead learns an identifiable and interpretable low-rank representation directly for general, non-uniform hypergraphs.

\paragraph{Heterogeneous hypergraph learning.}
Many methods incorporate observed heterogeneity, such as hyperlink types or semantic relations, through type-aware interactions or
hyperedge-type decompositions
\citep{sun2021heterogeneous,li2022heterogeneous,lu2023neighborhood}.
These methods rely on observed type information, whereas our framework models
latent heterogeneity in hyperlink formation and can further distinguish
different mechanisms within hyperlinks sharing the same observed type. Related
work also learns edge- or hyperlink-level representations: 
\cite{jo2021edge} learns graph-edge representations through a dual hypergraph
transformation, while \cite{wu2024general} assigns each observed
hyperlink its own embedding. Such individual embeddings are flexible and have
been used for hypergraph generation \citep{wu2025denoising,marelash}, but the
number of embeddings grows with the number of observed hyperlinks. The random-effect
parametrization in this work introduces shared structure across hyperlinks, thus improving efficiency and
parsimony.

\subsection{Contribution and organization of the paper}

The contributions of this paper are threefold. First, we introduce a general
framework for hypergraph representation learning with hyperlink random effects.
The framework exploits low-rank structure in hypergraphs while accommodating
latent heterogeneity in hyperlink formation. Second, we establish identifiability
of the model parameters and develop expectation-maximization-based estimation
procedures. The framework allows flexible specifications for the hyperlink
random effects; in this paper, we study three choices: categorical, Gaussian
mixture, and score-based effects. For each choice, we develop the corresponding
theory to support the model formulation, identifiability conditions, and
algorithm construction. Third, we conduct empirical studies through simulation and real-data
applications to demonstrate the effectiveness of the proposed framework in efficiently 
recovering latent structure, capturing heterogeneous interaction patterns, and
producing useful hypergraph representations.

The remainder of the paper is organized as follows. Section \ref{sec:model} introduces the general framework and model formulation of hypergraph representation learning using hyperlink random effects. Section \ref{sec:estimation} studies identifiability conditions and estimation procedures under three choices of the hyperlink random effects. Section \ref{sec:numerics} presents empirical results from simulations and real-data applications. Section \ref{sec:conclusion} concludes the paper with a discussion. Additional  numerical results, additional discussion, and proofs of theoretical results are provided in the Appendix.

\section{Hypergraph Representation Learning via Hyperlink Random Effects}
\label{sec:model}

Let $\cH([n],\cE_m)$ denote an observed hypergraph with node set
$[n]=\{1,\ldots,n\}$ and hyperlink multiset
$\cE_m=\{e_1,\ldots,e_m\}$. Here, ``multiset'' means that identical
hyperlinks are allowed to be observed repeatedly. Each hyperlink
$e_j\subseteq[n]$ contains the nodes participating in the $j$-th interaction.

In this section, we introduce the general hyperlink random-effects framework for
hypergraph representation learning. We assign each node $i\in[n]$ a latent
embedding $\bz_i$, which can reside in general latent spaces, and a degree
parameter $\alpha_i\in\RR$. Each hyperlink $j\in[m]$ is associated with an
unobserved random effect $\bff_j$, independently drawn from a latent distribution
$\PP_{\bF}$. Random effects are a standard tool for modeling unobserved
unit-level heterogeneity \citep{gelman2007data}; in our setting, the units are hyperlinks, and
$\PP_{\bF}$ shares information across hyperlinks while allowing them to arise
from different latent formation mechanisms. Conditional on the random effects $\{\bff_j\}_{j\in[m]}$ and the node parameters $\{\bz_i,\alpha_i\}_{i\in[n]}$, the node-inclusion indicators
$\bone_{\{i\in e_j\}}$ are independent Bernoulli variables:
\begin{equation}
\label{eq:hre}
    \bone_{\{i\in e_j\}}\mid \{\bff_j\}_{j\in[m]},\{\bz_i,\alpha_i\}_{i\in[n]}
    \sim
    \mathrm{Bernoulli}\left[
    \sigma\{\cK(\bff_j,\bz_i)+\alpha_i\}
    \right],~ i\in[n],
\end{equation}
where $\sigma(x)=\exp(x)/(1+\exp(x))$ is the logistic link and $\cK$ is a kernel measuring the compatibility between a hyperlink random effect and a node embedding.
In \eqref{eq:hre}, $\bz_i$ represents the latent position of node $i$, $\alpha_i$ captures its individual popularity to appear in hyperlinks, and $\bff_j$ represents the latent context or interaction pattern generating hyperlink $j$. The distribution $\PP_{\bF}$ captures population-level heterogeneity in hyperlink formation. For example, a discrete $\PP_{\bF}$ corresponds to a finite number of hyperlink types, a Gaussian mixture allows both type-level and within-type variation, and a score-based $\PP_{\bF}$ can capture more flexible non-Gaussian patterns.
In Section~\ref{sec:estimation}, we discuss three examples under this general framework: categorical, Gaussian-mixture, and score-based random effects.

The marginal likelihood of the observed hyperlink multiset is accordingly
\begin{equation}
\label{eq:general_likelihood}
    \mathcal{L}(\cE_m\mid\bZ,\balpha,\PP_{\bF})
    =
    \prod_{j=1}^m
    \int
    \prod_{i=1}^n
    \sigma\{\cK(\bff,\bz_i)+\alpha_i\}^{\bone_{\{i\in e_j\}}}
    \left[1-\sigma\{\cK(\bff,\bz_i)+\alpha_i\}\right]^{1-\bone_{\{i\in e_j\}}}
    d\PP_{\bF}(\bff),
\end{equation}
where $\bZ=(\bz_1,\ldots,\bz_n)^\top\in\RR^{n\times d}$ and $\balpha=(\alpha_1,\ldots,\alpha_n)^\top$.
In the rest of the paper, we focus on the inner-product kernel $\cK(\bff,\bz_i)=\bff^\top\bz_i$ in a $d$-dimensional Euclidean latent space $\RR^{d}$. Then
$\log\{\PP(i\in e_j|\bff_j)/\PP(i\notin e_j|\bff_j)\}
= \bff_j^\top\bz_i+\alpha_i$.
Thus, the hyperlink distribution in model \eqref{eq:hre} depends on
$\bff_j^{\top}\bz_i+\alpha_i$, which raises identifiability issues. Specifically,
for any $\bmu\in\RR^d$ and any nonsingular $\bA\in\RR^{d\times d}$, let
$\bff'=\bA^\top(\bff-\bmu)$, $\bz_i'=\bA^{-1}\bz_i$, and
$\alpha_i'=\alpha_i+\bmu^\top\bz_i$. Then
${\bff'}^\top\bz_i'+\alpha_i'=\bff^\top\bz_i+\alpha_i$. Thus, the location,
scale, and orientation of the latent space are not identifiable without further
normalization. To fix this, let $\mathcal{S}(\PP_{\bF})$ be an equivariant
positive-definite scale functional of $\PP_{\bF}$, and we consider the following identifiability conditions
\begin{equation}\label{eq:sgm_identifiability}
    \E_{\PP_{\bF}}(\bff)=\bzero_d,~~
    \mathcal{S}(\PP_{\bF})=\bI_d,~~
    \bZ^\top\bZ=\diag(\lambda_1,\ldots,\lambda_d),~~
    \lambda_1>\cdots>\lambda_d>0 .
\end{equation}
Under these normalization conditions and regular conditions on $\PP_{\bF}$,  the parameters
$(\PP_{\bF},\bZ,\balpha)$ are identifiable up to coordinate-wise sign flips,
which can be fixed by specifying the sign of one node-embedding coordinate in each latent dimension.

\paragraph{Shared-representation recovery.}
We now study 
shared-representation recovery of the framework.  Model \eqref{eq:hre} assumes conditional independence
of node inclusion indicators within each hyperlink given its random effect.
Our recovery theory, however, does not require this conditional-independence
assumption: we allow within-hyperlink conditional dependence while only assuming independence across hyperlinks.

Specifically, let $\bv_j=(\bone_{\{i\in e_j\}})_{i=1}^n$, and write
$\vartheta=(\bZ,\balpha,\bgamma)$, with the
random-effect distribution belonging to a finite-dimensional prior family
$\{\PP_{\bF,\bgamma}:\bgamma\in\Gamma\subset\RR^s\}$.
Define
\[
    p_j(\bff)
    =
    \prod_{i=1}^n
    \sigma(\bff^\top\bz_i+\alpha_i)^{v_{ji}}
    \{1-\sigma(\bff^\top\bz_i+\alpha_i)\}^{1-v_{ji}},
\]
which, under the model's conditional-independence assumption, is the
conditional probability of observing the hyperlink incidence vector $\bv_j$
given random effect $\bff$, and let
$q_{n,\vartheta}(\bv_j)=\int p_j(\bff)dP_{F,\gamma}(\bff)$
be the corresponding single-hyperlink marginal mass.
For the recovery theory, we use this marginal mass as a working law and
assume \emph{marginal compatibility},
$\PP(\bv_j=\bv)=q_{n,\vartheta_n^\star}(\bv)$, where
$\vartheta_n^\star=(\bZ_n^\star,\balpha_n^\star,\bgamma^\star)$.

\begin{theorem}[Shared-representation recovery]
\label{thm:shared-main}
Let $\widehat\vartheta
=(\widehat{\bZ},\widehat{\balpha},\widehat{\bgamma})$
globally maximize
$\widehat M_{m,n}(\vartheta)
=(mn)^{-1}\sum_{j=1}^m\log q_{n,\vartheta}(\bv_j)$
over $\cT_n$. Suppose Assumptions~\ref{ass:shared-space}--\ref{ass:shared-regularity}
in Appendix~\ref{sec:thm} hold, with fixed $d,s$.
Set $p_n=n(d+1)+s$ and $a_{m,n}=\sqrt{p_n\log m/m}$, and let
$g_n(\varepsilon)$ denote the population criterion gap
in~\eqref{eq:shared-population-gap}.
If $a_{m,n}\to0$ and $a_{m,n}/g_n(\varepsilon)\to0$ for every
$\varepsilon>0$, then there exists a diagonal sign matrix
\(D_{m,n}\in\cD_d\) such that
\[
 \frac{\|\widehat{\bZ}\bD-\bZ_n^\star\|_F}{\sqrt n}
 +\frac{\|\widehat{\balpha}-\balpha_n^\star\|_2}{\sqrt n}
 +\|\widehat{\bgamma}-\bgamma^\star\|_2
 \xrightarrow{\PP}0.
\]
\end{theorem}
Further details on identifiability and recovery consistency are provided
in Appendix~\ref{append:id}.
Theorem~\ref{thm:shared-main} provides a statistical recovery guarantee for
the global maximizer of the marginal objective. In practice, direct maximizing the objective
is challenging because evaluating the marginal likelihood requires
integration over the hyperlink random effects, which may be analytically
intractable. We therefore adopt an expectation-maximization (EM) strategy
that treats these random effects as latent variables to facilitate
optimization. In the next section, we introduce three specifications of
the hyperlink random effects: categorical, Gaussian mixture, and
score-based effects, and develop efficient estimation algorithms using
computational techniques tailored to each specification.

\section{Identifiability and Estimation under Different Random Effect Choices}
\label{sec:estimation}
The random-effect distribution $\PP_{\bF}$ can be chosen flexibly in practice,
depending on the structure of the application and the desired level of model
complexity. In this section, we introduce three representative choices:
categorical, Gaussian-mixture, and score-based random effects. The theoretical derivations supporting the identifiability
and estimation of these examples are provided in Appendix~\ref{append:alg}.
In the component-based examples below, namely the
categorical and Gaussian-mixture random effects, we use $K$ to denote the number
of latent hyperlink types and $\bpi=(\pi_1,\ldots,\pi_K)^\top$ to denote the
mixing probabilities, with $\bPi=\diag(\bpi)$.

\paragraph{Categorical random effects.} As a first example, let $\PP_{\bF}$ be a categorical distribution with $K$
support points, where $\delta_{\bff_k}$ denotes the Dirac measure concentrated at $\bff_k$:
\[
    \PP_{\bF}=\sum_{k=1}^K \pi_k \delta_{\bff_k},\qquad
    \pi_k>0,\quad \sum_{k=1}^K \pi_k=1.
\]
Equivalently, each hyperlink has a latent type \(k_j\in[K]\), with
\(\PP(k_j=k)=\pi_k\) and \(\bff_j=\bff_{k_j}\). The categorical random-effect
specification is one of the most parsimonious ways to model latent
heterogeneity, where each \(\bff_k\), \(k\in[K]\), represents one type of
latent interaction mechanism.  Let
$\bF=(\bff_1,\ldots,\bff_K)^\top\in\RR^{K\times d}$ and define
$q_{ki}=\sigma(\bff_k^\top\bz_i+\alpha_i)$. The likelihood is
\begin{equation}
\label{eq:cat_likelihood}
    \mathcal{L}(\bV\mid\bF,\bZ,\balpha,\bpi)
    =
    \prod_{j=1}^m
    \sum_{k=1}^K
    \pi_k
    \prod_{i=1}^n
    q_{ki}^{v_{ji}}(1-q_{ki})^{1-v_{ji}}
    =
    \prod_{j=1}^m\sum_{k=1}^K\pi_k p_j(\bff_k).
\end{equation}
For this example, the identifiability conditions become
\begin{equation}
\label{eq:cat_identifiability}
    \sum_{k=1}^K\pi_k\bff_k=\bzero_d,~~
    \bF^\top\bPi\bF=\bI_d,~~
    \bZ^\top\bZ=\diag(\lambda_1,\ldots,\lambda_d),~
    \lambda_1>\cdots>\lambda_d>0 .
\end{equation}
With the identifiability conditions, if the component-specific probability vectors
$(q_{k1},\ldots,q_{kn})$ are distinct across $k$, then
$(\bF,\bZ,\balpha,\bpi)$ is identifiable up to coordinate-wise sign flip and permutation of the component
labels. The categorical model can be estimated by an EM (expectation-maximization)
algorithm. Given \(\bQ=(q_{ki})\in(0,1)^{K\times n}\), where \(q_{ki}\) is the
candidate probability that node \(i\) appears on hyperlink \(j\) conditional on
\(k_j=k\), the E-step computes the posterior probability 
\(\Delta_{jk}=\PP(k_j=k\mid\bv_j)\), and the M-step updates \(\bpi\) and \(\bQ\). After estimating $\bQ$, we recover $\bF$, $\bZ$, and
$\balpha$ from the logit matrix $\bTheta=\operatorname{logit}(\bQ)$ and then
normalize them according to \eqref{eq:cat_identifiability}. The EM algorithm
for categorical hyperlink random effects is summarized in Algorithm~\ref{alg:EM}. Details on the implementation, including initialization, are provided in
Section~\ref{append:alg}.

\begin{algorithm}[t]
\caption{EM for categorical hyperlink random effects}
\label{alg:EM}
\small
\begin{algorithmic}[1]
\Require $\bV$; number of components $K$; embedding dimension $d$; initialization $\bQ^{(0)},\bpi^{(0)}$; tolerance $\epsilon$
\Ensure $\widehat{\bF},\widehat{\bZ},\widehat{\balpha},\widehat{\bPi}$
\State Set $t=0$.
\Repeat
\State Compute $\ell_{jk}^{(t)}=\prod_{i=1}^n(q_{ki}^{(t)})^{v_{ji}}(1-q_{ki}^{(t)})^{1-v_{ji}}$ and $\Delta_{jk}^{(t)}=\pi_k^{(t)}\ell_{jk}^{(t)}/\sum_{s=1}^K\pi_s^{(t)}\ell_{js}^{(t)}$.
\State Update $\pi_k^{(t+1)}=m^{-1}\sum_{j=1}^m\Delta_{jk}^{(t)}$ and $q_{ki}^{(t+1)}=\sum_{j=1}^m\Delta_{jk}^{(t)}v_{ji}/\sum_{j=1}^m\Delta_{jk}^{(t)}$.
\State Set $t\leftarrow t+1$.
\Until{$\|\bpi^{(t+1)}-\bpi^{(t)}\|_2\le\epsilon$}
\State Set $\widehat{\bQ}=\bQ^{(t+1)}$, $\widehat{\bpi}=\bpi^{(t+1)}$, $\widehat{\bPi}=\diag(\widehat{\bpi})$, $\widehat{\bTheta}=\operatorname{logit}(\widehat{\bQ})$, $\widehat{\balpha}=\widehat{\bTheta}^\top\widehat{\bpi}$, and $\widehat{\bG}=\widehat{\bTheta}-\bone_K\widehat{\balpha}^\top$.
\State Let $\widehat{\bG}=\bU_d\bD_d\bR_d^\top$ be the rank-$d$ SVD and set $\bF_0=\bU_d\bD_d^{1/2}$, $\bZ_0=\bR_d\bD_d^{1/2}$.
\State Recenter by $\bar\bff=\bF_0^\top\widehat{\bpi}$: set $\bF_0\leftarrow\bF_0-\bone_K\bar\bff^\top$ and $\widehat{\alpha}_i\leftarrow\widehat{\alpha}_i+\bar\bff^\top\bz_{0i}$ for $i\in[n]$.
\State Normalize with $\bS=\bF_0^\top\widehat{\bPi}\bF_0$: set $\bF_1=\bF_0\bS^{-1/2}$ and $\bZ_1=\bZ_0\bS^{1/2}$.
\State Eigendecompose $\bZ_1^\top\bZ_1=\bQ_0\bLambda\bQ_0^\top$ with decreasing diagonal entries in $\bLambda$; return $\widehat{\bF}=\bF_1\bQ_0$, $\widehat{\bZ}=\bZ_1\bQ_0$, $\widehat{\balpha}$, and $\widehat{\bPi}$.
\end{algorithmic}
\end{algorithm}

\paragraph{Gaussian mixture random effects.} As a second example, let $\PP_{\bF}$ be a Gaussian mixture:
\[
    \PP_{\bF}
    =
    \sum_{k=1}^K \pi_k \mathcal{N}(\bmu_k,\bSigma),
\]
where $\bmu_k\in\RR^d$ is the mean of component $k$ and $\bSigma$ is a common
positive-definite covariance matrix. Let
$\bU=(\bmu_1,\ldots,\bmu_K)^\top\in\RR^{K\times d}$. This model extends the
categorical example by allowing hyperlinks in the same latent group to have
continuous within-group variation. The likelihood is
\begin{equation}
\label{eq:gmm_likelihood}
    \mathcal{L}(\bV\mid\bZ,\balpha,\bpi,\bU,\bSigma)
    =
    \prod_{j=1}^m
    \sum_{k=1}^K
    \pi_k
    \int p_j(\bff)\phi_d(\bff;\bmu_k,\bSigma)\,d\bff,
\end{equation}
where $\phi_d(\cdot;\bmu,\bSigma)$ is the $d$-dimensional Gaussian density.
Using the second moment of $\bF$ as the scale functional,
the identifiability conditions become
\begin{equation}
\label{eq:gmm_identifiability}
    \sum_{k=1}^K\pi_k\bmu_k=\bzero_d,~~
    \bSigma+\sum_{k=1}^K \pi_k\bmu_k\bmu_k^\top=\bI_d,~~
    \bZ^\top\bZ=\diag(\lambda_1,\ldots,\lambda_d),~
    \lambda_1>\cdots>\lambda_d>0 .
\end{equation}
If $\bmu_k$ are distinct, $(\bZ,\balpha,\bpi,\bU,\bSigma)$ is
identifiable up to coordinate-wise sign changes and permutation of component
labels. We estimate this model by a variational EM.
The final estimates are centered, whitened, and rotated to satisfy
\eqref{eq:gmm_identifiability}. Details are given in
Appendix~\ref{append:alg:gm}.

\paragraph{Score-based random effects.} As a third example, we model $\PP_{\bF}$ by a score-based stochastic differential
equation (SDE) random effect $\PP_{\bF,\btheta}$, whose time-dependent score network is
parameterized by $\btheta$ \citep{song2020score}. The likelihood is
\begin{equation}
\label{eq:sgm_likelihood}
    \mathcal{L}(\bV\mid\bZ,\balpha,\btheta)
    =
    \prod_{j=1}^m
    \int p_j(\bff)\,d\PP_{\bF,\btheta}(\bff).
\end{equation}
This specification allows $\PP_{\bF}$ to be multimodal, skewed, or otherwise
non-Gaussian without imposing a finite mixture form. For a score-based random effect,
the neural-network parameter $\btheta$ itself is not identifiable, because
different networks may induce the same distribution. The identifiable object is
the generated distribution $\PP_{\bF,\btheta}$. We use the covariance scale
functional and impose
\begin{equation}
\label{eq:sgm_identifiability}
    \E_{\PP_{\bF,\btheta}}(\bff)=\bzero_d,~~
    \operatorname{Cov}_{\PP_{\bF,\btheta}}(\bff)=\bI_d,~~
    \bZ^\top\bZ=\diag(\lambda_1,\ldots,\lambda_d),~
    \lambda_1>\cdots>\lambda_d>0 .
\end{equation}
Under the general identifiability condition in Section~\ref{sec:model}, this
identifies $\PP_{\bF,\btheta}$, $\bZ$, and $\balpha$ up to coordinate-wise sign
changes. Estimation is achieved by a variational EM algorithm with a likelihood-weighted denoising score matching under a
variance-preserving SDE \citep{song2020score}.  After training, the estimates are transformed to
satisfy \eqref{eq:sgm_identifiability}. Let
$\widehat{\bmu}=\E_{\widehat{\PP}_{\bF}}(\bff)$ and
$\widehat{\bSigma}=\operatorname{Cov}_{\widehat{\PP}_{\bF}}(\bff)$, estimated
by Monte Carlo samples from the learned random effect. Set
\(\alpha_{0i}=\widehat{\alpha}_i+\widehat{\bmu}^\top\widehat{\bz}_i\).
The normalized node embeddings are obtained by a rotation, and the normalized
random-effect distribution is obtained by a pushforward transformation of
\(\widehat{\mathbb{P}}_{\bF}\); both depend on the Monte Carlo estimates.
Details on the implementation for the score-based random effects are given in
Appendix~\ref{append:alg:score}.

\section{Experiments}\label{sec:numerics}
We evaluate the proposed methods on both synthetic and real-world hypergraphs. We refer to the three proposed hypergraph representation learning (HRL) methods as HRL-Cat, HRL-GM, and HRL-Score, corresponding respectively to categorical, Gaussian-mixture, and score-based random effects. We also include HRL-Fixed~\citep{wu2024general}, which can be viewed as a special case of HRL-Cat in which each observed hyperlink is assigned its own latent category similar to fixed effects.

\subsection{Simulated hypergraphs}\label{sec:sim}
We generate synthetic hypergraphs from the model in Section~\ref{sec:model},
considering two distributions for the hyperlink random effects
\(\{\bff_j\}_{j=1}^m\): a Gaussian mixture and a heavy-tailed \(t\)-mixture
with \(4\) degrees of freedom. We fix the number of nodes and vary the number
of hyperlinks. Additional simulation details and results are provided in
Appendix~\ref{append:simulation}.
We evaluate two aspects of algorithm performance. First, we measure
node-embedding recovery using the normalized Frobenius error
\(\|\widehat{\bZ}-\bZ^*\|_F/\sqrt{nd}\), with results shown in
Figure~\ref{fig:sim_combine}. Second, we evaluate latent group recovery using
the adjusted Rand index (ARI) and normalized mutual information (NMI) between
the true and estimated group assignments. Due to space constraints, the mean
ARI and NMI over 20 repetitions are reported in
Table~\ref{tab:simulation-clustering}, while more detailed tables with standard
errors are provided in Appendix~\ref{append:sim-continuous-random-effects}.

\begin{figure}[t]
\centering\includegraphics[width=5in]{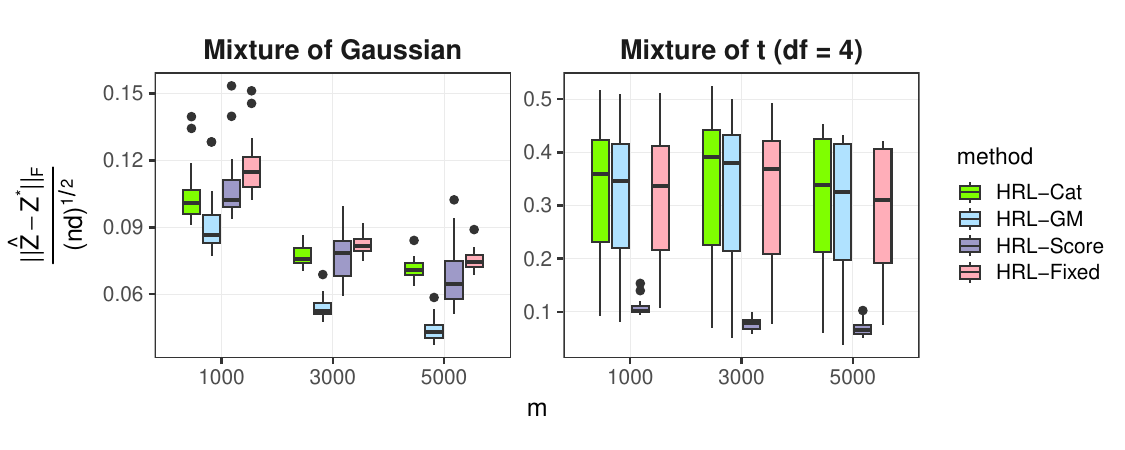}
	\caption{Recovery of node embeddings under Gaussian-mixture (left) and t-mixture (right), with $n=180$.}
    \label{fig:sim_combine}
\end{figure}

\begin{table}[t]
\centering
\small
\setlength{\tabcolsep}{3pt}
\caption{Clustering performance for Gaussian-mixture and heavy-tailed \(t\)-mixture designs. Detailed tables with standard errors are in Appendix~\ref{append:simulation}.}
\label{tab:simulation-clustering}
\begin{tabular}{lcccccccccccc}
\toprule
& \multicolumn{6}{c}{Gaussian mixture}
& \multicolumn{6}{c}{$t$-mixture, df = 4} \\
\cmidrule(lr){2-7} \cmidrule(lr){8-13}
Method
& \multicolumn{2}{c}{1000}
& \multicolumn{2}{c}{3000}
& \multicolumn{2}{c}{5000}
& \multicolumn{2}{c}{1000}
& \multicolumn{2}{c}{3000}
& \multicolumn{2}{c}{5000} \\
\cmidrule(lr){2-3} \cmidrule(lr){4-5} \cmidrule(lr){6-7}
\cmidrule(lr){8-9} \cmidrule(lr){10-11} \cmidrule(lr){12-13}
& ARI$\uparrow$ & NMI$\uparrow$
& ARI$\uparrow$ & NMI$\uparrow$
& ARI$\uparrow$ & NMI$\uparrow$
& ARI$\uparrow$ & NMI$\uparrow$
& ARI$\uparrow$ & NMI$\uparrow$
& ARI$\uparrow$ & NMI$\uparrow$ \\
\midrule
HRL-Cat
& 0.958 & 0.956 & 0.961 & 0.955 & 0.963 & 0.956
& 0.343 & 0.515 & 0.322 & 0.488 & 0.357 & 0.518 \\
HRL-GM
& 0.962 & 0.959 & \underline{0.962} & \underline{0.956} & 0.964 & \underline{0.956}
& \underline{0.343} & \underline{0.517} & \underline{0.344} & \underline{0.510} & \underline{0.358} & \underline{0.519} \\
HRL-Score
& \textbf{0.964} & \underline{0.960} & \textbf{0.965} & \textbf{0.959} & \textbf{0.966} & \textbf{0.959}
& \textbf{0.948} & \textbf{0.942} & \textbf{0.954} & \textbf{0.943} & \textbf{0.951} & \textbf{0.939} \\
HRL-Fixed
& \underline{0.962} & \textbf{0.960} & 0.962 & 0.956 & \underline{0.964} & \underline{0.956}
& 0.338 & 0.501 & 0.314 & 0.471 & 0.351 & 0.504 \\
\bottomrule
\end{tabular}
\end{table}

Under the Gaussian-mixture setup, HRL-GM is correctly specified and
consistently achieves the most accurate recovery of \(\bZ\). Meanwhile, all four
HRL variants recover the latent group structure well, with ARI and NMI values
above \(0.95\), supporting the validity of the algorithms and the proposed
identifiability normalization. Under the heavy-tailed \(t\)-mixture design,
HRL-Cat, HRL-GM, and HRL-Fixed fail to recover meaningful group assignments,
whereas HRL-Score continues to achieve both low embedding error and near-exact
group recovery. This robustness is consistent with its score-based formulation,
demonstrating its ability to adapt to more complex random-effect distributions.

\subsection{Real-world hypergraphs}\label{sec:real-data}

We now evaluate the utility of the learned hypergraph representations on
real-world hypergraphs. Specifically, we consider three downstream tasks:
row-disjoint cold-start experiment, multi-label classification, and node clustering.

\paragraph{Datasets.} We consider three real-world datasets from various domains. (i) The medical-record hypergraph built from MIMIC-III \cite{johnson2016mimic}, where nodes are International Classification of Diseases, Ninth Revision (ICD-9), and each patient record induces a hyperlink over the co-occurring clinical codes; (ii) The food-recipe hypergraph\footnote{\href{https://www.kaggle.com/datasets/hugodarwood/epirecipes/data}https://www.kaggle.com/datasets/hugodarwood/epirecipes/data}, where ingredients are nodes and each recipe forms a hyperlink containing its ingredients; (iii) The movie-rating hypergraph built from MovieLens\footnote{\href{https://grouplens.org/datasets/movielens/}{https://grouplens.org/datasets/movielens/}}, where movies are nodes and, for each user, the set of movies rated above 4 forms a hyperedge. Detailed construction and preprocessing procedures are provided in Appendix \ref{sec:real-data-details}.

\paragraph{Baselines.} We compare the four HRL-based methods (HRL-Cat, HRL-GM, HRL-Score, and
HRL-Fixed) with several baseline methods. For cold-start experiment, we compare
against NeuMF~\citep{he2017neural}, AllSet~\cite{chien2022allset}, HypeBoy~\citep{kim2024hypeboy}, MultiVAE~\citep{liang2018variational},
LightGCN~\citep{he2020lightgcn}, EHGNN~\citep{jo2021edge},
NNMF~\citep{dziugaite2015neural}, and GraphSAGE~\citep{hamilton2017inductive}. For node clustering, we exclude MultiVAE because
it does not support this task and compare only with the other four baseline methods.
For multi-label classification, we additionally include
Hyper-SAGNN~\citep{zhang2020hypersagnn} as a baseline.

\paragraph{Row-disjoint cold-start experiment.} In this task, we randomly hold out a subset of entire hyperlinks during training. For each hold-out hyperlink, we reveal a fraction $r_{\mathrm{obs}}$ of both present and absent entries
and predict node-inclusion probabilities for the remaining coordinates. We evaluate
predictive performance using the area under the receiver operating characteristic
curve (AUROC) and the area under the precision--recall curve (AUPRC), for which
larger values are better, as well as binary cross-entropy (BCE) and Brier score, for
which smaller values are better. As shown in Table~\ref{tab:cold-main}, our proposed HRL-GM and HRL-Score
occupy the top two positions across all metrics on all three datasets, outperforming HRL-Fixed and the other baselines. Even our simplest random-effect specification, HRL-Cat, outperforms HRL-Fixed across all metrics on Food Recipe and MIMIC-III, supporting the value of learning population-level hyperlink heterogeneity to improve prediction for new interactions with
limited observations.

\begin{table}[t]
\centering
\caption{
Cold-start experiment with $r_{obs}=0.2$. For each metric, the best result is
highlighted in bold and the second-best result is underlined. Detailed tables with standard errors are in Appendix~\ref{append:simulation}.
}
\label{tab:cold-main}
\setlength{\tabcolsep}{3pt}
\renewcommand{\arraystretch}{1.08}
\resizebox{\linewidth}{!}{%
\begin{tabular}{@{}l*{12}{c}@{}}
\toprule
& \multicolumn{4}{c}{Food Recipe}
& \multicolumn{4}{c}{MIMIC-III}
& \multicolumn{4}{c}{MovieLens} \\
\cmidrule(lr){2-5}
\cmidrule(lr){6-9}
\cmidrule(lr){10-13}
Method
& AUC $\uparrow$ & AUPRC $\uparrow$ & BCE $\downarrow$ & Brier $\downarrow$
& AUC $\uparrow$ & AUPRC $\uparrow$ & BCE $\downarrow$ & Brier $\downarrow$
& AUC $\uparrow$ & AUPRC $\uparrow$ & BCE $\downarrow$ & Brier $\downarrow$ \\
\midrule
HRL-Cat
& 0.787 & 0.160 & 0.132 & 0.031
& 0.771 & 0.396 & 0.334 & 0.097
& 0.728 & 0.644 & 0.601 & 0.205 \\

HRL-GM
& \textbf{0.827} & \underline{0.194}
& \underline{0.121} & \underline{0.030}
& \underline{0.795} & \textbf{0.437}
& \underline{0.315} & \textbf{0.093}
& \textbf{0.783} & \textbf{0.712}
& \textbf{0.545} & \textbf{0.184} \\

HRL-Score
& \underline{0.826} & \textbf{0.195}
& \textbf{0.120} & \textbf{0.029}
& \textbf{0.796} & \underline{0.436}
& \textbf{0.314} & \underline{0.093}
& \underline{0.776} & \underline{0.703}
& \underline{0.552} & \underline{0.187} \\

HRL-Fixed
& 0.769 & 0.130 & 0.206 & 0.041
& 0.745 & 0.337 & 0.506 & 0.111
& 0.767 & 0.682 & 0.617 & 0.199 \\

NeuMF
& 0.755 & 0.132 & 0.133 & 0.031
& 0.768 & 0.386 & 0.330 & 0.097
& 0.754 & 0.664 & 0.576 & 0.196 \\

AllSet
& 0.758 & 0.123 & 0.131 & 0.031
& 0.765 & 0.378 & 0.331 & 0.098
& 0.659 & 0.572 & 0.635 & 0.222 \\

HypeBoy
& 0.781 & 0.142 & 0.128 & 0.031
& 0.768 & 0.388 & 0.329 & 0.097
& 0.681 & 0.597 & 0.655 & 0.231 \\

MultiVAE
& 0.789 & 0.147 & 0.127 & 0.030
& 0.769 & 0.386 & 0.330 & 0.098
& 0.652 & 0.569 & 0.638 & 0.223 \\

LightGCN
& 0.528 & 0.038 & 0.147 & 0.032
& 0.658 & 0.237 & 0.369 & 0.109
& 0.698 & 0.610 & 0.612 & 0.212 \\

EHGNN
& 0.774 & 0.141 & 0.129 & 0.031
& 0.769 & 0.377 & 0.347 & 0.104
& 0.606 & 0.515 & 0.655 & 0.232 \\

NNMF
& 0.658 & 0.071 & 0.143 & 0.032
& 0.749 & 0.348 & 0.338 & 0.100
& 0.693 & 0.615 & 0.615 & 0.213 \\

GraphSAGE
& 0.776 & 0.128 & 0.131 & 0.031
& 0.768 & 0.383 & 0.339 & 0.101
& 0.664 & 0.579 & 0.645 & 0.226 \\
\bottomrule
\end{tabular}%
}
\end{table}

\paragraph{Multi-label classification.} In the second task, we evaluate the downstream utility of the learned node embeddings
for supervised prediction. MovieLens provides genre labels for movies, and each movie
may belong to multiple genres. We therefore perform multi-label classification by
training one logistic regression classifier per genre, using the learned movie embeddings
as predictors. The training proportion varies from \(10\%\) to \(90\%\). Following
\citet{zhang2020hypersagnn}, we report Micro-F1 and Macro-F1 scores, where
Micro-F1 emphasizes performance on popular genres, while Macro-F1 assigns equal
weight to each genre and is therefore more sensitive to rare labels.
As shown in Figure~\ref{fig:movie_small}, the HRL-based methods consistently
outperform the non-HRL baselines across all training proportions. HRL-GM, HRL-Score,
and HRL-Fixed form the top-performing group under both Micro-F1 and Macro-F1,
whereas HRL-Cat is weaker but still outperforms the neural baselines. The advantage
of the HRL-based methods on Macro-F1 indicates that the learned embeddings capture
information useful not only for dominant genres but also for less frequent ones.
HRL-Cat performs worse than the other HRL variants, suggesting that its
categorical latent-distribution assumption may be less expressive for the richer
structure of MovieLens. HRL-Fixed performs slightly better in some regimes, but
remains very close to HRL-GM and HRL-Score.

\begin{figure}[t]
\centering\includegraphics[width=5in]{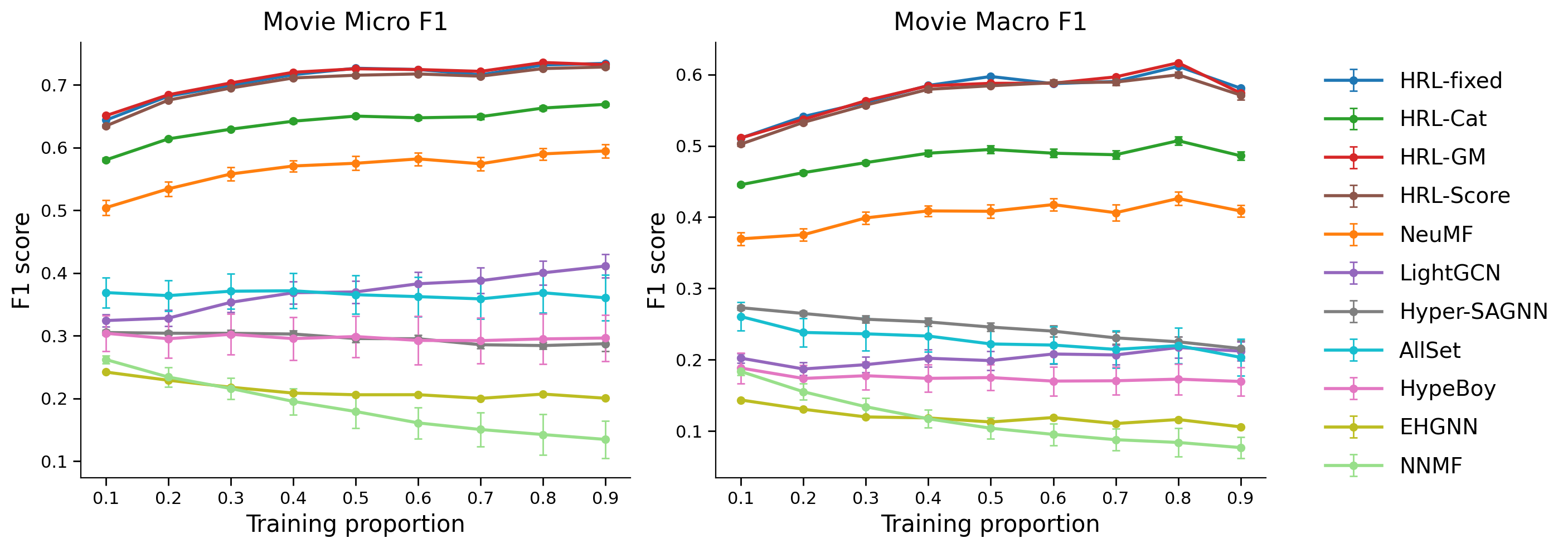}
	\caption{Performance of multi-label classification on Movie Lens datasets}
    \label{fig:movie_small}
\end{figure}

\paragraph{Node clustering.} 
Finally, we evaluate the learned representations on node clustering for the
medical-record and food-recipe hypergraphs. We present the results for the
medical-record hypergraph in the main text. In this hypergraph, nodes correspond to
ICD-9 codes, which encode information such as diagnoses, complications, and
care-related conditions. Thus, node clusters represent empirical clinical-code
phenotypes. For each method, we apply a soft clustering based on a Gaussian mixture model to the
learned embeddings and define core cluster members as those with posterior
membership probability at least \(0.9\). Since true labels are not observed, we
evaluate clustering quality through the semantic coherence of these core members.
Due to space constraints, we present here the results for HRL-Score with \(7\)
clusters, which yields the most interpretable and well-separated clustering structure
in Figure~\ref{fig:mimic_umap}. The model separates acute cardiac disease from
cardiometabolic comorbidities, and separates diabetes-specific complications from
general kidney dysfunction. Although the postoperative cluster is relatively large,
it accurately captures a broad, procedure-related care phenotype, reflecting the
diverse nature of ICD-9 annotations. In comparison, baseline embeddings tend to
exhibit less coherent clustering behavior. For example, LightGCN recovers some
local clinical-code patterns, but its clusters are either coarse when the number of
clusters is small or rather mixed when the number of clusters is large. Another
finding is that, while HRL-Fixed demonstrates strong performance in multi-label classification,
it produces a highly imbalanced clustering geometry. It separates several small
microclusters but leaves a large unseparated component even when the number of clusters increases.
Overall, HRL-Score learns embeddings with clear and meaningful cluster structures.
Detailed implementations and additional results are provided in
Appendix~\ref{sec:mimic}.

\begin{figure}[t]
    \centering
    \includegraphics[width=\linewidth]{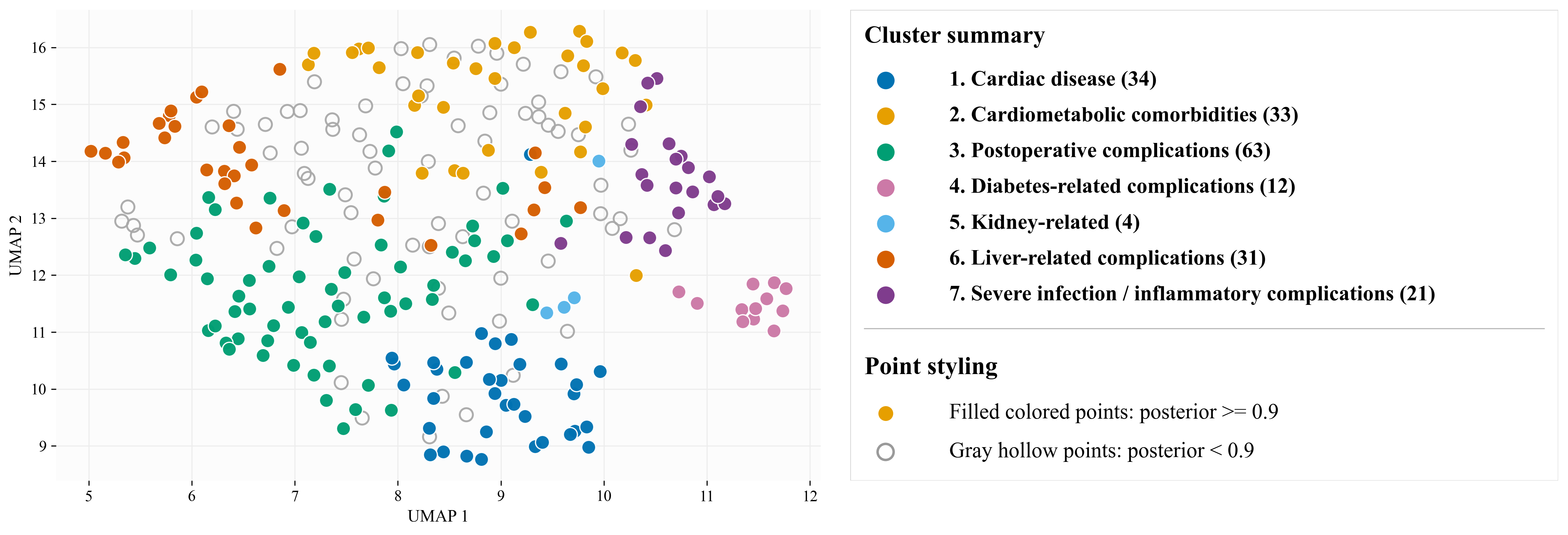}
    \caption{Uniform Manifold Approximation and Projection (UMAP) of the clustering results based on ICD-9 code embeddings
learned by HRL-Score with \(7\) clusters. Points with the same color belong to the
same group. Filled points represent core members with posterior membership
probability at least \(0.9\), and hollow points represent nodes with posterior
membership probability less than \(0.9\). The right panel reports the cluster names
and the sizes of the core-member sets.
    }
    \label{fig:mimic_umap}
\end{figure}

\paragraph{Discussion.}
The empirical results on real-world hypergraphs show that the HRL methods introduced
in this paper, especially HRL-GM and HRL-Score, provide strong and stable
performance across various downstream tasks. Different choices of hyperlink random
effects can lead to different empirical behavior and may be preferable for different
tasks and datasets. The choice is flexible and should be adapted to the application by
practitioners. In Appendix~\ref{append:select}, we also provide a data-driven strategy
for selecting the random-effect specification. Overall, these results suggest that the proposed hypergraph
representation learning framework learns meaningful representations and has broad
downstream utility.

\section{Conclusion}\label{sec:conclusion}

In this paper, we propose a general framework for hypergraph representation learning with hyperlink random effects. It represents each node by a latent embedding and a degree parameter and each hyperlink by a random-effect embedding, enabling direct modeling of non-uniform hypergraphs while capturing heterogeneous hyperlink formation and low-rank interaction structure. We establish identifiability of the model parameters and provide theoretical guarantees for representation-level recovery. We study categorical, Gaussian-mixture, and score-based random effects, develop corresponding estimation procedures, and demonstrate strong performance in simulations and real-data analyses.

We also acknowledge several limitations of the current work, which suggest directions for future research. The framework relies on a low-rank latent structure, so further study of model misspecification would be useful. The EM-based algorithms may be sensitive to initialization and local optima, and HRL-Score further depends on Monte Carlo approximations. Although the current initialization and implementation already produce good empirical results, further optimization of these components could improve practical applicability. Finally, the current framework focuses on a single observed hypergraph as a first step; extending the framework to dynamic, attributed, or multilayer hypergraphs would further broaden its applicability.

\bibliographystyle{plainnat}
\bibliography{ref}

@inproceedings{feng2019hypergraph,
  title={Hypergraph Neural Networks},
  author={Feng, Yifan and You, Haoxuan and Zhang, Zizhao and Ji, Rongrong and Gao, Yue},
  booktitle={Proceedings of the AAAI conference on artificial intelligence},
  volume={33},
  number={01},
  pages={3558--3565},
  year={2019}
}

@inproceedings{yadati2019hypergcn,
  title={Hypergcn: A new method for training graph convolutional networks on hypergraphs},
  author={Yadati, Naganand and Nimishakavi, Madhav and Yadav, Prateek and Nitin, Vikram and Louis, Anand and Talukdar, Partha},
  journal={Advances in neural information processing systems},
  volume={32},
  year={2019}
}

@inproceedings{zhang2020hypersagnn,
  title={Hyper-SAGNN: A Self-Attention Based Graph Neural Network for Hypergraphs},
  author={Zhang, Ruochi and Zou, Yuesong and Ma, Jian},
  booktitle={International Conference on Learning Representations},
  year={2020}
}

@article{dong2020hnhn,
  title={Hnhn: Hypergraph networks with hyperedge neurons},
  author={Dong, Yihe and Sawin, Will and Bengio, Yoshua},
  journal={arXiv preprint arXiv:2006.12278},
  year={2020}
}

@article{huang2021unignn,
  title={Unignn: a unified framework for graph and hypergraph neural networks},
  author={Huang, Jing and Yang, Jie},
  journal={arXiv preprint arXiv:2105.00956},
  year={2021}
}

@article{chien2022allset,
  title={You are allset: A multiset function framework for hypergraph neural networks},
  author={Chien, Eli and Pan, Chao and Peng, Jianhao and Milenkovic, Olgica},
  journal={arXiv preprint arXiv:2106.13264},
  year={2021}
}

@inproceedings{xia2021self,
  title={Self-supervised hypergraph convolutional networks for session-based recommendation},
  author={Xia, Xin and Yin, Hongzhi and Yu, Junliang and Wang, Qinyong and Cui, Lizhen and Zhang, Xiangliang},
  booktitle={Proceedings of the AAAI conference on artificial intelligence},
  volume={35},
  number={5},
  pages={4503--4511},
  year={2021}
}

@article{klamt2009hypergraphs,
  title={Hypergraphs and cellular networks},
  author={Klamt, Steffen and Haus, Utz-Uwe and Theis, Fabian},
  journal={PLoS computational biology},
  volume={5},
  number={5},
  pages={e1000385},
  year={2009},
  publisher={Public Library of Science San Francisco, USA}
}

@article{feng2021hypergraph,
  title={Hypergraph models of biological networks to identify genes critical to pathogenic viral response},
  author={Feng, Song and Heath, Emily and Jefferson, Brett and Joslyn, Cliff and Kvinge, Henry and Mitchell, Hugh D and Praggastis, Brenda and Eisfeld, Amie J and Sims, Amy C and Thackray, Larissa B and others},
  journal={BMC bioinformatics},
  volume={22},
  number={1},
  pages={287},
  year={2021},
  publisher={Springer}
}

@article{johnson2016mimic,
  title={MIMIC-{III}, a freely accessible critical care database},
  author={Johnson, Alistair EW and Pollard, Tom J and Shen, Lu and Lehman, Li-wei H and Feng, Mengling and Ghassemi, Mohammad and Moody, Benjamin and Szolovits, Peter and Anthony Celi, Leo and Mark, Roger G},
  journal={Scientific data},
  volume={3},
  number={1},
  pages={1--9},
  year={2016},
  publisher={Nature Publishing Group}
}

@article{johnson2023mimic,
  title={Mimic-{IV}},
  author={Johnson, Alistair and Bulgarelli, Lucas and Pollard, Tom and Horng, Steven and Celi, Leo Anthony and Mark, Roger},
  journal={PhysioNet. Available online at: https://physionet.org/content/mimiciv/1.0/(accessed August 23, 2021)},
  year={2020}
}

@article{benson2018simplicial,
  title={Simplicial closure and higher-order link prediction},
  author={Benson, Austin R and Abebe, Rediet and Schaub, Michael T and Jadbabaie, Ali and Kleinberg, Jon},
  journal={Proceedings of the National Academy of Sciences},
  volume={115},
  number={48},
  pages={E11221--E11230},
  year={2018},
  publisher={National Academy of Sciences}
}

@article{cencetti2021temporal,
  title={Temporal properties of higher-order interactions in social networks},
  author={Cencetti, Giulia and Battiston, Federico and Lepri, Bruno and Karsai, M{\'a}rton},
  journal={Scientific reports},
  volume={11},
  number={1},
  pages={7028},
  year={2021},
  publisher={Nature Publishing Group UK London}
}

@article{vahdat2021score,
  title={Score-based generative modeling in latent space},
  author={Vahdat, Arash and Kreis, Karsten and Kautz, Jan},
  journal={Advances in neural information processing systems},
  volume={34},
  pages={11287--11302},
  year={2021}
}

@article{thibeault2024low,
  title={The low-rank hypothesis of complex systems},
  author={Thibeault, Vincent and Allard, Antoine and Desrosiers, Patrick},
  journal={Nature Physics},
  volume={20},
  number={2},
  pages={294--302},
  year={2024},
  publisher={Nature Publishing Group UK London}
}

@article{ke2019community,
  title={Community detection for hypergraph networks via regularized tensor power iteration},
  author={Ke, Zheng Tracy and Shi, Feng and Xia, Dong},
  journal={arXiv preprint arXiv:1909.06503},
  year={2019}
}

@article{wu2024general,
  title={A general latent embedding approach for modeling non-uniform high-dimensional sparse hypergraphs with multiplicity},
  author={Wu, Shihao and Xu, Gongjun and Zhu, Ji},
  journal={arXiv preprint arXiv:2410.12108},
  year={2024}
}

@article{ji2022co,
  title={Co-citation and co-authorship networks of statisticians},
  author={Ji, Pengsheng and Jin, Jiashun and Ke, Zheng Tracy and Li, Wanshan},
  journal={Journal of Business \& Economic Statistics},
  volume={40},
  number={2},
  pages={469--485},
  year={2022},
  publisher={Taylor \& Francis}
}

@inproceedings{liang2018variational,
  title={Variational autoencoders for collaborative filtering},
  author={Liang, Dawen and Krishnan, Rahul G and Hoffman, Matthew D and Jebara, Tony},
  booktitle={Proceedings of the 2018 world wide web conference},
  pages={689--698},
  year={2018}
}

@inproceedings{he2020lightgcn,
  title={Lightgcn: Simplifying and powering graph convolution network for recommendation},
  author={He, Xiangnan and Deng, Kuan and Wang, Xiang and Li, Yan and Zhang, Yongdong and Wang, Meng},
  booktitle={Proceedings of the 43rd International ACM SIGIR conference on research and development in Information Retrieval},
  pages={639--648},
  year={2020}
}

@article{dziugaite2015neural,
  title={Neural network matrix factorization},
  author={Dziugaite, Gintare Karolina and Roy, Daniel M},
  journal={arXiv preprint arXiv:1511.06443},
  year={2015}
}

@inproceedings{he2017neural,
  title={Neural collaborative filtering},
  author={He, Xiangnan and Liao, Lizi and Zhang, Hanwang and Nie, Liqiang and Hu, Xia and Chua, Tat-Seng},
  booktitle={Proceedings of the 26th international conference on world wide web},
  pages={173--182},
  year={2017}
}

@article{chatterjee2015matrix,
  title={Matrix estimation by universal singular value thresholding},
  author={Chatterjee, Sourav},
  journal={The {A}nnals of {S}tatistics},
  pages={177--214},
  year={2015},
  publisher={JSTOR}
}

@article{kumar2018hypergraph,
  title={Hypergraph clustering: a modularity maximization approach},
  author={Kumar, Tarun and Vaidyanathan, Sankaran and Ananthapadmanabhan, Harini and Parthasarathy, Srinivasan and Ravindran, Balaraman},
  journal={arXiv preprint arXiv:1812.10869},
  year={2018}
}

@article{lee2020robust,
  title={Robust hypergraph clustering via convex relaxation of truncated MLE},
  author={Lee, Jeonghwan and Kim, Daesung and Chung, Hye Won},
  journal={IEEE Journal on Selected Areas in Information Theory},
  volume={1},
  number={3},
  pages={613--631},
  year={2020},
  publisher={IEEE}
}

@inproceedings{ghoshdastidar2015provable,
  title={A provable generalized tensor spectral method for uniform hypergraph partitioning},
  author={Ghoshdastidar, Debarghya and Dukkipati, Ambedkar},
  booktitle={International Conference on Machine Learning},
  pages={400--409},
  year={2015},
  organization={PMLR}
}

@article{newman2001structure,
  title={The structure of scientific collaboration networks},
  author={Newman, Mark EJ},
  journal={Proceedings of the National Academy of Sciences},
  volume={98},
  number={2},
  pages={404--409},
  year={2001},
  publisher={National Acad Sciences}
}

@article{ji2016coauthorship,
  title={Coauthorship and citation networks for statisticians},
  author={Ji, Pengsheng and Jin, Jiashun},
  journal={The Annals of Applied Statistics},
  volume ={10},
  pages={1779-1812},
  year={2016}
}

@article{fowler2006connecting,
  title={Connecting the congress: a study of cosponsorship networks},
  author={Fowler, James H},
  journal={Political Analysis},
  volume={14},
  number={4},
  pages={456--487},
  year={2006},
  publisher={Cambridge University Press}
}

@article{lee2017time,
  title={Time-dependent community structure in legislation cosponsorship networks in the Congress of the Republic of Peru},
  author={Lee, Sang Hoon and Magallanes, Jos{\'e} Manuel and Porter, Mason A},
  journal={Journal of Complex Networks},
  volume={5},
  number={1},
  pages={127--144},
  year={2017},
  publisher={OUP}
}

@article{ghoshdastidar2017consistency,
  title={Consistency of spectral hypergraph partitioning under planted partition model},
  author={Ghoshdastidar, Debarghya and Dukkipati, Ambedkar},
  journal={The Annals of Statistics},
  volume={45},
  pages={289-315},
  year={2017}
}

@article{yuan2021high,
  title={High-order joint embedding for multi-level link prediction},
  author={Yuan, Yubai and Qu, Annie},
  journal={Journal of the American Statistical Association},
  volume={118},
  pages={1692-1706},
  year={2021},
  publisher={Taylor \& Francis}
}

@article{yuan2022testing,
  title={Testing community structure for hypergraphs},
  author={Yuan, Mingao and Liu, Ruiqi and Feng, Yang and Shang, Zuofeng},
  journal={The Annals of Statistics},
  volume={50},
  number={1},
  pages={147--169},
  year={2022},
  publisher={Institute of Mathematical Statistics}
}

@article{jo2021edge,
  title={Edge representation learning with hypergraphs},
  author={Jo, Jaehyeong and Baek, Jinheon and Lee, Seul and Kim, Dongki and Kang, Minki and Hwang, Sung Ju},
  journal={Advances in Neural Information Processing Systems},
  volume={34},
  pages={7534--7546},
  year={2021}
}

@article{wu2025denoising,
  title={Denoising diffused embeddings: a generative approach for hypergraphs},
  author={Wu, Shihao and Yang, Junyi and Xu, Gongjun and Zhu, Ji},
  journal={arXiv preprint arXiv:2501.01541},
  year={2025}
}

@inproceedings{marelash,
  title={{ReLaSH}: Reconstructing Joint Latent Spaces for Efficient Generation of Synthetic Hypergraphs with Hyperlink Attributes},
  author={Ma, Feiyan and Wu, Shihao and Xu, Gongjun and Zhu, Ji},
  booktitle={The Fourteenth International Conference on Learning Representations},
    year={2026}
}

@inproceedings{sun2021heterogeneous,
  title={Heterogeneous hypergraph embedding for graph classification},
  author={Sun, Xiangguo and Yin, Hongzhi and Liu, Bo and Chen, Hongxu and Cao, Jiuxin and Shao, Yingxia and Viet Hung, Nguyen Quoc},
  booktitle={Proceedings of the 14th ACM international conference on web search and data mining},
  pages={725--733},
  year={2021}
}

@inproceedings{li2022heterogeneous,
  title={Heterogeneous hypergraph neural network for friend recommendation with human mobility},
  author={Li, Yongkang and Fan, Zipei and Zhang, Jixiao and Shi, Dengheng and Xu, Tianqi and Yin, Du and Deng, Jinliang and Song, Xuan},
  booktitle={Proceedings of the 31st ACM international conference on information \& knowledge management},
  pages={4209--4213},
  year={2022}
}

@article{lu2023neighborhood,
  title={Neighborhood overlap-aware heterogeneous hypergraph neural network for link prediction},
  author={Lu, Yifan and Gao, Mengzhou and Liu, Huan and Liu, Zehao and Yu, Wei and Li, Xiaoming and Jiao, Pengfei},
  journal={Pattern Recognition},
  volume={144},
  pages={109818},
  year={2023},
  publisher={Elsevier}
}

@article{zhen2023community,
  title={Community detection in general hypergraph via graph embedding},
  author={Zhen, Yaoming and Wang, Junhui},
  journal={Journal of the American Statistical Association},
  volume={118},
  number={543},
  pages={1620--1629},
  year={2023},
  publisher={Taylor \& Francis}
}

@article{cho2021gaussian,
   title={Gaussian variational estimation for multidimensional item response theory},
  author={Cho, April E and Wang, Chun and Zhang, Xue and Xu, Gongjun},
  journal={British Journal of Mathematical and Statistical Psychology},
  volume={74},
  pages={52--85},
  year={2021},
  publisher={Wiley Online Library}
}

@article{jaakkola2000bayesian,
  title={Bayesian parameter estimation via variational methods},
  author={Jaakkola, Tommi S and Jordan, Michael I},
  journal={Statistics and Computing},
  volume={10},
  number={1},
  pages={25--37},
  year={2000},
  publisher={Springer}
}

@article{song2020score,
  title={Score-based generative modeling through stochastic differential equations},
  author={Song, Yang and Sohl-Dickstein, Jascha and Kingma, Diederik P and Kumar, Abhishek and Ermon, Stefano and Poole, Ben},
  journal={arXiv preprint arXiv:2011.13456},
  year={2020}
}

@book{bishop2006pattern,
  title={Pattern recognition and machine learning},
  author={Bishop, Christopher M and Nasrabadi, Nasser M},
  volume={4},
  number={4},
  year={2006},
  publisher={Springer}
}

@article{kim2024hypeboy,
  title={Hypeboy: Generative self-supervised representation learning on hypergraphs},
  author={Kim, Sunwoo and Kang, Shinhwan and Bu, Fanchen and Lee, Soo Yong and Yoo, Jaemin and Shin, Kijung},
  journal={arXiv preprint arXiv:2404.00638},
  year={2024}
}

@book{gelman2007data,
  title={Data analysis using regression and multilevel/hierarchical models},
  author={Gelman, Andrew and Hill, Jennifer},
  year={2007},
  publisher={Cambridge university press}
}

@article{hamilton2017inductive,
  title={Inductive representation learning on large graphs},
  author={Hamilton, Will and Ying, Zhitao and Leskovec, Jure},
  journal={Advances in neural information processing systems},
  volume={30},
  year={2017}
}





\clearpage


\appendix
\etocdepthtag.toc{appendix}

\begingroup

\etocsettagdepth{main}{none}
\etocsettagdepth{appendix}{subsection}

\etocsettocstyle
  {%
  \section*{Appendix Contents}
   \vspace{0.4em}
   \noindent\rule{\linewidth}{0.4pt}\par
   \vspace{0.8em} 
  }
  {%
   \vspace{0.8em}
   \noindent\rule{\linewidth}{0.4pt}\par
  }

\tableofcontents

\endgroup

\section{Theoretical Analysis}
\subsection{Discussion on the Identifiability Conditions}\label{append:id}
In this section, we provide additional details on the identifiability conditions
introduced in Section~\ref{sec:model}. We focus on the inner-product kernel
$\cK(\bff,\bz_i)=\bff^\top\bz_i$. Under this kernel, the conditional
probability of node $i$ appearing in hyperlink $j$ depends on the parameters only
through
\[
    \bff_j^\top\bz_i+\alpha_i .
\]
For any $\bmu\in\RR^d$ and any nonsingular matrix $\bA\in\RR^{d\times d}$, let
\[
    \bff'=\bA^\top(\bff-\bmu),~~
    \bz_i'=\bA^{-1}\bz_i,~~
    \alpha_i'=\alpha_i+\bmu^\top\bz_i .
\]
Then
\[
    {\bff'}^\top\bz_i' + \alpha_i'
    =
    \bff^\top\bz_i+\alpha_i ,
\]
and  the conditional hyperlink distribution remains unchanged. This shows
that the location of the random-effect distribution and the scale and orientation
of the latent space cannot be identified without further normalization.

To address this, we impose the following identifiability conditions:
\[
    \E_{\PP_{\bF}}(\bff)=\bzero_d,~~
    \mathcal{S}(\PP_{\bF})=\bI_d,~~
    \bZ^\top\bZ=\diag(\lambda_1,\ldots,\lambda_d),
    ~~
    \lambda_1>\cdots>\lambda_d>0 .
\]
Here $\mathcal{S}(\PP_{\bF})$ is a positive-definite scale functional of the
random-effect distribution. We assume that it is location invariant and
equivariant under nonsingular linear transformations, in the sense that if
$\bF'\sim \bB(\bF-\ba)$ for a nonsingular matrix $\bB$, then
\[
    \mathcal{S}(\PP_{\bF'})
    =
    \bB\,\mathcal{S}(\PP_{\bF})\,\bB^\top .
\]
A common example is the covariance matrix when the second moment exists.

The condition $\E_{\PP_{\bF}}(\bff)=\bzero_d$ fixes the location of the
random-effect distribution. The condition
$\mathcal{S}(\PP_{\bF})=\bI_d$ fixes its scale and removes general invertible
linear transformations, leaving only orthogonal transformations. The diagonal
condition on $\bZ^\top\bZ$, together with the strict ordering
$\lambda_1>\cdots>\lambda_d$, fixes the remaining rotations up to
coordinate-wise sign changes. If the diagonal entries were not ordered, there
would also be a permutation among latent coordinates. If some
eigenvalues were repeated, rotations within the corresponding eigenspaces would
remain unidentified.

These identifiability conditions are not intended to impose restriction on the
hyperlink distribution. Instead, they choose one representative from each
equivalence class of parameterizations that induce the same hyperlink
distribution. To see this, consider an arbitrary parameterization
$(\PP_{\bF},\bZ,\balpha)$ with $\rank(\bZ)=d$ and positive-definite
$\mathcal{S}(\PP_{\bF})$. Let
\[
    \bmu_{\bF}=\E_{\PP_{\bF}}(\bff),
    ~~
    \bC=\mathcal{S}(\PP_{\bF}) .
\]
First define
\[
    \widetilde{\bff}
    =
    \bC^{-1/2}(\bff-\bmu_{\bF}),
    ~~
    \widetilde{\bz}_i
    =
    \bC^{1/2}\bz_i,
    ~~
    \widetilde{\alpha}_i
    =
    \alpha_i+\bmu_{\bF}^\top\bz_i .
\]
Then
\[
    \widetilde{\bff}^{\top}\widetilde{\bz}_i+\widetilde{\alpha}_i
    =
    \bff^\top\bz_i+\alpha_i ,
\]
so the hyperlink distribution is unchanged. Moreover,
$\E(\widetilde{\bff})=\bzero_d$ and
$\mathcal{S}(\PP_{\widetilde{\bF}})=\bI_d$. Next, let
\[
    \widetilde{\bZ}^{\top}\widetilde{\bZ}
    =
    \bGamma \diag(\lambda_1,\ldots,\lambda_d)\bGamma^\top
\]
be the eigendecomposition, where
$\lambda_1>\cdots>\lambda_d>0$ after ordering and $\bGamma$ is orthogonal. Define
\[
    \bff^\star=\bGamma^\top\widetilde{\bff},
    ~~
    \bz_i^\star=\bGamma^\top\widetilde{\bz}_i,
    ~~
    \alpha_i^\star=\widetilde{\alpha}_i .
\]
Then
\[
    {\bff^\star}^{\top}\bz_i^\star+\alpha_i^\star
    =
    \widetilde{\bff}^{\top}\widetilde{\bz}_i+\widetilde{\alpha}_i
    =
    \bff^\top\bz_i+\alpha_i ,
\]
and the induced hyperlink distribution is again unchanged. In addition,
$\E(\bff^\star)=\bzero_d$,
$\mathcal{S}(\PP_{\bF^\star})=\bI_d$, and
\[
    {\bZ^\star}^{\top}\bZ^\star
    =
    \diag(\lambda_1,\ldots,\lambda_d).
\]
Thus, any admissible parameterization can be transformed into one satisfying the
identifiability conditions above.

Finally, we describe the remaining coordinate-wise sign flip issue. Suppose two normalized
parameterizations are related through the transformation
\[
    \bff'=\bA^\top(\bff-\bmu),\qquad
    \bz_i'=\bA^{-1}\bz_i,\qquad
    \alpha_i'=\alpha_i+\bmu^\top\bz_i ,
\]
and both satisfy the stated identifiability conditions. Since both
$\PP_{\bF}$ and $\PP_{\bF'}$ have mean zero, we must have $\bmu=\bzero_d$.
Since both have scale functional equal to $\bI_d$, the equivariance of
$\mathcal{S}$ implies
\[
    \bI_d=\mathcal{S}(\PP_{\bF'})
    =
    \bA^\top \mathcal{S}(\PP_{\bF}) \bA
    =
    \bA^\top\bA ,
\]
so $\bA$ is orthogonal. Moreover, because both $\bZ^\top\bZ$ and
${\bZ'}^\top\bZ'$ are diagonal with strictly ordered diagonal entries, the only
orthogonal transformations preserving these conditions are coordinate-wise sign
changes. Therefore, the normalized parameters are identifiable up to
coordinate-wise sign flips. The signs can be fixed, if desired,
by fixing the sign of one nonzero entry in each column of $\bZ$. Compared with
the location and general linear-transformation non-identifiability, the remaining
coordinate-wise sign ambiguity is minor and can be easily fixed when prior
knowledge about one node is available, by assigning signs to the embedding
coordinates of that node; alternatively, one can choose a reference node with
pre-specified signs.

\subsection{Proof of Theorem \ref{thm:shared-main}}\label{sec:thm}
In this section, we give the regularity conditions and a proof of theorem \ref{thm:shared-main}.

Write $\vartheta=(\bZ,\balpha,\bgamma)$ and
$\vartheta_n^\star=(\bZ_n^\star,\balpha_n^\star,\bgamma^\star)$,
with the coordinate signs and any mixture labels fixed as in
Section~\ref{sec:model}. Define
\begin{equation}
 r_n(\vartheta_1,\vartheta_2)^2
 =\frac{\|\bZ_1-\bZ_2\|_F^2}{n}
  +\frac{\|\balpha_1-\balpha_2\|_2^2}{n}
  +\|\bgamma_1-\bgamma_2\|_2^2.
 \label{eq:shared-metric}
\end{equation}
This is the Euclidean distance between
$\bx_n(\vartheta)=
(n^{-1/2}\operatorname{vec}(\bZ)^\top,
 n^{-1/2}\balpha^\top,\bgamma^\top)^\top\in\RR^{p_n}$,
where $p_n=n(d+1)+s$ is an upper bound on the normalized parameter
dimension. Let
\begin{equation}
 h_{n,\vartheta}(\bv)=\frac1n\log q_{n,\vartheta}(\bv),
 \qquad M_n(\vartheta)=\E_\star h_{n,\vartheta}(\bv_1),
 \label{eq:shared-population-criterion}
\end{equation}
where $\E_\star$ denotes expectation under the true marginal law.
Then $\widehat M_{m,n}(\vartheta)
=m^{-1}\sum_{j=1}^m h_{n,\vartheta}(\bv_j)$.

\begin{assumption}[Compact normalized parameter space]
\label{ass:shared-space}
The dimensions $d,s$ are fixed. The parameter space $\cT_n$ is compact,
contains $\vartheta_n^\star$, and satisfies
\eqref{eq:sgm_identifiability} with coordinate signs fixed jointly for
the embeddings and the prior, and with mixture-label ordering fixed
when applicable. For constants $C_Z,C_\alpha<\infty$ and a compact
$\Gamma\subset\RR^s$, all independent of $n$, every
$\vartheta\in\cT_n$ satisfies
\[
 \max_{i\in[n]}\|\bz_i\|_2\le C_Z,
 \qquad \max_{i\in[n]}|\alpha_i|\le C_\alpha,
 \qquad \bgamma\in\Gamma.
\]
In particular, the prior uses the main-text normalization
$\E_{\PP_{\bF,\bgamma}}(\bff)=\bzero_d$ and
$\mathcal S(\PP_{\bF,\bgamma})=\bI_d$.
\end{assumption}

\begin{assumption}[Independent hyperlinks and marginal compatibility]
\label{ass:shared-compatibility}
The incidence vectors $\bv_1,\ldots,\bv_m$ are independent and
identically distributed with mass function $q_{n,\vartheta_n^\star}$.
No independence or conditional independence is required among the
coordinates of a single $\bv_j$.
\end{assumption}

\begin{assumption}[Integral identification]
\label{ass:shared-identification}
For every $n$ and $\vartheta\in\cT_n$,
$q_{n,\vartheta}(\bv)=q_{n,\vartheta_n^\star}(\bv)$ for all
$\bv\in\{0,1\}^n$ implies $\vartheta=\vartheta_n^\star$ under
the fixed sign and label conventions.
\end{assumption}

\begin{assumption}[Bounded and Lipschitz normalized likelihood]
\label{ass:shared-regularity}
The working masses are strictly positive. There are constants
$C,L<\infty$, independent of $n$, such that for every
$\bv\in\{0,1\}^n$ and $\vartheta,\vartheta_1,\vartheta_2\in\cT_n$,
\[
 -C\le h_{n,\vartheta}(\bv)\le0,
 \qquad
 |h_{n,\vartheta_1}(\bv)-h_{n,\vartheta_2}(\bv)|
 \le Lr_n(\vartheta_1,\vartheta_2).
\]
\end{assumption}

For $\varepsilon>0$, define the population criterion gap
\begin{equation}
 g_n(\varepsilon)
 =\inf_{\substack{\vartheta\in\cT_n:\
                  r_n(\vartheta,\vartheta_n^\star)\ge\varepsilon}}
 \{M_n(\vartheta_n^\star)-M_n(\vartheta)\},
 \label{eq:shared-population-gap}
\end{equation}
with $\inf\varnothing=+\infty$.

\begin{proof}[Proof of Theorem~\ref{thm:shared-main}]
Compactness and Assumption~\ref{ass:shared-regularity} ensure that
$\widehat M_{m,n}$ has a global maximizer.

By marginal compatibility, for every $\vartheta\in\cT_n$,
\begin{align}
 M_n(\vartheta_n^\star)-M_n(\vartheta)
 &=\frac1n\sum_{\bv\in\{0,1\}^n}
 q_{n,\vartheta_n^\star}(\bv)
 \log\frac{q_{n,\vartheta_n^\star}(\bv)}{q_{n,\vartheta}(\bv)}
 \nonumber\\
 &=\frac1n D_{\mathrm{KL}}
   (q_{n,\vartheta_n^\star}\|q_{n,\vartheta})\ge0,
 \label{eq:shared-kl-gap}
\end{align}
with the equality holds precisely when the two mass functions coincide, and
Assumption~\ref{ass:shared-identification} then gives
$\vartheta=\vartheta_n^\star$. Thus the population maximizer is unique.
For each fixed $n$, continuity and compactness show that
$g_n(\varepsilon)>0$ whenever the constraint set in
\eqref{eq:shared-population-gap} is nonempty, otherwise a zero minimum would
produce a distinct parameter with the same marginal law.

The scaled parameter set $\bx_n(\cT_n)$ lies in a Euclidean ball
of radius $R<\infty$ independent of $n$ by
Assumption~\ref{ass:shared-space}. It therefore has a $1/m$-net
$\{\vartheta_1,\ldots,\vartheta_N\}\subset\cT_n$ with
$N\le(1+2Rm)^{p_n}$. To obtain this bound, take a maximal
$1/m$-separated subset and compare the volumes of its disjoint
radius-$1/(2m)$ balls with the radius-$(R+1/(2m))$ ball containing them.
Set
\[
 \Delta_{m,n}
 =\sup_{\vartheta\in\cT_n}
  |\widehat M_{m,n}(\vartheta)-M_n(\vartheta)|.
\]
For any parameter and a nearest net point, the empirical and
population criteria each differ by at most $L/m$. Hence
\[
 \Delta_{m,n}\le
 \max_{1\le k\le N}
 |\widehat M_{m,n}(\vartheta_k)-M_n(\vartheta_k)|+\frac{2L}{m}.
\]
At each net point, the summands are independent and lie in $[-C,0]$.
Hoeffding's inequality and a union bound give, for every $t>0$,
\begin{equation}
 \PP\!\left(\Delta_{m,n}>\frac{2L}{m}+t\right)
 \le2\exp\!\left\{p_n\log(1+2Rm)-\frac{2mt^2}{C^2}\right\}.
 \label{eq:shared-uniform-bound}
\end{equation}
Taking
$t=C\sqrt{\{p_n\log(1+2Rm)+\log m\}/(2m)}$
makes the right-hand side $2/m$. Since $p_n\ge1$, this proves
\begin{equation}
 \Delta_{m,n}=O_{\PP}\!\left(\sqrt{\frac{p_n\log m}{m}}\right)
 =O_{\PP}(a_{m,n}).
 \label{eq:shared-ulln}
\end{equation}

Global maximality gives
$\widehat M_{m,n}(\widehat\vartheta)
\ge\widehat M_{m,n}(\vartheta_n^\star)$, and therefore
\begin{align*}
 0\le M_n(\vartheta_n^\star)-M_n(\widehat\vartheta)
 &\le
 |M_n(\vartheta_n^\star)-\widehat M_{m,n}(\vartheta_n^\star)|
 +|\widehat M_{m,n}(\widehat\vartheta)-M_n(\widehat\vartheta)|
 \\&\le2\Delta_{m,n}.
\end{align*}
Fix $\varepsilon>0$. If
$r_n(\widehat\vartheta,\vartheta_n^\star)\ge\varepsilon$, then
$g_n(\varepsilon)\le2\Delta_{m,n}$. When the constraint set is nonempty,
\[
 \PP\{r_n(\widehat\vartheta,\vartheta_n^\star)\ge\varepsilon\}
 \le\PP\!\left\{\frac{\Delta_{m,n}}{g_n(\varepsilon)}\ge\frac12\right\}
 \longrightarrow0,
\]
since \eqref{eq:shared-ulln} and the theorem's gap condition imply
$\Delta_{m,n}/g_n(\varepsilon)=o_{\PP}(1)$.

If the constraint set is empty, by the definition of the constraint, this would imply $ \PP\{r_n(\widehat\vartheta,\vartheta_n^\star)\ge\varepsilon\}=0$.
Combining these two cases gives
\[
 \frac{\|\widehat{\bZ}-\bZ_n^\star\|_F}{\sqrt n}
 +\frac{\|\widehat{\balpha}-\balpha_n^\star\|_2}{\sqrt n}
 +\|\widehat{\bgamma}-\bgamma^\star\|_2
 \le\sqrt3\,r_n(\widehat\vartheta,\vartheta_n^\star)
 \xrightarrow{\PP}0,
\]
which completes the proof.
\end{proof}

\section{Algorithm Details and Theoretical Derivation}\label{append:alg}
In this section, we provide full algorithm details, including the intuition behind, the exact updating formula in E-Step and M-Step for all the three special cases. We additionally provide initialization methods and cross-validation based methods to choose $K$ and $d$.
\subsection{Categorical Random Effect}
As shown in Section~\ref{sec:model}, when $\mathbb{P}_{\bF}$ is a categorical distribution with $K$ support points, each hyperlink is generated from one of $K$ latent types. Equivalently, each hyperlink $j\in[m]$ has an unobserved type label $k_j\in[K]$, with
$\mathbb{P}(k_j=k)=\pi_k,\bff_j=\bff_{k_j}$.

Conditional on the latent type $k_j=k$, the entries of the incidence vector $\bv_j=(v_{j1},\ldots,v_{jn})^\top$ are independent Bernoulli random variables with success probabilities $\{q_{ki}\}_{i=1}^n$. Considering the absence of information regarding the group affiliation of each hyperlink, a natural approach is to employ the EM algorithm. 

To derive the EM updates, we introduce the unobserved latent membership matrix
$\bDelta=(\Delta_{jk})\in\{0,1\}^{m\times K}$, where $\Delta_{jk}=1$ if hyperlink $j$ is generated from latent type $k$, and $\Delta_{jk}=0$ otherwise. If this auxiliary membership matrix were observed, the corresponding complete-data likelihood would be
\begin{equation}
\label{eq:cat-complete-likelihood}
    \mathcal{L}(\bV,\bDelta\mid\bF,\bZ,\balpha,\bpi)
    =
    \prod_{j=1}^m
    \prod_{k=1}^K
    \left\{
        \pi_k
        \prod_{i=1}^n
        q_{ki}^{v_{ji}}
        (1-q_{ki})^{1-v_{ji}}
    \right\}^{\Delta_{jk}} .
\end{equation}

The corresponding complete-data log-likelihood is
\begin{equation}
\label{eq:cat-complete-loglik}
    \ell(\bV,\bDelta\mid\bF,\bZ,\balpha,\bpi)
    =
    \sum_{j=1}^m
    \sum_{k=1}^K
    \Delta_{jk}
    \left[
        \log \pi_k
        +
        \sum_{i=1}^n
        \left\{
            v_{ji}\log q_{ki}
            +
            (1-v_{ji})\log(1-q_{ki})
        \right\}
    \right].
\end{equation}

For computational convenience, we first estimate the cluster-specific probability matrix
$\bQ=(q_{ki})\in(0,1)^{K\times n}$, where $\bQ=\mathrm{expit}(\bTheta)$ and $\bTheta$ is a $K\times n$ matrix with $\theta_{k,i}=\bff_k^{\top}\bz_i+\alpha_i$. Under this parameterization, the complete-data log-likelihood can be written as
\begin{equation}
\label{eq:cat-complete-loglik-Q}
    \ell(\bQ,\bpi\mid\bV,\bDelta)
    =
    \sum_{j=1}^m
    \sum_{k=1}^K
    \Delta_{jk}
    \left[
        \log \pi_k
        +
        \sum_{i=1}^n
        \left\{
            v_{ji}\log q_{ki}
            +
            (1-v_{ji})\log(1-q_{ki})
        \right\}
    \right].
\end{equation}
We then follow the standard EM algorithm. Given current estimates $(\bQ^{(t)},\bpi^{(t)})$, the E-step calculates 
\begin{equation}\begin{aligned}
\mathcal{Q}(\bQ,\bpi\mid\bQ^{(t)},\bpi^{(t)})
    &:=\E[l(\bQ,\bpi|\bV,\bDelta)|\bV,\bQ^{(t)},\bpi^{(t)}]\\
    &=
    \sum_{j=1}^m
    \sum_{k=1}^K
    \Delta_{jk}^{(t)}
    \left[
        \log \pi_k
        +
        \sum_{i=1}^n
        \left\{
            v_{ji}\log q_{ki}
            +
            (1-v_{ji})\log(1-q_{ki})
        \right\}
    \right]
\end{aligned}\label{eq:Estep}
\end{equation}
where 
$$\begin{aligned}\Delta_{jk}^{(t)}
    &:=
    \mathbb{E}\left[
        \Delta_{jk}
        \mid
        \bv_j,\bQ^{(t)},\bpi^{(t)}
    \right]
    =
    \mathbb{P}\left(
        k_j=k
        \mid
        \bv_j,\bQ^{(t)},\bpi^{(t)}
    \right)
    =
    \frac{
        \pi_k^{(t)} \ell_{jk}^{(t)}
    }{
        \sum_{s=1}^K
        \pi_s^{(t)} \ell_{js}^{(t)}
    }\\
    \ell_{jk}^{(t)}
   & =
    \prod_{i=1}^n
    \left(q_{ki}^{(t)}\right)^{v_{ji}}
    \left(1-q_{ki}^{(t)}\right)^{1-v_{ji}}
    \end{aligned}
$$
Then in the M-step, we maximizes $\mathcal{Q}(\bQ,\bpi\mid\bQ^{(t)},\bpi^{(t)})$ shown in equation \ref{eq:Estep} with respect to $(\bQ,\bpi)$. For the mixing weights, we maximize $\sum_{j=1}^m\sum_{k=1}^K \Delta_{jk}^{(t)}\log \pi_k$
subject to $\sum_{k=1}^K\pi_k=1$, which gives
\begin{equation}
\label{eq:EM-pi-update}
    \pi_k^{(t+1)}
    =
    \frac{1}{m}
    \sum_{j=1}^m
    \Delta_{jk}^{(t)},
    \qquad k\in[K].
\end{equation}
For each $q_{ki}$, the relevant part of the objective is $\sum_{j=1}^m
    \Delta_{jk}^{(t)}
    \left\{
        v_{ji}\log q_{ki}
        +
        (1-v_{ji})\log(1-q_{ki})
    \right\}$, taking the derivative with respect to $q_{ki}$ and setting it equal to zero yields
\[
    \sum_{j=1}^m
    \Delta_{jk}^{(t)}
    \left(
        \frac{v_{ji}}{q_{ki}}
        -
        \frac{1-v_{ji}}{1-q_{ki}}
    \right)
    =
    0,
\]
and hence
\begin{equation}
\label{eq:EM-q-update}
    q_{ki}^{(t+1)}
    =
    \frac{
        \sum_{j=1}^m
        \Delta_{jk}^{(t)} v_{ji}
    }{
        \sum_{j=1}^m
        \Delta_{jk}^{(t)}
    },
    \qquad k\in[K],\ i\in[n].
\end{equation}
Thus the EM algorithm alternates between computing the posterior membership probabilities in \eqref{eq:Estep} and updating the cluster weights and probability matrix according to \eqref{eq:EM-pi-update} and \eqref{eq:EM-q-update}.

In principle, one could directly maximize the likelihood over
$(\bF,\bZ,\balpha,\bpi)$ using a projected gradient method, as in
\citep{wu2024general}. In our implementation, however, we adopt a more computationally efficient two-step strategy. We first estimate the probability matrix $\bQ$ and the mixing weights $\bpi$ by EM procedure discussed above. After obtaining $\widehat{\bQ}$, we recover estimates of $(\bF,\bZ,\balpha)$ through a low-rank logistic factorization.

Specifically, after the EM algorithm returns $\widehat{\bQ}$, we recover $\bTheta$ via the logit function, i.e., set $\hat\theta_{ki}=\log\frac{\hat q_{ki}}{1-\hat q_{ki}},\forall k\in[K],i\in[n]$. We then calculate the column means of $\hat\bTheta$ as pre-normalized estimate of the node-specific intercepts: $\widehat{\alpha}_i=\frac{1}{K}\sum_{k=1}^K \widehat{\theta}_{ki},i\in[n]$. After centering $\hat\Theta$ through $\widehat{\bG}=\widehat{\bTheta}-\bone_K\widehat{\balpha}^{\top}$,
we compute a rank-$d$ singular value decomposition of $\widehat{\bG}$ to obtain preliminary estimates of $\bF$ and $\bZ$. A final normalization step is then applied to ensure that the resulting latent variables satisfy the identifiability constraints described earlier. The full procedure is summarized in Algorithm~\ref{alg:EM}.
\subsection{Gaussian Mixture Random Effects}\label{append:alg:gm}
We next consider the case where $\mathbb{P}_{\bF}$ is a Gaussian mixture:
\[
    \mathbb{P}_{\bF}
    =
    \sum_{k=1}^K
    \pi_k \mathcal{N}(\bmu_k,\bSigma),
\]
where $\bmu_k\in\mathbb{R}^d$ is the mean of component $k$ and $\bSigma$ is a common positive-definite covariance matrix. 

The corresponding observed-data likelihood is
\begin{equation}
\label{eq:gmm-observed-likelihood}
    \mathcal{L}(\bV\mid\bZ,\balpha,\bpi,\bU)
    =
    \prod_{j=1}^m
    \sum_{k=1}^K
    \pi_k
    \int
    p_j(\bff)
    \phi_d(\bff;\bmu_k,\bSigma_d)
    d\bff ,
\end{equation}
where
\[
    p_j(\bff)
    =
    \prod_{i=1}^n
    \sigma(\bff^\top\bz_i+\alpha_i)^{v_{ji}}
    \{1-\sigma(\bff^\top\bz_i+\alpha_i)\}^{1-v_{ji}},
\]
and $\phi_d(\cdot;\bmu,\bSigma)$ denotes the $d$-dimensional Gaussian density with mean $\bmu$ and covariance $\bI_d$. 

Since all components share the same covariance matrix $\bSigma$, consider $\tilde \bF=\bF\bSigma^{-1/2}$, $\tilde \bF$ can be viewed as generated from a Gaussian mixture $\sum_{k=1}^K\pi_k\cN(\bU\bSigma^{-1/2},\bI_d)$. Let $\tilde\bZ=\bZ\bSigma^{1/2}$, and $\tilde \bU=\bU\bSigma^{-1/2}$ then
$$\cL(\bV|\bZ,\balpha,\bpi,\bU)=\cL(\bV|\tilde\bZ,\tilde\balpha,\bpi,\tilde\bU)$$

Thus from the likelihood function, we can't identify $\tilde\bU,\tilde\bZ$ from $\bU,\bZ$. 
To enhance computational efficiency, we fit $\bSigma=\bI_d$ in the estimating stage, thereby eliminating the need for its optimization. Then after the final estimates are derived, we normalize $\bU,\bZ,\balpha$ s.t. they satisfy the regularity condition proposed in \eqref{eq:gmm_identifiability}.

Intuitively, we can directly adjust the idea proposed in the categorical random effect case, choosing both latent group labels $k_j$ and $\{\bff_j\}_{j\in[m]}$ as latent variables, then the complete-data log-likelihood would be
\begin{equation}
\label{eq:gmm-complete-loglik}
\begin{aligned}
    \ell_c
    &:= \ell(\bpi,\bU,\bZ,\balpha|\{k_j\}_{j\in[m]},\{\bff_j\}_{j\in[m]})\\
    &=
    \sum_{j=1}^m
    \left[
        \log\pi_{k_j}
        +
        \log \phi_d(\bff_j;\bmu_{k_j},\bI_d)
        +
        \sum_{i=1}^n
        \left\{
            v_{ji} x_{ji}
            -
            \log(1+e^{x_{ji}})
        \right\}
    \right]
\end{aligned}
\end{equation}
where we introduce $x_{ji}:=\bff_j^\top\bz_i+\alpha_i$ for notational simplicity. 

However, directly computing E-step based on \eqref{eq:gmm-complete-loglik} requires the posterior distribution of $(k_j,\bff_j)$, given the observed hyperlink $\bv_j$. Under the current parameter values
$(\bU,\bZ,\balpha,\bpi)$, this posterior is proportional to
\begin{equation}
\label{eq:gmm-exact-posterior}
    p(k_j=k,\bff_j=\bff\mid\bv_j)
    \propto
    \pi_k
    \phi_d(\bff;\bmu_k,\bI_d)
    \prod_{i=1}^n
    \sigma(\bff^\top\bz_i+\alpha_i)^{v_{ji}}
    \{1-\sigma(\bff^\top\bz_i+\alpha_i)\}^{1-v_{ji}}
\end{equation}
which is not analytically tractable since the logistic likelihood is not conjugate to
the Gaussian random-effect distribution.

We therefore adopt a variational EM approach. Let
\[
    \bpsi=(\bpi,\bU,\bZ,\balpha)
\]
collect the model parameters. For any candidate distribution
$q_j(k,\bff)$ over the latent variables $(k_j,\bff_j)$, we have
\begin{equation}
\begin{aligned}
    \log p_{\bpsi}(\bv_j)
    &=\int_{\bff_j}\sum_{k=1}^K\log p_{\bpsi}(\bv_j)\cdot q_j(k,\bff_j)d\bff_j\\
    &=\int_{\bff_j}\sum_{k_j=1}^K\log\frac{p_{\bpsi}(\bv_j,\bff_j,k_j)}{p_{\bpsi}(\bff_j,k_j|\bv_j)}\cdot q_j(k,\bff_j)d\bff_j\\
    &=\int_{\bff_j}\sum_{k_j=1}^K\log\frac{p_{\bpsi}(\bv_j,\bff_j,k_j)}{q_j(k,\bff_j)}\cdot q_j(k,\bff_j)d\bff_j+\operatorname{KL}
    \left\{
        q_j(k,\bff)
        \,\middle\|\,
        p_{\bpsi}(k_j=k,\bff_j=\bff\mid\bv_j)
    \right\}\\
    &= \mathcal{F}_j(q_j;\bpsi)+\operatorname{KL}
    \left\{
        q_j(k,\bff)
        \,\middle\|\,
        p_{\bpsi}(k,\bff\mid\bv_j)
    \right\}
\end{aligned}\label{eq: gmm-main}
\end{equation}
where
\begin{equation}
    \mathcal{F}_j(q_j;\bpsi)
    =
    \mathbb{E}_{q_j}
    \left[
        \log p_{\bpsi}(\bv_j,\bff_j,k_j)
    \right]
    -
    \mathbb{E}_{q_j}
    \left[
        \log q_j(k_j,\bff_j)
    \right]\label{eq:ELBO}
\end{equation}
is the evidence lower bound. Since the KL divergence is nonnegative,
$\mathcal{F}_j(q_j;\bpsi)$ is a lower bound for $\log p_{\bpsi}(\bv_j)$.

In the variational EM framework, in the E-step, we update the candidate function by minimizing $\operatorname{KL}
    \left\{
        q_j(k,\bff)
        \,\middle\|\,
        p_{\bpsi}(k_j=k,\bff_j=\bff\mid\bv_j)
    \right\}$; then in the M-step, we treat the variational posterior $q_j(k,\bff)$ as fixed, update parameters $\bpsi$ by maximizing the evidence lower bound $\cF_j(q_j;\bpsi)$. Moreover, notice that the second term in \eqref{eq:ELBO} does not depend on the model parameters, thus M-step is equivalent to maximizing $\mathbb{E}_{q_j}
    \left[
        \log p_{\bpsi}(\bv_j,\bff_j,k_j)
    \right]$. Thus, the variational EM procedure alternates between approximating the
posterior distribution of the latent variables and updating the model
parameters using expectations under this approximate posterior.

As mentioned in~\citep{cho2021gaussian}, when the number of nodes $n$ becomes reasonably large, the posterior distribution $p(\bff_j|k_j,\bv_j,\bpsi)$ can be well approximated by a Gaussian distribution. Motivated by this, we restrict the candidate posterior to the
Gaussian-mixture variational family,
namely
\begin{equation}
    q_j(k,\bff)=r_{jk}q_{jk}(\bff),
    \qquad
    q_{jk}(\bff)=\phi_d(\bff;\bm_{jk},\bC_{j})\label{eq:gmm-family}
\end{equation}
Here we assume $q_{jk}$ share the same $\bC_j$ to improve computational efficiency.

However, even this variational EM objective is not directly tractable. The
term
\[
    \log(1+\exp(\bff_j^\top\bz_i+\alpha_i))
\]
appears both in the E-step and M-step, making both steps intractable directly. To address this
obstacle, we utilize the local variational lower bound on the logistic sigmoid function originally proposed by \cite{jaakkola2000bayesian}
\begin{equation}
    \begin{aligned}
        \frac{e^x}{1+e^x}&=\max_{\xi}\frac{e^\xi}{1+e^\xi}\exp \{\frac{(x-\xi)}{2}-\eta (\xi)(x^2-\xi^2)\}\\
        &\geq \frac{e^\xi}{1+e^\xi}\exp \{\frac{(x-\xi)}{2}-\eta (\xi)(x^2-\xi^2)\}
    \end{aligned}\label{variational-bound}
\end{equation}
To be more specific, we introduce variational parameter $\xi_{jik}$ for $\bff_j^{\top}\bz_i+\alpha_i$, and iteratively update $\{\xi_{jik}\}$ in the M-step. This gives a pointwise lower bound on the logistic
log-likelihood and yields the lower-bounded complete-data objective
$\widetilde{\ell}$ defined as:
\begin{equation}
\label{eq:gmm-elbo-after-bound}
\begin{aligned}
    \log P(\bv_j,\bff_j,k_j|\pi,\bZ,\balpha,\bU)&=\Sigma_{i=1}^n\{v_{ji}(\bff_j^{\top}\bz_z+\alpha_i)-\log [1+\exp (\bff_j^{\top}\bz_i+\alpha_i)] \}\\
    & \ \ \ +\log p(\bff_j|k_j,\pi,\bZ,\balpha,\bU)+\log p(k_j|\bpi)\\
    &\geq \Sigma_{i=1}^n\{\log\frac{e^{\xi_{jik_j}}}{1+e^{\xi_{jik_j}}}+(v_{ji}-1/2)(\bff_j^{\top}\bz_i+\alpha_i)-\frac{1}{2}\xi_{jik_j}\\
    &\qquad -\eta(\xi_{jik_j})[(\bff_j^{\top}\bz_i+\alpha_i)^2-\xi_{jik_j}] \}\\
    &\qquad+\log p(\bff_j|k_j,\pi,\bZ,\balpha,\bU)+\log p(k_j|\pi)\\
    &:=\tilde l(\bxi_j,\bpi,\bZ,\balpha,\bU|\bff_j,k_j,\bv_j)
\end{aligned}
\end{equation}
where $\bxi_j=\{\xi_{jik}\}_{i\in[n],k\in[K]}$.

Then in the E-step, notice that once $\bpsi,\bXi:=\{\xi_{jik}\}$ are fixed,
$$\begin{aligned}
    &\quad\text{minimize} \operatorname{KL}
    \left\{
        q_j(k,\bff)
        \,\middle\|\,
        p_{\bpsi}(k,\bff\mid\bv_j)
    \right\}\\
    &\Leftrightarrow \text{minimize}\ \E_{q_j}[\log q_j(k,\bff)]-\E_{q_j}[\log P(\bv_j,\bff_j,k_j|\pi,\bZ,\balpha,\bU)]
\end{aligned}$$

Utilizing \eqref{eq:gmm-elbo-after-bound}, instead of directly minimizing 
$$\E_{q_j}[\log q_j(k,\bff)]-\E_{q_j}[\log P(\bv_j,\bff_j,k_j|\pi,\bZ,\balpha,\bU)]$$
we minimize its upper bound 
$$
\begin{aligned}
    &\quad\E_{q_j}[\log q_j(k,\bff)]-\E_{q_j}[\tilde l(\bxi_j,\bpi,\bZ,\balpha,\bU|\bff_j,k_j,\bv_j)]\\
    &= \sum_{k=1}^kr_{jk}\{-\frac{1}{2}\log|\bC_j|-\sum_{i\in[n]}(v_{ji}-\frac{1}{2})\bm_{jk}^{\top}\bz_i\\
    &\quad +\sum_{i\in[n]}\eta(\xi_{jki})[\bz_i^{\top}(\bm_{jk\bm_{jk}^{\top}+\bC_j})\bz_i] +2\alpha_i\bm_{jk}^{\top}\bz_i\}\\
    &\quad +\frac{1}{2}\sum_{k\in[K]}r_{jk}[\Tr(\bm_{jk}\bm_{jk}^{\top}+\bC_j)-2\bm_{jk}^{\top}\bmu_k]+C
\end{aligned}$$
where $C$ is a constant does not depend on $\{\bm_{jk},\bC_j\}$.

Let $\eta_{jik}:=\eta(\xi_{jik})$, for fixed global parameters $(\bU,\bZ,\balpha,\bpi)$ and local parameters $\{\xi_{jik}\}$, through direct calculation, E-step updates variational posterior through
\begin{equation}
\label{eq:gmm-e-cov}
\begin{aligned}
    \mathbf{C}_{j}^{-1} &=\bI_d+2\sum_{k=1}^K\sum_{i=1}^n
    r_{jk}\eta_{jik}\bz_i\bz_i^\top \\
    \mathbf{m}_{jk}&=(\bI_d+2\sum_{i=1}^n\eta_{jik}\bz_i\bz_i^\top)^{-1}
   [
        \bmu_k
        +
        \sum_{i=1}^n
        \{
            v_{ji}
            -
            \frac12
            -
            2\eta_{jik}\alpha_i
        \}
        \bz_i].\\
        r_{jk}&=\frac{h_jk}{\sum_{t\in[K]}h_{jt}}
    \end{aligned}
\end{equation}
where 
$$\begin{aligned}
    h_{jk}&:=\exp(g_{jK}-g_{jk})\\
    g_{jk}&=-\log\pi_k-\frac{1}{2}\bm_{jk}^{\top}\bm_{jk}+\frac{1}{2}\bmu_k^{\top}\bmu_k\\
    &\quad -\Sigma_{i=1}^n[\log \frac{e^{\xi_{jik}}}{1+e^{\xi_{jik}}}-\frac{1}{2}\xi_{jik}-\eta(\xi_{jik})(\alpha_i^2-\xi_{jik}^2)]+\sum_{i=1}^n\eta(\xi_{jik})\bz_i^{\top}(\bC_j-\bm_{jk}\bm_{jk}^{\top})\bz_i
\end{aligned}$$

Similarly, in the M-step, we optimizes the tractable lower bound
 \label{eq:gmm-expected-lower-bound}
\begin{equation}
\begin{aligned}
\widetilde{\mathcal{F}}_j(q_j,\bxi_j;\bpsi)
    &=
    \mathbb{E}_{q_j}
    \left[
        \widetilde{\ell}
        (\bxi_j,\bpi,\bZ,\balpha,\bU
        \mid \bff_j,k_j,\bv_j)
    \right]\\
    &=\sum_{j=1}^m
    \sum_{k=1}^K
    r_{jk}
    \int
    \widetilde{\ell}(\bxi_j,\bpi,\bZ,\balpha,\bU|\bff_j,k,\bv_j)q_{jk}(\bff)d\bff \\         &=\sum_{j=1}^m\sum_{k=1}^K\sum_{i=1}^nr_{jk}\log\frac{e^{\xi_{jik}}}{1+e^{\xi_{jik}}}+\sum_{j=1}^m
    \sum_{k=1}^K
    \sum_{i=1}^n
    r_{jk}
    \left(v_{ji}-\frac12\right)
    \left(
        \bm_{jk}^{\top}\bz_i+\alpha_i
    \right)\\
    &\quad
    -
    \frac12
    \sum_{j=1}^m
    \sum_{k=1}^K
    \sum_{i=1}^n
    r_{jk}\xi_{jik}\\
    &\quad -
    \sum_{j=1}^m
    \sum_{k=1}^K
    \sum_{i=1}^n
    r_{jk}\eta(\xi_{jik})
    \left[
        \bz_i^{\top}\bS_{jk}\bz_i
        +
        2\alpha_i\bm_{jk}^{\top}\bz_i
        +
        \alpha_i^2
        -
        \xi_{jik}^2
    \right]  \\                    
    &\quad
    -
    \frac12
    \sum_{j=1}^m
    \sum_{k=1}^K
    r_{jk}
    \left[
        \operatorname{tr}(\bS_{jk})
        -
        2\bm_{jk}^{\top}\bmu_k
        +
        \bmu_k^{\top}\bmu_k
    \right]                        
    +
    \sum_{j=1}^m
    \sum_{k=1}^K
    r_{jk}\log\pi_k .
\end{aligned}
\end{equation}

Through direct computation, the parameters can be updated using closed-form expressions presented below: $\forall j\in[m],i\in[n],k\in[K]$ 
\begin{align}
    \bz_i
    &=
    \frac{1}{2}
[
        \sum_{j=1}^m
        \sum_{k=1}^K
        \eta(\xi_{jik})
        r_{jk}
        \left(
            \bC_{jk}
            +
            \bm_{jk}\bm_{jk}^{\top}
        \right)]^{-1}
    \sum_{j=1}^m
    \sum_{k=1}^K
    r_{jk}
    [(v_{ji}-\frac12)
        -
        2\eta(\xi_{jik})\alpha_i]
    \bm_{jk}
    \label{eq:gmm-update-z} \\
    \alpha_i
    &=
    \frac{
        \sum_{j=1}^m
        \sum_{k=1}^K
        r_{jk}
        [
            v_{ji}
            -
            \frac12
            -
            2\eta(\xi_{jik})
            \bm_{jk}^{\top}\bz_i]
    }{
        2
        \sum_{j=1}^m
        \sum_{k=1}^K
        \eta(\xi_{jik})r_{jk}
    },
    \label{eq:gmm-update-alpha} \\
    \pi_k
    &=
    \frac{1}{m}
    \sum_{j=1}^m
    r_{jk},
    \qquad k\in[K],
    \label{eq:gmm-update-pi} \\[2mm]
    \bmu_k
    &=
    \frac{
        \sum_{j=1}^m
        r_{jk}\bm_{jk}
    }{
        \sum_{j=1}^m
        r_{jk}
    }
    \label{eq:gmm-update-mu} \\
    \xi_{jik}^{2}
    &=
    \bz_i^{\top}
    \bm_{jk}\bm_{jk}^{\top}
    \bz_i
    +
    \bz_i^{\top}
    \bC_{jk}
    \bz_i
    +
    2\alpha_i
    \bm_{jk}^{\top}\bz_i
    +
    \alpha_i^2
    \label{eq:gmm-update-xi}
\end{align}
The resulting variational EM procedure is summarized in Algorithm~\ref{alg:MGVEM}.

\begin{algorithm}[t]
\caption{Variational EM for Gaussian-mixture hyperlink random effects}
\label{alg:MGVEM}
\small
\begin{algorithmic}[1]
\Require $\bV$; component number $K$; embedding dimension $d$; initialization $\bZ^{(0)},\balpha^{(0)},\bpi^{(0)},\bU^{(0)},\bXi^{(0)}$; tolerance $\epsilon$
\Ensure $\widetilde{\bU},\widetilde{\bZ},\widetilde{\balpha},\hat{\bpi}$
\Repeat
\State E-step: update $r_{jk}^{(t)}$, $\bm_{jk}^{(t)}$, and $\bC_{jk}^{(t)}$ for $j\in[m]$ and $k\in[K]$ using the current parameters, according to \eqref{eq:gmm-e-cov}.
\State M-step: update $\bZ$, $\balpha$, $\bpi$, $\bU$, and $\bXi$ by maximizing the variational lower bound, using \eqref{eq:gmm-update-z}-\eqref{eq:gmm-update-xi}.
\Until{$\|\bZ^{(t+1)}-\bZ^{(t)}\|_F\le\epsilon$}
\State Set $\widehat{\bZ}$, $\widehat{\balpha}$, $\widehat{\bpi}$, $\widehat{\bPi}$, $\widehat{\bU}$, and $\widehat{\bSigma}$ to the last iterates.
\State Center: set $\bar\bmu=\widehat{\bU}^\top\widehat{\bpi}$, $\bU_0=\widehat{\bU}-\bone_K\bar\bmu^\top$, and $\alpha_{0i}=\widehat{\alpha}_i+\bar\bmu^\top\widehat{\bz}_i$ for $i\in[n]$.
\State Whiten: 
set $\widehat{\bSigma}   =
    \bI_d
    +
    \sum_{k=1}^K
    \widehat{\pi}_k
    (\widehat{\bmu}_k-\bar{\bmu})
    (\widehat{\bmu}_k-\bar{\bmu})^{\top}
    =
    \bI_d+\bU_0^{\top}\widehat{\bPi}\bU_0$, let $
    \bU_1
    =
    \bU_0\widehat{\bSigma}^{-1/2}, \bZ_1=
    \widehat{\bZ}\widehat{\bSigma}^{1/2}.
$
\State Rotate: eigendecompose $\hat\bZ_1^\top\hat\bZ_1=\bQ_0\bLambda\bQ_0^\top$ with decreasing diagonal entries in $\bLambda$; return $\widetilde{\bU}=\bU_1\bQ_0$, $\widetilde{\bZ}=\bZ_1\bQ_0$, $\widetilde{\balpha}=(\alpha_{01},\ldots,\alpha_{0n})^\top$, and $\hat{\bpi}$.
\end{algorithmic}
\end{algorithm}
\subsection{Score-based Random Effects}\label{append:alg:score}
We finally consider a flexible specification of the random-effect distribution
$\mathbb{P}_{\bF}$. Instead of assuming a categorical distribution or a
Gaussian mixture, we model $\mathbb{P}_{\bF}$ by a score-based generative
model with parameter $\btheta$, denoted by $\mathbb{P}_{\bF,\btheta}$. Under
this model, each hyperlink $j\in[m]$ has a latent random effect
$\bff_j\in\mathbb{R}^d$, where
\[
    \bff_j\sim \mathbb{P}_{\bF,\btheta}.
\]
The observed-data likelihood is
\begin{equation}
\label{eq:score-observed-likelihood}
    \mathcal{L}(\bV\mid\bZ,\balpha,\btheta)
    =
    \prod_{j=1}^m
    \int
    p_j(\bff)\,
    d\mathbb{P}_{\bF,\btheta}(\bff).
\end{equation}

This formulation is mainly motivated by latent score-based generative modeling
\citep{vahdat2021score}. In latent score-based generative models, a diffusion
model is trained on the latent variables produced by a variational autoencoder,
with an encoder mapping data to latent space and a decoder mapping latent
variables back to data space. Our setting differs in two important respects.
First, the  variables $\{\bff_j\}_{j=1}^m$ are already part of the
statistical model, so no separate neural encoder is needed. Posterior
inference for $\bff_j$ can be performed using the likelihood of the observed
incidence vector $\bv_j$. Second, the observation model is the logistic
hyperlink model specified above, so no separate neural decoder is required.
Thus, we only borrow the score-based latent-prior training machinery, while
the encoder and decoder are replaced by model-based posterior inference and
the logistic incidence likelihood, respectively.

Let
\[
    \bF=(\bff_1,\ldots,\bff_m)^\top\in\mathbb{R}^{m\times d}
\]
denote the matrix of hyperlink random effects. 

Similar to what we have mentioned in the Gaussian mixture case, we still adopt the variational EM framework, treating $\bF$ as latent variables. Then the complete log-likelihood is
\begin{equation}
\label{eq:score-complete-loglik}
\begin{aligned}
    \ell(\bZ,\balpha,\btheta\mid\bV,\bF)
    =
    \sum_{j=1}^m
    \sum_{i=1}^n
    \left[
        v_{ji}(\bff_j^\top\bz_i+\alpha_i)
        -
        \log\{1+\exp(\bff_j^\top\bz_i+\alpha_i)\}
    \right]
    +
    \sum_{j=1}^m
    \log p_{\btheta}(\bff_j),
\end{aligned}
\end{equation}
where $p_{\btheta}$ denotes the density associated with
$\mathbb{P}_{\bF,\btheta}$ when it exists.

Notice that for different $j\in[m]$, even though those $\{\bff_j\}$ are still conditionally independent given $\bphi$, the posterior distributions $q(\bff_j|\bV,\phi)=q(\bff_j|\bv_j,\phi)$ are different. Thus we denote the posterior distribution for $\bff_j$ as $q^j_{\bphi}(\cdot|\bv)$ for notational simplicity. And we have
\begin{equation}
\label{eq:score-exact-posterior}
    q^j_{\bphi^{(t)}}(\bff_j\mid\bv_j)
    \propto
    p_{\btheta^{(t)}}(\bff_j)
    \prod_{i=1}^n
    \sigma(x_{ji}^{(t)})^{v_{ji}}
    \{1-\sigma(x_{ji}^{(t)})\}^{1-v_{ji}},
    \qquad
    x_{ji}^{(t)}
    =
    \bff_j^\top\bz_i^{(t)}+\alpha_i^{(t)}.
\end{equation}

Following the variational EM framework, we introduce candidate densities $q_j(\bff)$ to approximate $q_{\bphi}^j(\bff|\bv)$. And since the logistic term still brings challenges, we utilize the local variational lower bound shown in \eqref{variational-bound} again, by introducing variational parameters $\xi_{ji}$ for $\bff_j^{\top}\bz_i+\alpha_i$. Then the complete log-likelihood in \eqref{eq:score-complete-loglik} can be lower bounded pointwise by
\begin{equation}
\label{eq:score-lower-complete}
\begin{aligned}
    \widetilde{\ell}
    (\bZ,\balpha,\btheta,\bXi\mid\bV,\bF)
    &:=
    \sum_{j=1}^m
    \sum_{i=1}^n
        [\log\sigma(\xi_{ji})
        +
        \left(v_{ji}-\frac12\right)
        (\bff_j^\top\bz_i+\alpha_i)
        -
        \frac{\xi_{ji}}{2}]  \\
    &\qquad-\sum_{j=1}^m\sum_{i=1}^n\eta(\xi_{ji})
        \left\{
            (\bff_j^\top\bz_i+\alpha_i)^2
            -
            \xi_{ji}^2
        \right\}
    +
    \sum_{j=1}^m
    \log p_{\btheta}(\bff_j).
\end{aligned}
\end{equation}

In the E-step, given parameters $\bphi =(\bZ,\balpha,\btheta)$, we need to update $q_j(\bff)$. By the fundamental result in the theory of variational inference \citep{bishop2006pattern}, the approximate posterior $q_k(z_k)$ that maximizes the lower bound on the log-marginal likelihood (i.e. ELBO) takes a specific form
$$\log q_k^*(z_k)=\E_{j\neq k}[\log p(\bx,\bz)]+\text{constant}$$
for latent variable $\bz$. In our case, we consider maximizing the surrogate ELBO by replacing ${\ell}
    (\bZ,\balpha,\btheta,\bXi\mid\bV,\bF)$ with $\tilde \ell$ defined in \eqref{eq:score-lower-complete}
Then following the theory, the best $q_j(\bff)$ would satisfy
$$\begin{aligned}
    \log q_j^*(\bff)&=\E_{q_{j'\neq j}}\big\{\sum_{j=1}^m
    \sum_{i=1}^n
        [\log\sigma(\xi_{ji})
        +
        \left(v_{ji}-\frac12\right)
        (\bff_j^\top\bz_i+\alpha_i)
        -
        \frac{\xi_{ji}}{2}]  \\
    &\qquad-\sum_{j=1}^m\sum_{i=1}^n\eta(\xi_{ji})
        \left\{
            (\bff_j^\top\bz_i+\alpha_i)^2
            -
            \xi_{ji}^2
        \right\}
    +
    \sum_{j=1}^m
    \log p_{\btheta}(\bff_j)\big\}+C\\
    &=\sum_{i=1}^n(v_{ji}-\frac{1}{2})(\bff_j^{\top}\bz_i)-\sum_{i=1}^n\eta(\xi_{ji})(\bff_j^{\top}\bz_i+\alpha_i)^2+\log p_{\btheta}(\bff_i)+C
\end{aligned}$$
which indicates that 
\begin{equation}
\label{eq:score-tilted-posterior}
    q_j^*(\bff)
    \propto
    p_{\btheta}(\bff)
    \exp
    \left\{
        \bb_j^\top\bff
        -
        \frac12\bff^\top\bA_j\bff
    \right\},
\end{equation}
where
\begin{equation}
\label{eq:score-A-b}
    \bA_j
    =
    2\sum_{i=1}^n
    \eta(\xi_{ji})\bz_i\bz_i^\top,
    \qquad
    \bb_j
    =
    \sum_{i=1}^n
    \left[
        v_{ji}
        -
        \frac12
        -
        2\eta(\xi_{ji})\alpha_i
    \right]\bz_i .
\end{equation}

We thus set $q_j(\bff)$ to be $q^*_j(\bff)$, of the form shown in \eqref{eq:score-tilted-posterior} in the E-step. As we will show later in \eqref{eq:score_qfunction}, $q_j(\bff)$ would influence the M-step only through $$\bm_j:=\E_{q_j}[\bff_j], \bS_j:=\E[\bff_j\bff_j^{\top}]$$

\textbf{Thus in the E-step}, we only update $\bm_j,\bS_j$ by drawing samples from $q_j(\cdot)$ and then estimate it through sample mean and covariance, which can be achieved by adopting the prior-sampling reweighting methods. Specifically, for $\forall j\in[m]$, we draw samples from the diffusion prior $$\bff^{(1)},\dots,\bff^{(B)}\sim p_{\btheta}(\bff)$$
then we compute the weights 
    $$\log w_{jb}=\bb_i^{\top}\bff^{(b)}-\frac{1}{2}\bff^{(b)\top}A_i\bff^{(b)}$$
    and normalize the weights using 
    \begin{equation}
    \bar w_{jb}=\frac{\exp(\log w_{jb})}{\sum_{l=1}^B\exp(\log w_{jl})}\label{eq:sample-weight}
    \end{equation}
 We finally estimate $\bm_j$ and $\bS_j$ as 
    \begin{equation}
        \hat\bm_j=\sum_{b=1}^B \bar{w}_{jb} \bff^{(b)}, \quad \hat{\bS}_j=\sum_{b=1}^B\bar{w}_{jb} \bff^{(b)} \bff^{(b) \top}\label{eq:score-estep}
    \end{equation}

\textbf{Then in the M-step:}

Same as in section \ref{append:alg:gm}, we optimizes the tractable lower bound, i.e. the expectation of $\tilde l(\bZ,\balpha,\btheta,\bXi|\bV,\bF)$ shown in \eqref{eq:score-lower-complete}, with expectation taken w.r.t the approximate posterior. This would give us 
\begin{equation}
    \begin{aligned}
\tilde Q(\bZ,\balpha,\btheta,\bXi)&:=\sum_{j=1}^m\sum_{i=1}^n\{\log\sigma(\xi_{ji})+(v_{ji}-\frac{1}{2})(\bm_j^{\top}\bz_i+\alpha_i)-\frac{1}{2}\xi_{ji}\}\\
&\quad-\sum_{j=1}^m\sum_{i=1}^n\eta(\xi_{ji})[\bz_i^{\top}(\bS_j+\bm_j\bm_j^{\top})\bz_i+\alpha_i^2+2\alpha_i\bm_j^{\top}\bz_i-\xi_{ji}^2]\}\\
&\quad+\sum_{j=1}^m\E_{q_j{(\bff_j)}}[\log p_{\btheta}(\bff_j)]
    \end{aligned}\label{eq:score_qfunction}
\end{equation}

With $\{\hat\bm_j,\hat\bS_j\}_{j\in[m]}$, $(\bZ\,\balpha,\bXi)$ can be updated directly by taking derivatives w.r.t \eqref{eq:score_qfunction}, which gives the following closed-form coordinate updates:
\begin{align}
    \bz_i
    &=
    \frac12
    \left[
        \sum_{j=1}^m
        \eta(\xi_{ji})\widehat{\bS}_j
    \right]^{-1}
    \sum_{j=1}^m
    \left[
        v_{ji}
        -
        \frac12
        -
        2\eta(\xi_{ji})\alpha_i
    \right]
    \widehat{\bm}_j,
    \qquad i\in[n],
    \label{eq:score-update-z}
    \\[1mm]
    \alpha_i
    &=
    \frac{
        \sum_{j=1}^m
        \left[
            v_{ji}
            -
            \frac12
            -
            2\eta(\xi_{ji})
            \widehat{\bm}_j^\top\bz_i
        \right]
    }{
        2\sum_{j=1}^m\eta(\xi_{ji})
    },
    \qquad i\in[n],
    \label{eq:score-update-alpha}
    \\[1mm]
    \xi_{ji}^2
    &=
    \bz_i^\top\widehat{\bS}_j\bz_i
    +
    2\alpha_i\widehat{\bm}_j^\top\bz_i
    +
    \alpha_i^2,
    \qquad j\in[m],\ i\in[n].
    \label{eq:score-update-xi}
\end{align}
Equivalently, one may update $(\bz_i,\alpha_i)$ jointly. Let
\[
    \widetilde{\bw}_i
    =
    \begin{pmatrix}
        \bz_i\\
        \alpha_i
    \end{pmatrix},
    \qquad
    \widetilde{\bm}_j
    =
    \begin{pmatrix}
        \widehat{\bm}_j\\
        1
    \end{pmatrix},
    \qquad
    \widetilde{\bS}_j
    =
    \begin{pmatrix}
        \widehat{\bS}_j & \widehat{\bm}_j\\
        \widehat{\bm}_j^\top & 1
    \end{pmatrix}.
\]
Then
\begin{equation}
\label{eq:score-joint-node-update}
    \widetilde{\bw}_i
    =
    \left[
        2\sum_{j=1}^m
        \eta(\xi_{ji})\widetilde{\bS}_j
    \right]^{-1}
    \left[
        \sum_{j=1}^m
        \left(v_{ji}-\frac12\right)
        \widetilde{\bm}_j
    \right],
    \qquad i\in[n].
\end{equation}
This joint update is often more stable than alternating between
\eqref{eq:score-update-z} and \eqref{eq:score-update-alpha}.

For the score-based prior parameter $\btheta$, we 
$$\text{minimize }\sum_{i=j}^m\E_{q_j{\bff_j}}[-\log p_{\btheta}(\bff_i)]$$
Since the density $p_{\btheta}$ is generally not evaluated in closed form, we do not
maximize $\sum_j\mathbb{E}_{q_j}\log p_{\btheta}(\bff_j)$ directly. Instead, at current stage $t$,
we train the score model using samples from approximate posterior
$\{q_j^{(t)}(\cdot)\}_{j=1}^m$. These samples are obtained by sampling-importance-resampling, same as we did in the E-step, for each $j$, we resample from
$\{\bff^{(b)}\}_{b=1}^B$ according to the normalized weights
$\{\bar\omega_{jb}\}_{b=1}^B$ shown in \eqref{eq:sample-weight}.

Given posterior samples $\bff_{0,j}$, we can directly refine the Likelihood training with importance sampling algorithm (Algorithm 1) in~\citep{vahdat2021score} to update $\btheta$, using a variance-preserving SDE (VPSDE~\citep{song2020score}). This step is summarized in Algorithm \ref{alg:LSGM}.

\begin{algorithm}[t]
\caption{Likelihood-weighted training of the score-based random-effect prior}
\label{alg:LSGM}
\small
\begin{algorithmic}[1]
\Require $\bV$; current $\bZ^{(t)},\balpha^{(t)}$; current posterior approximates $\{q_{_j}^{(t)}(\bff\mid\bv_j)\}_{j=1}^m$; current score parameter $\btheta^{(t)}$
\Ensure Updated score parameter $\btheta$
\State Draw $\bff_0^j\sim q_{j}^{(t)}(\bff\mid\bv_j)$ independently for $j\in[m]$.
\State Draw $\tau\sim r_{\mathrm{ll}}(\tau)$, compute the VPSDE mean $\bmu_\tau(\bff_0^j)$ and variance $\sigma_\tau^2$, and draw $\bff_\tau^j=\bmu_\tau(\bff_0^j)+\sigma_\tau\bepsilon^j$ with $\bepsilon^j\overset{i.i.d.}{\sim}\mathcal{N}(\bzero_d,\bI_d)$.
\State Update $\btheta$ by minimizing $\sum_{j=1}^m\{w_{\mathrm{ll}}(\tau)/(2r_{\mathrm{ll}}(\tau))\}\|\bepsilon^j-\bepsilon_{\btheta}(\bff_\tau^j,\tau)\|_2^2$ using stochastic gradient descent.
\end{algorithmic}
\end{algorithm}

To be more specific, we apply the standard latent score-matching objective. For a variance-preserving diffusion process, let
\[
    \bff_t
    =
    \ba_{\tau}(\bff_{0,j})
    +
    \sigma_t\bepsilon,
    \qquad
    \bepsilon\sim\mathcal{N}(\bzero_d,\bI_d),
\]
where $\ba_{\tau}$ and $\sigma_t^2$ are determined by the VPSDE schedule. Then we update $\btheta$ by minimizing
\begin{equation}
\label{eq:score-dsm-loss}
    \mathcal{L}_{\mathrm{DSM}}(\btheta)
    =
    \mathbb{E}
    \left[
        \frac{w_{\mathrm{ll}}(t)}{2r_{\mathrm{ll}}(t)}
        \left\|
            \bepsilon
            -
            \bepsilon_{\btheta}(\bff_t,t)
        \right\|_2^2
    \right],
\end{equation}
where $t\sim r_{\mathrm{ll}}(t)$ and $w_{\mathrm{ll}}(t)$ is the likelihood
weighting function. In practice, we adopt score-based model for this density and update $\btheta$ through a
denoising score-matching objective using samples from the approximate
posterior distribution of the latent random effects.

This is the same score-training mechanism used for latent
score-based generative models~\cite{vahdat2021score}, except that the training samples here are
drawn from the model-based posterior distribution of the hyperlink random
effects rather than from a VAE encoder. In the coding implementation, we adjusted the code from \citep{vahdat2021score}\footnote{\href{https://nvlabs.github.io/LSGM/}https://nvlabs.github.io/LSGM/} to match our Algorithm \ref{alg:HRLScore}.

The full variational SDE algorithm for Score-based random effects is summarized in Algorithm \ref{alg:HRLScore}.
\begin{algorithm}[t]
\caption{Monte Carlo variational EM for score-based hyperlink random effects}
\label{alg:HRLScore}
\small
\begin{algorithmic}[1]
\Require $\bV\in\{0,1\}^{m\times n}$; latent dimension $d$;
initial values $\bZ^{(0)}$, $\balpha^{(0)}$, $\bXi^{(0)}$, and
$\btheta^{(0)}$; posterior Monte Carlo size $B$; tolerance $\epsilon$.
\Ensure $\widehat{\bZ}$, $\widehat{\balpha}$, and $\widehat{\btheta}$.

\Repeat
    \State \textbf{E-Step:} \begin{itemize}
        \item For each hyperlink $j\in[m]$, compute $\bA_j^{(t)}$ and $\bb_j^{(t)}$ using
    \eqref{eq:score-A-b}.
    \item Draw prior samples $\bff^{(1)},\ldots,\bff^{(B)}\sim\mathbb{P}_{\bF,\btheta^{(t)}}$ using the reverse diffusion sampler.
    \item For each $j\in[m]$, compute normalized importance weights
    $\{\omega_{jb}^{(t)}\}_{b=1}^B$ and estimate
    $\widehat{\bm}_j^{(t)}$ and $\widehat{\bS}_j^{(t)}$ using
    \eqref{eq:score-estep}.
    \end{itemize}

    \State \textbf{M-Step:} 
    \begin{enumerate}
        \item For each node $i\in[n]$, update
    $(\bz_i^{(t+1)},\alpha_i^{(t+1)})$ using the joint update
    \eqref{eq:score-joint-node-update}, or equivalently the coordinate
    updates \eqref{eq:score-update-z} and \eqref{eq:score-update-alpha}.
    \item Update
    $\xi_{ji}^{(t+1)}$ for all $j\in[m]$ and $i\in[n]$ using \eqref{eq:score-update-xi}.
    \item Update $\btheta^{(t+1)}$ following Algorithm \ref{alg:LSGM}
    \end{enumerate}
\Until{$\|\bZ^{(t+1)}-\bZ^{(t)}\|_F\leq \epsilon$}

\State Set $\widehat{\bZ}$, $\widehat{\balpha}$, and $\widehat{\btheta}$ to
the final iterates.

\State Center and whiten: draw samples
$\bff^{(1)},\ldots,\bff^{(B_0)}\sim\mathbb{P}_{\bF,\widehat{\btheta}}$ and
estimate
\[
    \widehat{\bm}_{\bF}
    =
    \frac{1}{B_0}
    \sum_{b=1}^{B_0}
    \bff^{(b)},
    \qquad
    \widehat{\bSigma}_{\bF}
    =
    \frac{1}{B_0}
    \sum_{b=1}^{B_0}
    \left(\bff^{(b)}-\widehat{\bm}_{\bF}\right)
    \left(\bff^{(b)}-\widehat{\bm}_{\bF}\right)^\top .
\]
Set $\alpha_{0i} = \widehat{\alpha}_i
    +
    \widehat{\bm}_{\bF}^{\top}\widehat{\bz}_i,
    \bZ_1
    =
    \widehat{\bZ}\widehat{\bSigma}_{\bF}^{1/2}.$
\State Rotate: eigendecompose $\bZ_1^\top\bZ_1=\bQ_0\bLambda\bQ_0^\top$ 
with decreasing diagonal entries in $\bLambda$, then set \ \State Return $\widetilde{\bZ}
    =
    \bZ_1\bQ_0,
    \widetilde{\balpha}
    =
    (\alpha_{01},\ldots,\alpha_{0n})^\top$
\end{algorithmic}
\end{algorithm}

\subsection{Initialization}\label{sec:initialization}
In general, EM algorithm is sensitive to initial parameter values. We therefore propose initialization schemes for all three algorithms, mainly based on Universal Singular Value Thresholding (USVT) \citep{chatterjee2015matrix}, followed by $K$-means clustering. 

\paragraph{Initialization for categorical random effects} In words, the initialization procedure first estimates an $m\times n$ probability matrix for the observed incidence matrix without imposing the discrete latent-type structure. After applying the logit transformation, we remove node-specific intercepts and compute a rank-$d$ approximation. This gives preliminary embeddings for all $m$ hyperlinks and all $n$ nodes. We then apply kmeans to the hyperlink embeddings, treating the resulting cluster centers as initial estimates of the $K$ support points $\{\bff_k\}_{k=1}^K$. This kmeans step also allows for the derivation of an initial estimate of the $K\times n$ probability matrix $\bQ^0$ by setting $q_{ki}^0=\mathrm{expit}(\bff_k^{0\top}\bz_i^0+\alpha_i^0),\forall k\in[K], i\in [n]$. The whole initialization procedure is summarized in Algorithm \ref{alg:EM_initial}.

\begin{algorithm}[t]
\caption{Initialization of Algorithm \ref{alg:EM} based on USVT and K-means}
\label{alg:EM_initial}
\small
\begin{algorithmic}[1]
\Require Data: $\bV$; total number of groups: $K$; dimension of the latent variable: $d$; threshold: $\tau$; lower bound: $a$; upper bound: $b$
\Ensure $\bF^0,\bZ^0,\balpha^0,\bpi^0$
\State Let $\tilde \bP=\sum_{s_j\geq \tau}s_j\bu_j\bv_j^{\top}$, where $\sum s_j\bu_j\bv_j^{\top}$ is the SVD of $\bV$. Elementwisely project $\tilde\bP$ to the interval [a,b] to obtain $\hat\bP$. Let $\hat\bTheta=\text{logit}(\hat\bP)$.
\State Let $\balpha_0=\argmin_{\balpha}\left\|\hat\bTheta-\bone_m\balpha^{\top}\right\|_F^2$.
\State Let $\hat\bG=\hat\bTheta-\bone_m\balpha_0^{\top},\bU^0=\bU_d\bD_d^{1/2},\bZ^0=\bV_d\bD_d^{1/2}$, where $\bU_d\bD_d\bV_d^{\top}$ is the top-$d$ singular value component of $\hat\bG$.
\State Run k-means with group number $M$ on the rows of $\bU^0$ to get $M$ centers, and the centers are considered as $\bF^0$, the membership distribution are considered as $\bpi^0$.
\State Normalize $\bF^0,\bZ^0,\balpha^0$ in exactly the way as stated in Algorithm \ref{alg:EM}.
\State Let $\bQ^0=\mathrm{expit}(\bF^0\bZ^{0\top}+\bone_K\balpha^{0\top})$.
\end{algorithmic}
\end{algorithm}

\paragraph{Initialization for Gaussian mixture} For the Gaussian mixture case, the initialization steps is quite similar to that in the categorical, which is presented in Algorithm \ref{alg:MGVEM_initial}. The only differences are that now we additionally set initial values for $\{r_{jk}\}$ as the posterior probability that hyperlink $j$ belongs to group $k$, and we initialize $\xi_{jik}$ by drawing samples from the $k$-th initial Gaussian component, trying to follow the idea that $\xi_{jik}$ is introduced for $\bff_j^{\top}\bz_i+\alpha_i|k_j=k$.

\begin{algorithm}[t]
\caption{Initialization of Algorithm \ref{alg:MGVEM}}
\label{alg:MGVEM_initial}
\small
\begin{algorithmic}[1]
\Require Data: $\bV\in\mathbb{R}^{m\times n}$; total number of groups: $K$; dimension of the latent variable: $d$; threshold: $\tau$; lower bound: $a$; upper bound: $b$
\Ensure $\bZ^0,\balpha^0,\bpi^0,\bU^{(0)},\{r_{jk}^0\},\bXi^0$

\State Let $\tilde \bP=\sum_{s_j\geq \tau}s_j\bu_j\bv_j^{\top}$,
where $\sum_j s_j\bu_j\bv_j^{\top}$ is the SVD of $\bV$. Elementwisely project $\tilde\bP$ to the interval $[a,b]$ to obtain $\hat\bP$. Let $\hat\bTheta=\operatorname{logit}(\hat\bP)$.

\State Set $\balpha_0
=
\argmin_{\balpha\in\mathbb{R}^n}
\left\|
\hat\bTheta-\bone_m\balpha^{\top}
\right\|_F^2$.

\State Let $\hat\bG=\hat\bTheta-\bone_m\balpha_0^{\top}, \bQ^0=\bU_d\bD_d^{1/2}, \bZ^0=\bV_d\bD_d^{1/2}$,
where $\bU_d\bD_d\bV_d^{\top}$ is the top-$d$ singular value component of $\hat\bG$.

\State Run K-means with group number $K$ on the rows of $\bQ^0$ to obtain clusters $\mathcal{C}_1,\ldots,\mathcal{C}_K$. Let $\bq_j^0$ denote the $j$-th row of $\bQ^0$, for $j\in[m]$.

\State For each $k\in[K]$, set $\pi_k^0=\frac{|\mathcal{C}_k|}{m}, \bmu_k^0
=
\frac{1}{|\mathcal{C}_k|}
\sum_{i\in\mathcal{C}_k}\bq_j^0$
and $\bSigma_k^0
=
\frac{1}{|\mathcal{C}_k|}
\sum_{j\in\mathcal{C}_k}
(\bq_j^0-\bmu_k^0)(\bq_j^0-\bmu_k^0)^{\top}$.

\State Normalize $\bU^{0}, \bZ^0, \balpha^0$ in the same way as stated in Algorithm~\ref{alg:MGVEM}.

\State For each $k\in[K]$, randomly draw $m$ samples from $\mathcal{N}(\bmu_k^0,\bSigma_k^0)$, denoted as $\{\bff^0_{jk}\}$, and set $\xi^0_{jik}={\bff_{kj}^0}^{\top}\bz_i^0+\alpha_i^0$
\State For each $i\in[m]$ and $k\in[K]$, set
$r_{jk}^0
=
\frac{
\pi_k^0\,
\phi_d(\bq_j^0;\bmu_k^0,\bSigma_k^0)
}{
\sum_{\ell=1}^K
\pi_\ell^0\,
\phi_d(\bq_j^0;\bmu_\ell^0,\bSigma_\ell^0)
}$. 

\end{algorithmic}
\end{algorithm}

\paragraph{Initialization for Score-based random effects} To initialize $\bZ^0,\balpha^0$, we directly follow the first six steps described in Algorithm \ref{alg:MGVEM_initial}. We then consider the Gaussian mixture distribution with parameters obtained from the first six steps as initial distribution for $\bff_j$. Then we draw $m$ samples $\{\bff_j^0\}_{j\in[m]}$ from this initial distribution and set $\xi_{ji}^0={\bff_j^0}^{\top}\bz_i+\alpha_i$. At last, in order to initialize $\btheta$, we pretrain the score network on $\{\bff_j^0\}_{j\in[m]}$ for a specified number of steps, following the same steps as shown in Algorithm \ref{alg:LSGM}, with the only difference being now in the first step, we draw $B$ samples from $\{\bff_j^0\}_{j\in [m]}$ with replacement rather than draw samples from the posterior proposals.

\paragraph{Choices of parameters in the initialization steps} In all conducted experiments, we choose $a=\exp(-4)/2,b=1-10^{-3},\tau=\sqrt{\sum_{ji}v_{ji}/p}$ and for the k-means step, we use $100$ random seeds. For the Score-based initialization, we use the default value $B=256$, while treating the pre-training epochs as a hyper-parameters to be tuned.

\subsection{Model Selection}\label{append:select}
All of the three newly proposed Algorithm~\ref{alg:EM} takes the latent dimension $d$ as an input. In practice, however, this dimension is usually unknown and must be selected from the data. We therefore consider validation-based procedures for choosing $d$. The general principle is to split the observed incidence matrix into a training part and a validation part, fit the model on the training data for a range of candidate dimensions, and select the value of $d$ that gives the best predictive performance on the validation data.

\subsubsection{Model Selection for Categorical Random Effects}
First consider the categorical random effects, we propose two validation schemes. The first scheme is based on the ability of the fitted model to predict missing entries of the incidence matrix. Let $\Omega_{\mathrm{val}}\subset[m]\times[n]$ denote a randomly selected set of held-out entries, sampled according to a prespecified dropout rate, and let
$\Omega_{\mathrm{tr}}
    =
    ([m]\times[n])\setminus \Omega_{\mathrm{val}}$
be the set of observed training entries. The resulting partially observed incidence matrix is denoted by $\widetilde{\bV}$. For initialization, we impute the missing entries in $\widetilde{\bV}$ by zeros and apply the initialization procedure described above. During the subsequent EM iterations, however, only the observed training entries are used. Subsequently, we adjust the EM algorithm. We now update $\bDelta$ using:
\begin{equation}
\label{eq:missing-entry-estep}
    \Delta_{jk}^{(t)}
    =
    \frac{
        \pi_k^{(t)}
        \prod_{i:(j,i)\in\Omega_{\mathrm{tr}}}
        \left(q_{ki}^{(t)}\right)^{v_{ji}}
        \left(1-q_{ki}^{(t)}\right)^{1-v_{ji}}
    }{
        \sum_{\ell=1}^K
        \pi_\ell^{(t)}
        \prod_{i:(j,i)\in\Omega_{\mathrm{tr}}}
        \left(q_{\ell i}^{(t)}\right)^{v_{ji}}
        \left(1-q_{\ell i}^{(t)}\right)^{1-v_{ji}}
    } ,
    \qquad
    j\in[m],\ k\in[K].
\end{equation}
In the M-step, the cluster-specific probability matrix is updated entrywise by
\begin{equation}
\label{eq:missing-entry-mstep}
    \widehat q_{ki}
    =
    \frac{
        \sum_{j:(j,i)\in\Omega_{\mathrm{tr}}}
        \Delta_{jk} v_{ji}
    }{
        \sum_{j:(j,i)\in\Omega_{\mathrm{tr}}}
        \Delta_{jk}
    },
    \qquad
    k\in[K],\ i\in[n],
\end{equation}

After fitting the model for a candidate dimension $d$, we evaluate its prediction error on the held-out entries. For each $(j,i)\in\Omega_{\mathrm{val}}$, the predicted probability is set as 
\[
    \widehat p_{ji}^{(d)}
    =
    \sigma\left(
        \widehat{\bff}_{\widehat k_j}^{\top}
        \widehat{\bz}_i
        +
        \widehat\alpha_i
    \right),
    \qquad
    \widehat k_j
    =
    \arg\max_{k\in[K]} \widehat \Delta_{jk},
\]

Alternatively, we can also adopt a soft assignment approach, set the predicted probability as
\[
    \widehat p_{ji}^{(d)}
    =
    \sum_{k=1}^K
    \widehat\Delta_{jk}
    \sigma\left\{
        \left(\widehat{\bff}_k^{(d)}\right)^\top
        \widehat{\bz}_i^{(d)}
        +
        \widehat\alpha_i^{(d)}
    \right\}.
\]
We then select the latent dimension that minimizes the Brier Score $\frac{1}{|\Omega_{\mathrm{val}}|}
    \sum_{(j,i)\in\Omega_{\mathrm{val}}}
    \left(
        \widehat p^{(d)}_{ji}-v_{ji}
    \right)^2$.
    
The second scheme is based on the ability of the model to predict entire held-out hyperlinks. Let $\mathcal{J}_{\mathrm{val}}\subset[m]$ denote a randomly selected set of held-out hyperlinks, and let $\mathcal{J}_{\mathrm{tr}} =[m]\setminus \mathcal{J}_{\mathrm{val}}$
be the training hyperlinks. We fit the model using the submatrix $\bV^{\mathrm{tr}}=\{v_{ji}: j\in\mathcal{J}_{\mathrm{tr}},\ i\in[n]\}$ for each candidate dimension $d$. Denote the held-out submatrix by $\bV^{\mathrm{val}} =\{v_{ji}: j\in\mathcal{J}_{\mathrm{val}},\ i\in[n]\}$. After fitting the model on the training hyperlinks, we assign each held-out hyperlink $j\in\mathcal{J}_{\mathrm{val}}$ to its most likely latent type. This is achieved by finding $k_j$ that maximizes
\begin{equation}
\label{eq:heldout-row-posterior}
    \widehat \Delta_{jk}^{\mathrm{val}}
    =
    \frac{
        \widehat\pi_k
        \prod_{i=1}^n
        \widehat q_{ki}^{v_{ji}}
        \left(1-\widehat q_{ki}\right)^{1-v_{ji}}
    }{
        \sum_{\ell=1}^K
        \widehat\pi_\ell
        \prod_{i=1}^n
        \widehat q_{\ell i}^{v_{ji}}
        \left(1-\widehat q_{\ell i}\right)^{1-v_{ji}}
    },
    \qquad
    j\in\mathcal{J}_{\mathrm{val}}
\end{equation}
We then compute the predicted probability matrix for those held-out hyperlinks as
$\widehat q_{ji}^{\mathrm{val}}
    =
    \widehat q_{\widehat k_j i}
    =
    \sigma\left(
        \widehat{\bff}_{\widehat k_j}^{\top}
        \widehat{\bz}_i
        +
        \widehat\alpha_i
    \right),
    \qquad
    j\in\mathcal{J}_{\mathrm{val}},\ i\in[n]$. Finally, we compare the normalized Frobenius loss $
    \mathrm{Err}_{\mathrm{row}}(d)
    =
    \frac{
        \left\|
            \widehat{\bQ}^{\mathrm{val}}
            -
            \bV^{\mathrm{val}}
        \right\|_F^2
    }{
        \left\|
            \bV^{\mathrm{val}}
        \right\|_F^2
    }$ and choose the latent dimension such that this value is the smallest.

\subsubsection{Model Selection for Gaussian Mixture and Score-based Random Effects}
For Gaussian-mixture case and Score-based case, the same validation principle also applies, where we randomly mask entries of $\bV$, fit the corresponding model using only the observed training entries, and set the posterior predictive probability as
\begin{equation}
    p_{ji}^{pred}
    =\E_{\hat p(\bff_j|obs)}
    \left[
        \sigma\left(
            \bff_j^\top\widehat\bz_i+\widehat\alpha_i
        \right)
    \right]\label{eq:prediction}
\end{equation}

Here for both cases, we can't obtain a closed form expression for \eqref{eq:prediction}. We thus propose methods to estimate it.

\textbf{For the Gaussian-mixture case}, let's consider the final estimates obtained before the final normalization step, this does not effect $p_{ji}^{pred}$ by change of variables, but can simplify the estimating procedure. We just write $\tilde \bff_j,\tilde\bz_i$ as $\bff_j,\bz_i$ for notational simplicity.
\begin{equation}
    p_{ji}^{pred}=\sum_{k=1}^K\underbrace{P(k_j=k|\text{observed})}_{:=\tau_{jk}}\underbrace{\E_{\bff_j|k_j,\text{observed}}{[\sigma(\bff_j^{\top}\bz_i+\alpha_i)]}}_{:q_{ji}^k}
\end{equation}
where 
$$p(\bff_j|k_j=k,\text{observed})\propto\Pi_{j:(j,i)\in\Omega_{tr}}\sigma_{ji}^{v_{ji}}(1-\sigma_{ji})^{1-v_{ji}}\phi_d(\bff_j;\bmu_k,\bI_d)$$
We define the log value of the above RHS (up to a constant does not depend on $\bff_j$) as
\begin{equation}
    \label{eq:l_jk}
    \ell_{jk}(\bff):=\sum_{j:(j,i)\in\Omega_{tr}}[v_{ji}(\bff_i^{\top}\bz_j+\alpha_j)-\log(1+\exp(\bff_i^{\top}\bz_j+\alpha_j))]-\frac{1}{2}\|\bmu_k-\bff_i\|^2
\end{equation}

Then by Laplacian approximation, we have 
\begin{equation}
\begin{aligned}
    \label{laplacian}
    \bff_j|k_j,\text{obs}&\approx \cN(\hat \bff_{jk},\bS_{jk}), \text{ with } \hat \bff_{jk}:=\argmax \ell_{jk}(\bff), \bS_{jk}:=[-\nabla^2\ell_{jk}(\bff_{jk})]^{-1}\\
    \Rightarrow \bff_j^{\top}\bz_i+\alpha_i|k_j,\text{obs}&\approx \cN(\hat \bff_{jk}^{\top}\bz_i+\alpha_i,\bz_i^{\top}\bS_{jk}\bz_i)\\
    q_{ji}^k&\approx \E_{\bx\sim \cN(0,\bI_d)}[\sigma(\hat\bff_{jk}^{\top}\bz_i+\alpha_i+\sqrt{\bz_i^{\top}\bS_{jk}\bz_i}\bx)]\\
    &\approx \frac{1}{\sqrt{\pi}}\sum_{r=1}^Rw_r\sigma\big (\hat\bff_{jk}^{\top}\bz_i+\alpha_i+\sqrt{2\bz_i^{\top}\bS_{jk}\bz_i}\bx_r\big )
\end{aligned}
\end{equation}
where the last approximation utilizes Gaussian-Hermite approximation, with $w_r,\bx_r$ representing Hermite weights and nodes respectively.

For $\tau_{jk}$, since in Algorithm \ref{alg:MGVEM}, $r_{jk}$ can be viewed as an approximation for $P(k_j=k|obs)$, we can directly use the $\{r_{jk}^{(T)}\}$ in the final EM round $T$ to approximate $\tau_{jk}$.

Thus it remains to estimate $\hat\bff_{jk}$ defined in \eqref{laplacian}, which can be achieved by using Newton's method. To be more specific, we set $\bff_{jk}^{(0)}=\bmu_m$, and then iteratively update $\bff_{jk}$ using
$$\bff_{jk}^{(t+1)}=\bff_{jk}^{(t)}+(-\nabla^2\ell_{jk})^{-1}\nabla\ell_{jk}(\bff_{jk}^{(t)})$$
until convergence.

Putting all together, $\forall j\in[m]$, we follow the procedure summarized below to obtain $\hat p_{ji}^{pred}$.
\begin{itemize}
    \item Compute $\hat\bff_{jk}$ using Newton's method.
    \item Compute $\bS_{jk}=[-\nabla^2\ell_{jk}(\bff_{jk})]^{-1}$
    \item Compute $\hat q_{ji}^k:=\frac{1}{\sqrt{\pi}}\sum_{r=1}^Rw_r\sigma\big (\hat\bff_{jk}^{\top}\bz_i+\alpha_i+\sqrt{2\bz_i^{\top}\bS_{jk}\bz_i}\bx_r$
    \item Set $\hat p_{ji}=\sum_{k=1}^Kr_{jk}^T\hat q_{ji}^k$
\end{itemize}

Notice that we can also approximate $p_{ji}^{pred}$ by adopting a Monte Carlo method, drawing samples from $\hat p(\bff_j|obs)$.

\textbf{For the Score-based case}, we directly approximate \eqref{eq:prediction} through $\E_{q_j(\bff)}[\sigma(\bff_j^{\top}\hat\bz_i+\hat\alpha_i)]$, which can further be estimated using a Monte Carlo method, drawing samples from the final approximate posterior distribution using an importance-sampling procedure, as described in section \ref{append:alg:score}.

\paragraph{Remark on extensions} The same validation criteria can also be used to select the number of latent
components $K$ for the categorical and Gaussian-mixture random-effect models.
In that case, one evaluates a grid of candidate pairs $(K,d)$ and selects the
pair with the smallest validation error. For the score-based random-effect
model, there is no finite mixture size $K$, so the same procedure is used only for selecting $d$ and other implementation hyperparameters.

\section{Supplementary for Empirical Results}\label{append:numerics}
\paragraph{Computational resources.}
All experiments were run on a local MacBook Pro with an Apple M3 Pro chip, an 11-core CPU, and 18 GB of unified memory. The experiments were CPU-only; no GPU acceleration, cloud computing, cluster scheduler, or distributed computation was used. HRL-Cat and HRL-GM were implemented in R and executed as local CPU processes. HRL-Score was implemented in Python and also executed locally on CPU.

For reference, on a synthetic hypergraph with \(m=5000\) hyperedges and \(p=180\) nodes, a single full fit of HRL-Cat takes less than one minute, a single full fit of HRL-GM takes approximately eight minutes, and a single full fit of HRL-Score takes approximately three minutes for 700 epochs. These timings are intended as representative wall-clock runtimes for one model fit; repeated experimental runs and hyperparameter searches scale accordingly.

The implementation code is available in the anonymized repository\footnote{\href{https://anonymous.4open.science/r/Hypergraph-Representation-Learning-7463}{https://anonymous.4open.science/r/Hypergraph-Representation-Learning-7463}} 
\subsection{Additional Simulation Results}\label{append:simulation}
This section provides additional simulation results for the proposed HRL methods. We first study the categorical random-effects model, which corresponds to HRL-Cat, and evaluate parameter recovery under three settings: balanced, unbalanced, and larger-component. We then examine the empirical behavior of the validation procedure for selecting the latent embedding dimension. Lastly, we give the full simulation details for the Gaussian-mixture and heavy-tailed random-effect designs used in Section~\ref{sec:sim}, together with additional recovery results.

\subsubsection{Categorical Random Effects}
\label{append:sim-cat}

In this whole subsection, we generate hypergraphs from the categorical random-effects model introduced in Section~\ref{sec:estimation}. In this setting, the random-effect distribution is
\[
    \mathbb P_{\mathbf F}^*
    =
    \sum_{k=1}^K \pi_k^* \delta_{\mathbf f_k^*},
    \qquad
    \pi_k^*>0,\qquad
    \sum_{k=1}^K \pi_k^*=1,
\]
where \(K\) is the number of latent hyperlink types and \(\mathbf f_k^*\in\mathbb R^d\) is the type-specific random effect. 

For each hyperlink \(j\in[m]\), we first sample a latent type
\[
    g_j \sim \mathrm{Categorical}(\boldsymbol\pi^*),
\]
and then generate the incidence vector \(\mathbf v_j=(v_{j1},\ldots,v_{jn})^\top\) independently across nodes:
\[
    v_{ji}\mid g_j=k
    \sim
    \mathrm{Bernoulli}(q_{ki}^*),
    \qquad
    q_{ki}^*
    =
    \sigma\!\left\{
    \mathbf f_k^{*\top}\mathbf z_i^*+\alpha_i^*
    \right\}.
\]

We denote the full type-by-node probability matrix by
\[
    \mathbf Q^*
    =
    (q_{ki}^*)_{k\in[K],\,i\in[n]}
    \in [0,1]^{K\times n}.
\]

For each configuration, we generate 100 independent hypergraphs. The final estimates produced by HRL-Cat are normalized to satisfy the identifiability constraints in \eqref{eq:cat_identifiability}, and the true parameters are generated and normalized to satisfy the same constraints. Therefore, before computing estimation errors, we only need to account for the remaining indeterminacies: permutation of the \(K\) latent hyperlink types and coordinate-wise sign flips of the \(d\)-dimensional embeddings.

We first correct for label switching. Let \(\mathcal S_K\) denote the set of permutations of \([K]\). We choose the permutation
\[
    \widehat{\rho}
    =
    \arg\min_{\rho\in\mathcal S_K}
    \left\|
    \widehat{\mathbf Q}_{\rho(\cdot),:}
    -
    \mathbf Q^*
    \right\|_F ,
\]
where \(\widehat{\mathbf Q}_{\rho(\cdot),:}\) denotes the matrix obtained by permuting the rows of \(\widehat{\mathbf Q}\) according to \(\rho\). We then apply the same permutation to the estimated component-specific quantities:
\[
    \widehat{\mathbf Q}
    \leftarrow
    \widehat{\mathbf Q}_{\widehat{\rho}(\cdot),:},
    \qquad
    \widehat{\mathbf F}
    \leftarrow
    \widehat{\mathbf F}_{\widehat{\rho}(\cdot),:},
    \qquad
    \widehat{\boldsymbol\pi}
    \leftarrow
    \widehat{\boldsymbol\pi}_{\widehat{\rho}(\cdot)}.
\]

After aligning the component labels, we correct the remaining coordinate-wise sign ambiguity. For each latent coordinate \(l\in[d]\), we choose
\[
    s_l
    =
    \arg\min_{s\in\{-1,1\}}
    \left\|
    s\,\widehat{\mathbf Z}_{[:,l]}
    -
    \mathbf Z^*_{[:,l]}
    \right\|_2 .
\]
Let \(\mathbf S=\mathrm{diag}(s_1,\ldots,s_d)\). We then replace
\[
    \widehat{\mathbf Z}
    \leftarrow
    \widehat{\mathbf Z}\mathbf S,
    \qquad
    \widehat{\mathbf F}
    \leftarrow
    \widehat{\mathbf F}\mathbf S .
\]
This sign correction preserves the fitted linear predictors
\[
    \widehat{\mathbf f}_k^\top \widehat{\mathbf z}_i+\widehat{\alpha}_i
\]
and hence preserves the fitted incidence probabilities. After these alignment steps, we evaluate parameter recovery using the normalized Frobenius errors
\[
    \mathrm{Err}_{Q}
    =
    \frac{\|\widehat{\mathbf Q}-\mathbf Q^*\|_F}{\sqrt{Kn}},
    \qquad
    \mathrm{Err}_{F}
    =
    \frac{\|\widehat{\mathbf F}-\mathbf F^*\|_F}{\sqrt{Kd}},
    \qquad
    \mathrm{Err}_{Z}
    =
    \frac{\|\widehat{\mathbf Z}-\mathbf Z^*\|_F}{\sqrt{nd}}.
\]
where $\bF^*=(\bff_1^*,\dots,\bff_K^*)^{\top}\in\RR^{K,d}$.
Unless otherwise stated, the intercepts are generated independently as
\[
    \alpha_i^* \overset{i.i.d.}{\sim} \mathcal N(0,0.3),
    \qquad i\in[n].
\]
\textbf{Case 1: balanced categorical design}

We first consider a balanced categorical design with \(K=8\) latent hyperlink types and embedding dimension \(d=3\). Before normalization, the support matrix is set as
\small{\[
\mathbf F^*
=
\begin{pmatrix}
 1.5 &  1.0 &  0.5 \\
-1.5 &  1.0 &  0.5 \\
-1.5 & -1.0 &  0.5 \\
 1.5 & -1.0 &  0.5 \\
 1.5 &  1.0 & -0.5 \\
-1.5 &  1.0 & -0.5 \\
-1.5 & -1.0 & -0.5 \\
 1.5 & -1.0 & -0.5
\end{pmatrix}.
\]}
\normalsize

We set the mixing probabilities to be uniform, i.e. \(\pi_1^*=\cdots=\pi_8^*=1/8\), and set the node embedding matrix \(\mathbf Z^*\in\mathbb R^{n\times d}\) s.t it has a block-sparse structure. Specifically, for each coordinate \(t\in[d]\), only nodes in the corresponding block
\[
    \left\{ \frac{(t-1)n}{d}+1,\ldots,\frac{tn}{d} \right\}
\]
have nonzero entries in coordinate \(t\), and these nonzero entries are independently generated from \(\mathrm{Uniform}(1,3)\). We then normalize $\bF^*,\bZ^*,\balpha^*$ to match the regularity condition \eqref{eq:cat_identifiability}.

We vary the number of hyperlinks \(m\in\{500,1000,1500,2000\}\) and the number of nodes \(n\in\{30,90,270,810\}\). Table~\ref{tab:sim1} and Figure~\ref{fig:sim1} summarize the estimation errors over 100 Monte Carlo repetitions.

\begin{table}[t]
\centering
\caption{Parameter recovery under the balanced categorical design. Values are averaged over 100 Monte Carlo repetitions.}
\label{tab:sim1}
\small
\renewcommand{\arraystretch}{1.15}
\setlength{\tabcolsep}{4pt}
\resizebox{1.0\textwidth}{!}{
\begin{tabular}{c|cccc|cccc|cccc}
\toprule
& \multicolumn{4}{c|}{$\|\hat{\bQ}-\bQ^*\|_F/\sqrt{K n}$}
& \multicolumn{4}{c|}{$\|\hat{\bF}-\bF^*\|_F/\sqrt{K d}$}
& \multicolumn{4}{c}{$\|\hat{\bZ}-\bZ^*\|_F/\sqrt{n d}$} \\
\cline{2-13}
& $n=30$ & $n=90$ & $n=270$ & $n=810$
& $n=30$ & $n=90$ & $n=270$ & $n=810$
& $n=30$ & $n=90$ & $n=270$ & $n=810$ \\
\midrule
$m=500$  & 0.057 & 0.039 & 0.037 & 0.037 & 0.297 & 0.181 & 0.156 & 0.152 & 0.347 & 0.282 & 0.269 & 0.274 \\
$m=1000$ & 0.032 & 0.026 & 0.025 & 0.025 & 0.148 & 0.097 & 0.093 & 0.091 & 0.204 & 0.175 & 0.169 & 0.171 \\
$m=1500$ & 0.024 & 0.021 & 0.020 & 0.020 & 0.107 & 0.077 & 0.075 & 0.069 & 0.150 & 0.135 & 0.133 & 0.133 \\
$m=2000$ & 0.021 & 0.018 & 0.017 & 0.017 & 0.087 & 0.066 & 0.063 & 0.059 & 0.123 & 0.113 & 0.113 & 0.113 \\
\bottomrule
\end{tabular}
}
\end{table}

\begin{figure}[t]
\centering\includegraphics[width=5in]{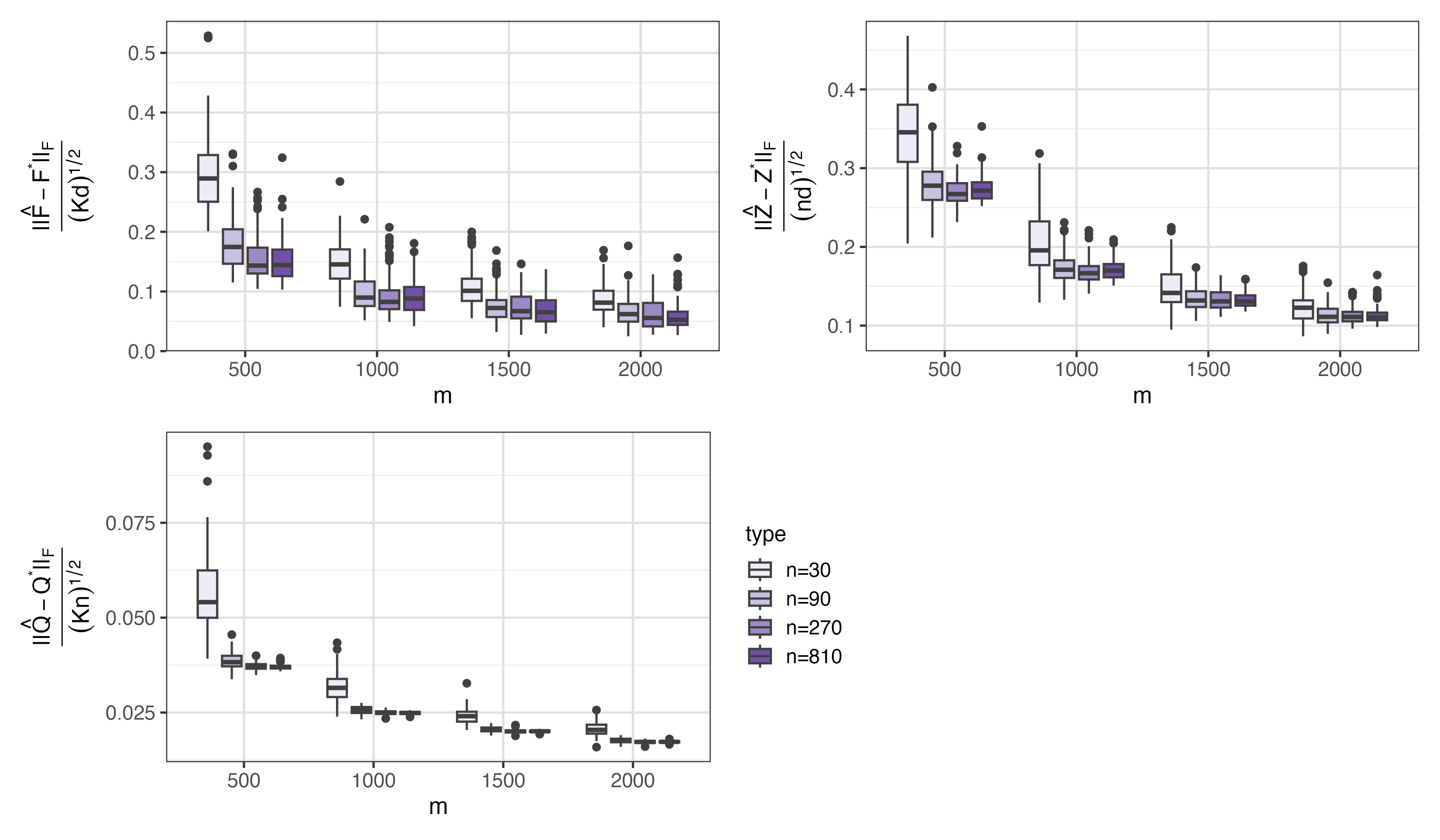}
	\caption{Monte Carlo distribution of normalized estimation errors under the balanced categorical design. The number of hyperlinks \(m\) varies along the horizontal axis, and different box colors correspond to different node sizes \(n\).}
    \label{fig:sim1}
\end{figure}

\textbf{Case 2: unbalanced categorical design}
\label{append:sim-cat-unbalanced}

We next examine the effect of unequal mixing probabilities. The unnormalized support points \(\mathbf F^*\), node embeddings \(\mathbf Z^*\), and intercepts \(\boldsymbol\alpha^*\) are generated as in the balanced design, while the mixing probabilities satisfy
\[
    \pi_1^*:\pi_2^*:\cdots:\pi_8^*
    =
    1:1:2:2:3:3:4:4.
\]
We then apply the same normalization step as in the balanced design. This setting tests whether HRL-Cat remains stable when some latent hyperlink types are substantially less frequent than others. For this case, we use the exactly the same grid of \((m,n)\) values as in the balanced design. The results are reported in Table~\ref{tab:sim2} and Figure~\ref{fig:sim2}.

\begin{table}[t]
\centering
\caption{Parameter recovery under the unbalanced categorical design. The mixing probabilities satisfy \(\pi_1^*:\cdots:\pi_8^*=1:1:2:2:3:3:4:4\). Values are averaged over 100 Monte Carlo repetitions.}
\label{tab:sim2}
\small
\renewcommand{\arraystretch}{1.15}
\setlength{\tabcolsep}{4pt}
\resizebox{1.0\textwidth}{!}{
\begin{tabular}{c|cccc|cccc|cccc}
\toprule
& \multicolumn{4}{c|}{$\|\hat{\bQ}-\bQ^*\|_F/\sqrt{K n}$}
& \multicolumn{4}{c|}{$\|\hat{\bF}-\bF^*\|_F/\sqrt{K d}$}
& \multicolumn{4}{c}{$\|\hat{\bZ}-\bZ^*\|_F/\sqrt{n d}$} \\
\cline{2-13}
& $n=30$ & $n=90$ & $n=270$ & $n=810$
& $n=30$ & $n=90$ & $n=270$ & $n=810$
& $n=30$ & $n=90$ & $n=270$ & $n=810$ \\
\midrule
$m=500$  & 0.083 & 0.054 & 0.050 & 0.049 & 0.452 & 0.267 & 0.239 & 0.234 & 0.446 & 0.351 & 0.341 & 0.341 \\
$m=1000$ & 0.048 & 0.032 & 0.031 & 0.031 & 0.241 & 0.129 & 0.120 & 0.118 & 0.293 & 0.223 & 0.217 & 0.218 \\
$m=1500$ & 0.034 & 0.025 & 0.024 & 0.024 & 0.152 & 0.093 & 0.083 & 0.080 & 0.211 & 0.168 & 0.162 & 0.163 \\
$m=2000$ & 0.027 & 0.021 & 0.020 & 0.020 & 0.115 & 0.068 & 0.064 & 0.062 & 0.166 & 0.132 & 0.132 & 0.134 \\
\bottomrule
\end{tabular}
}
\end{table}

\begin{figure}[t]
\centering\includegraphics[width=5in]{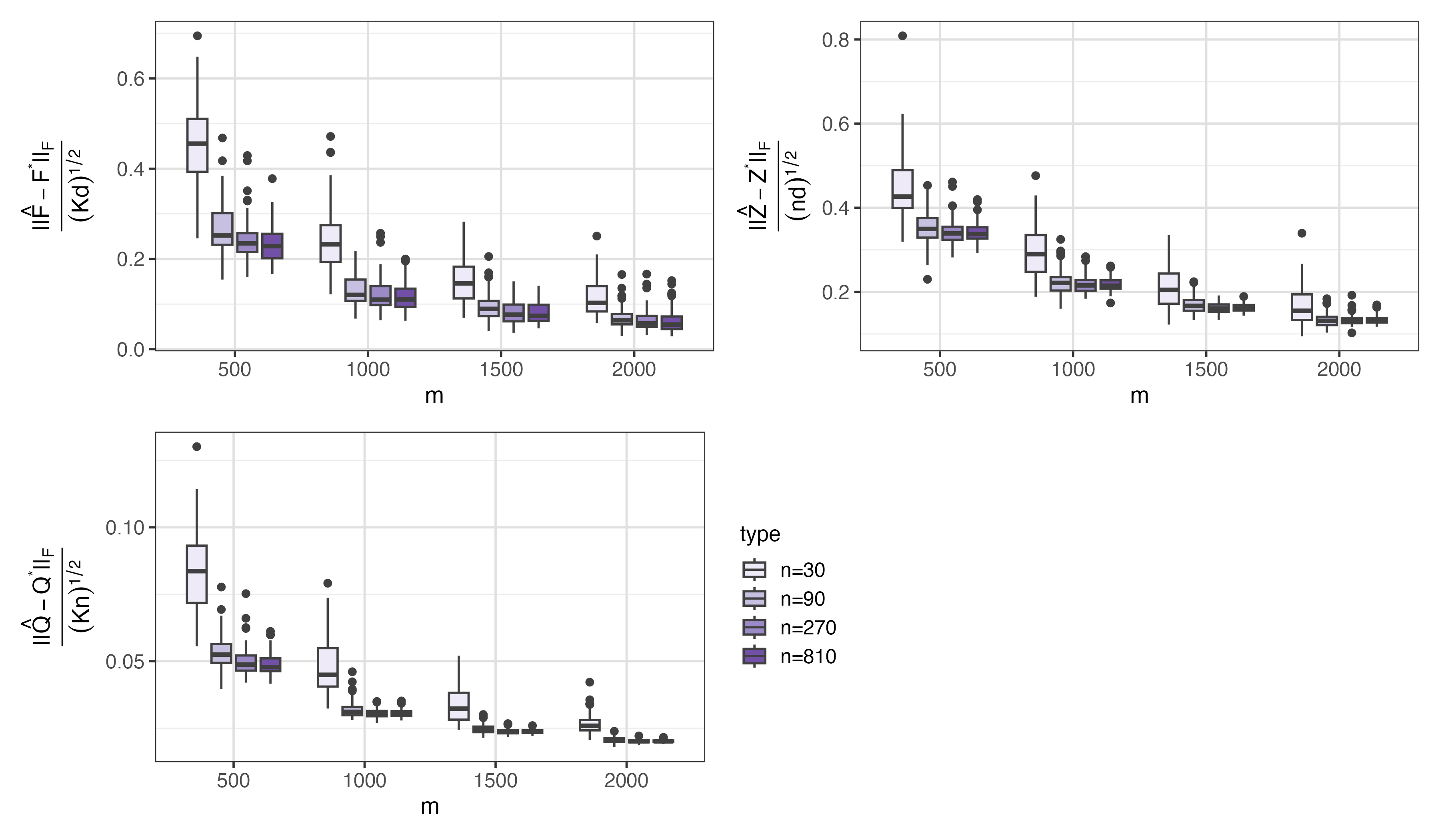}
	\caption{Monte Carlo distribution of normalized estimation errors under the unbalanced categorical design. The number of hyperlinks \(m\) varies along the horizontal axis, and different box colors correspond to different node sizes \(n\).}
    \label{fig:sim2}
\end{figure}

\textbf{Case 3: larger number of latent hyperlink types}
\label{append:sim-cat-large}

We further consider a more challenging setting with \(K=32\) latent hyperlink types and embedding dimension \(d=5\). The rows of \(\mathbf F^*\) are constructed from all \(2^5=32\) sign vectors in \(\{-1,1\}^5\), followed by coordinate-wise scaling with coefficients
\[
    (0.8,\ 0.7,\ 0.6,\ 0.5,\ 0.4).
\]
The node embedding matrix \(\mathbf Z^*\) is generated using the same block-sparse construction as above, now with \(d=5\) and the mixing probabilities are uniform over the \(32\) latent types. Then we again adopt a normalization step for these true parameters.

For this case, we vary \(m\in\{1000,2000,3000,4000\}\) and \(n\in\{50,250,500,1000\}\). Results are summarized in Table~\ref{tab:sim3}.

\begin{table}[t]
\centering
\caption{Parameter recovery under the larger categorical design with \(K=32\) latent hyperlink types and embedding dimension \(d=5\). Values are averaged over 100 Monte Carlo repetitions.}
\label{tab:sim3}
\small
\renewcommand{\arraystretch}{1.15}
\setlength{\tabcolsep}{4pt}
\resizebox{0.92\textwidth}{!}{
\begin{tabular}{c|cccc|cccc|cccc}
\toprule
& \multicolumn{4}{c|}{$\|\hat{\bQ}-\bQ^*\|_F/\sqrt{K n}$}
& \multicolumn{4}{c|}{$\|\hat{\bF}-\bF^*\|_F/\sqrt{K d}$}
& \multicolumn{4}{c}{$\|\hat{\bZ}-\bZ^*\|_F/\sqrt{n d}$} \\
\cline{2-13}
& $n=50$ & $n=250$ & $n=500$ & $n=1000$
& $n=50$ & $n=250$ & $n=500$ & $n=1000$
& $n=50$ & $n=250$ & $n=500$ & $n=1000$ \\
\midrule
$m=1000$ & 0.112 & 0.056 & 0.042 & 0.041 & 0.421 & 0.189 & 0.156 & 0.160 & 0.451 & 0.313 & 0.286 & 0.288 \\
$m=2000$ & 0.061 & 0.027 & 0.025 & 0.025 & 0.219 & 0.119 & 0.116 & 0.114 & 0.291 & 0.205 & 0.205 & 0.206 \\
$m=3000$ & 0.037 & 0.020 & 0.020 & 0.020 & 0.144 & 0.100 & 0.096 & 0.094 & 0.209 & 0.174 & 0.170 & 0.168 \\
$m=4000$ & 0.027 & 0.017 & 0.017 & 0.017 & 0.115 & 0.083 & 0.084 & 0.080 & 0.173 & 0.146 & 0.149 & 0.145 \\
\bottomrule
\end{tabular}
}
\end{table}

\paragraph{Summary of categorical simulations.}
Across all the three categorical designs, HRL-Cat accurately recovers the latent structure when the categorical random-effects model is correctly specified. For fixed node size \(n\), the errors decrease as the number of hyperlinks \(m\) increases, reflecting the increased information indeed helps for estimating the latent group distribution and node embeddings. For fixed \(m\), increasing \(n\) also improves estimation in small- and moderate-dimensional settings, although the improvement becomes less pronounced once \(n\) is sufficiently large. The unbalanced design leads to larger errors than the balanced design, as expected because rare hyperlink types are observed less frequently, but the errors still decrease steadily with \(m\). The larger design with \(K=32\) and \(d=5\) is more challenging, yet the same decreasing-error pattern remains. These results provide empirical support for the identifiability normalization and show that the categorical HRL model can recover both the type-specific incidence probabilities and the underlying node embeddings across a range of simulation regimes.

\subsubsection{Selection of the Latent Embedding Dimension}
\label{append:sim-select-d}

We next evaluate the validation procedure proposed in Appendix~\ref{append:select} for selecting the latent embedding dimension \(d\). We consider two simulation settings based on the balanced categorical and larger categorical designs. In the first setting, the true dimension is \(d^*=3\), with \(K=8\), \(m=1000\), and \(n=120\). In the second setting, the true dimension is \(d^*=5\), with \(K=32\), \(m=2000\), and \(n=500\).

For each setting and each candidate dimension, we repeatedly mask a subset of incidence entries, refit the model to the partially observed hypergraph, and evaluate the validation criterion. We repeat this procedure 50 times for each configuration. The results are shown in Figures~\ref{fig:validation-1} and~\ref{fig:validation-3}.

\begin{figure}[t]
\centering\includegraphics[width=5in]{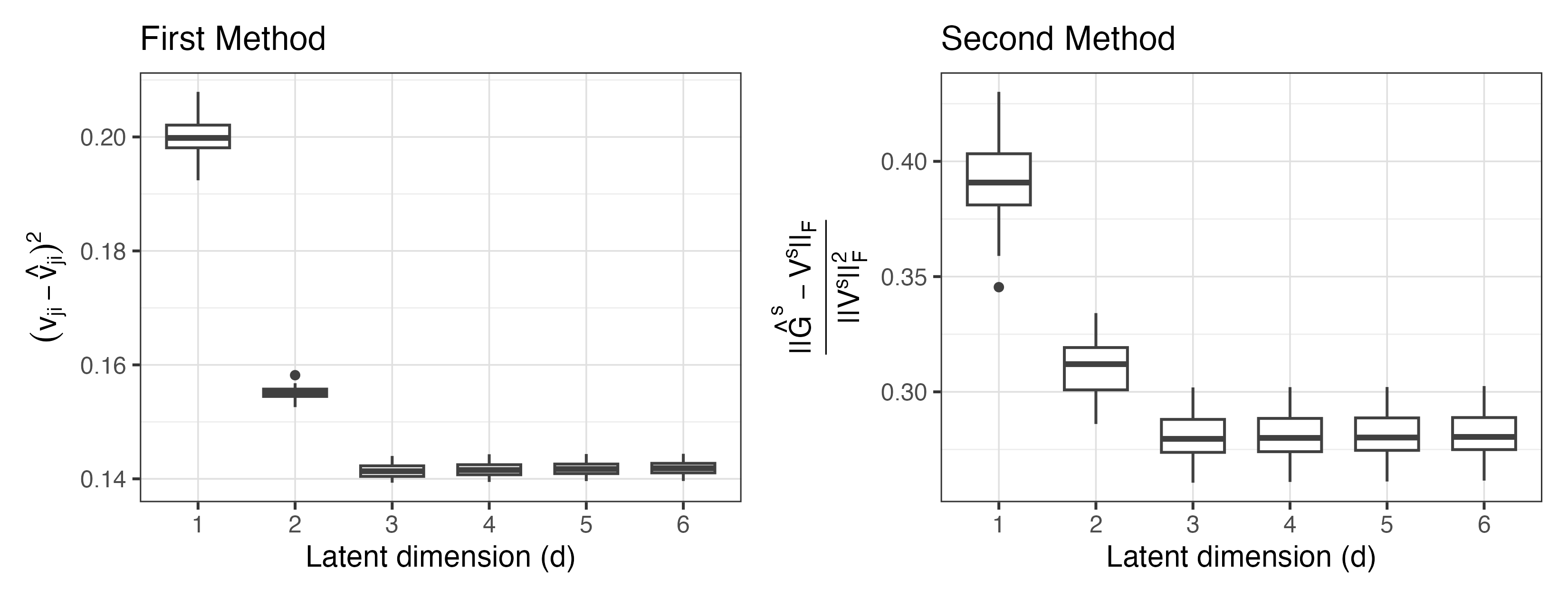}
	\caption{Validation curves for selecting the latent embedding dimension \(d\) under the balanced categorical design. The true embedding dimension is \(d^*=3\).}
    \label{fig:validation-1}
\end{figure}

\begin{figure}[t]
\centering\includegraphics[width=5in]{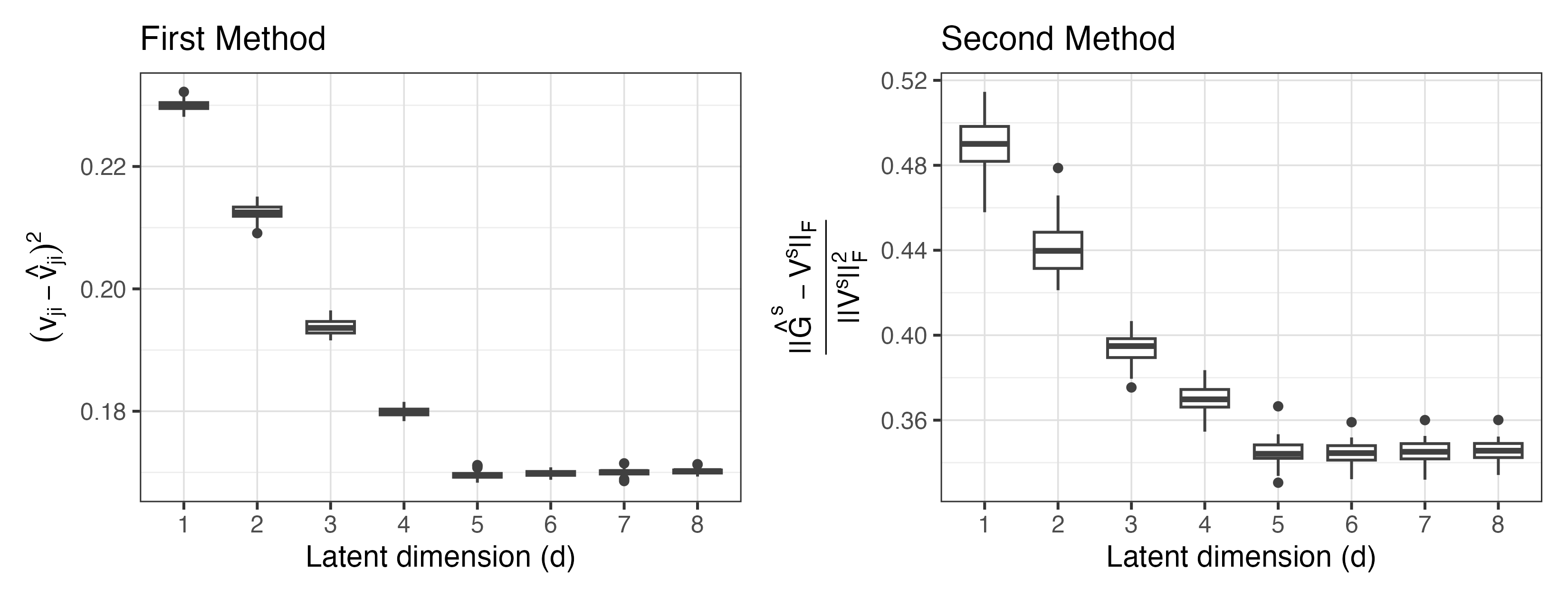}
	\caption{Validation curves for selecting the latent embedding dimension \(d\) under the larger categorical design. The true embedding dimension is \(d^*=5\)}
    \label{fig:validation-3}
\end{figure}

In both settings, the validation error decreases sharply when the candidate dimension is smaller than the true dimension, and then stabilizes once the candidate dimension reaches or exceeds the true value. This pattern suggests that underfitting the latent dimension leads to clearly worse predictive behavior, whereas overfitting the dimension produces little additional improvement. Therefore, in real-data applications, we select the smallest dimension whose validation performance is close to the best observed performance.

\subsubsection{Gaussian-mixture and Heavy-tailed Random-effect Designs}
\label{append:sim-continuous-random-effects}

We now describe the continuous random-effect simulations used in Section~\ref{sec:sim}. These simulations complement the categorical experiments by comparing the proposed methods under a well-specified Gaussian-mixture design and a misspecified heavy-tailed design. Throughout this subsection, data are generated from the inner-product incidence model
\[
    v_{ji}\mid \mathbf f_j^*,\mathbf z_i^*,\alpha_i^*
    \sim
    \mathrm{Bernoulli}
    \left[
    \sigma\!\left(
    \mathbf f_j^{*\top}\mathbf z_i^*+\alpha_i^*
    \right)
    \right],
    \qquad j\in[m],\ i\in[n],
\]
independently across node--hyperlink pairs conditional on the latent parameters.

We fix the number of nodes at \(n=180\), the latent dimension at \(d=3\), and the number of latent mixture components at \(K=8\). We vary the number of hyperlinks over
\[
    m\in\{1000,3000,5000\}.
\]
For each configuration, we generate 20 independent hypergraphs.

\paragraph{}
For each hyperlink \(j\), we first sample a latent component label
\[
    g_j\sim \mathrm{Categorical}(\boldsymbol\pi^*),
    \qquad
    \pi_1^*=\cdots=\pi_8^*=\frac{1}{8}.
\]
Thus, the eight mixture components are equally represented in expectation. These labels are used only for data generation and for evaluating latent group recovery; they are not provided to any estimation method.

\paragraph{Mixture labels and component centers.}
For each hyperlink \(j\), we sample a latent component label
\[
    g_j\sim \mathrm{Categorical}(\boldsymbol\pi^*),
    \qquad
    \pi_1^*=\cdots=\pi_8^*=\frac{1}{8}.
\]
The component centers are set as the eight vertices of a three-dimensional cube:
\[
    \boldsymbol\mu_k^*
    \in
    \{-3.5,3.5\}^3,
    \qquad k=1,\ldots,8.
\]

\paragraph{Gaussian-mixture design.}
In the Gaussian-mixture design, the hyperlink random effects are generated as
\[
    \mathbf f_j^*\mid g_j=k
    \sim
    \mathcal N(\boldsymbol\mu_k^*,\mathbf I_d),
    \qquad
    d=3.
\]
Equivalently, the population random-effect distribution is
\[
    \mathbb P_{\mathbf F}^*
    =
    \frac{1}{K}
    \sum_{k=1}^K
    \mathcal N(\boldsymbol\mu_k^*,\mathbf I_d).
\]
This setting is correctly specified for HRL-GM, since HRL-GM explicitly models the hyperlink random effects through a Gaussian mixture.

\paragraph{Heavy-tailed \(t\)-mixture random design.}
In the heavy-tailed design, we keep the same mixture labels, component centers, node embeddings, and intercepts, but replace each Gaussian component by a scaled multivariate \(t\)-distribution with \(4\) degrees of freedom. Specifically,
\[
    \mathbf f_j^*\mid g_j=k
    \sim
    t_{\nu}
    \left(
    \boldsymbol\mu_k^*,
    \frac{\nu-2}{\nu}\mathbf I_d
    \right),
    \qquad
    \nu=4.
\]
This scaling gives each component covariance \(\mathbf I_d\). Thus, the Gaussian-mixture and \(t\)-mixture designs have comparable component variances, and the main difference is the heavier tail behavior of the latter. The \(t\)-mixture setting is intentionally misspecified for HRL-GM and is used to evaluate robustness to non-Gaussian random effects.

\paragraph{Node embeddings and intercepts.}
We design the true node embedding matrix
\[
    \mathbf Z^*
    =
    (\mathbf z_1^*,\ldots,\mathbf z_n^*)^\top
    \in\mathbb R^{n\times d}
\]
in a way s.t. it has a block-sparse structure. We divide the nodes into three equal blocks. For nodes in the first block, only the first coordinate is nonzero:
\[
    z_{i1}^*
    \overset{i.i.d.}{\sim}
    \mathrm{Uniform}\!\left(\frac{3}{7},\frac{4}{7}\right),
    \qquad
    i=1,\ldots,\frac{n}{3}.
\]
For nodes in the second block, only the second coordinate is nonzero:
\[
    z_{i2}^*
    \overset{i.i.d.}{\sim}
    \mathrm{Uniform}\!\left(\frac{2}{7},\frac{3}{7}\right),
    \qquad
    i=\frac{n}{3}+1,\ldots,\frac{2n}{3}.
\]
For nodes in the third block, only the third coordinate is nonzero:
\[
    z_{i3}^*
    \overset{i.i.d.}{\sim}
    \mathrm{Uniform}\!\left(\frac{1}{7},\frac{2}{7}\right),
    \qquad
    i=\frac{2n}{3}+1,\ldots,n.
\]
All remaining entries of \(\mathbf Z^*\) are set to zero. The node intercepts are generated independently as
\[
    \alpha_i^*
    \overset{i.i.d.}{\sim}
    \mathcal N(0,0.3),
    \qquad
    i=1,\ldots,n.
\]

\paragraph{Hypergraph generation.}
Given \(\{\mathbf f_j^*\}_{j=1}^m\), \(\{\mathbf z_i^*\}_{i=1}^n\), and \(\{\alpha_i^*\}_{i=1}^n\), we generate the incidence matrix
\[
    \mathbf V=(v_{ji})\in\{0,1\}^{m\times n}
\]
entrywise according to
\[
    v_{ji}
    \sim
    \mathrm{Bernoulli}
    \left[
    \sigma
    \left(
    \mathbf f_j^{*\top}\mathbf z_i^*+\alpha_i^*
    \right)
    \right].
\]

\paragraph{HRL-Score implementation.}
For HRL-Score, hyperparameters are selected using a validation missing-value prediction task described in Appendix~\ref{append:missing-value-implementation}. Specifically, we mask \(20\%\) of incidence entries for validation, train on the remaining \(80\%\), and select the configuration with the highest validation AUPRC. We tune the number of pretraining steps in \(\{100,200,400\}\), the hidden dimension in \(\{128,256\}\), and the learning rate in \(\{5\times 10^{-4},10^{-3},2\times 10^{-3}\}\). We fix \(\beta_0=0.1\), \(\beta_1=20\), \(\sigma_0=0\), and batch size \(B=256\). Each model is trained for 700 epochs, and final estimates are computed using the median over checkpoints from epochs 600--700.

\paragraph{Normalization and comparison.}
Because the inner-product model is identifiable only up to the normalization constraints described in Section~\ref{sec:model}, the raw generating parameters and raw fitted parameters are not directly comparable. Therefore, before computing estimation errors, we normalize both the true parameters and the estimated parameters using the normalization step corresponding to each fitted algorithm. For HRL-Cat, HRL-GM, and HRL-Score, we follow the corresponding final normalization steps in Algorithms~\ref{alg:EM}, \ref{alg:MGVEM}, and~\ref{alg:HRLScore}. For HRL-Fixed and for the true generating parameters, we use the same moment-normalization principle as in HRL-Score, replacing the estimated random-effect mean and covariance by their empirical or population counterparts. This normalization preserves the linear predictors
\[
    \mathbf f_j^\top\mathbf z_i+\alpha_i
\]
while placing all parameters in a comparable identifiable scale.

Node-embedding recovery is then evaluated using the normalized Frobenius error $\mathrm{Err}_{Z}=\frac{
    \|\widehat{\mathbf Z}-\mathbf Z^*\|_F
    }{
    \sqrt{nd}
    }$, which gives Figure~\ref{fig:sim_combine} shown in the main text. Additionally, we also report the recovery of $\balpha$ through $\|\hat\balpha-\balpha^*\|_2/\sqrt{n}$, shown in Figure~\ref{fig:alpha}.

\begin{figure}[t]
\centering\includegraphics[width=5in]{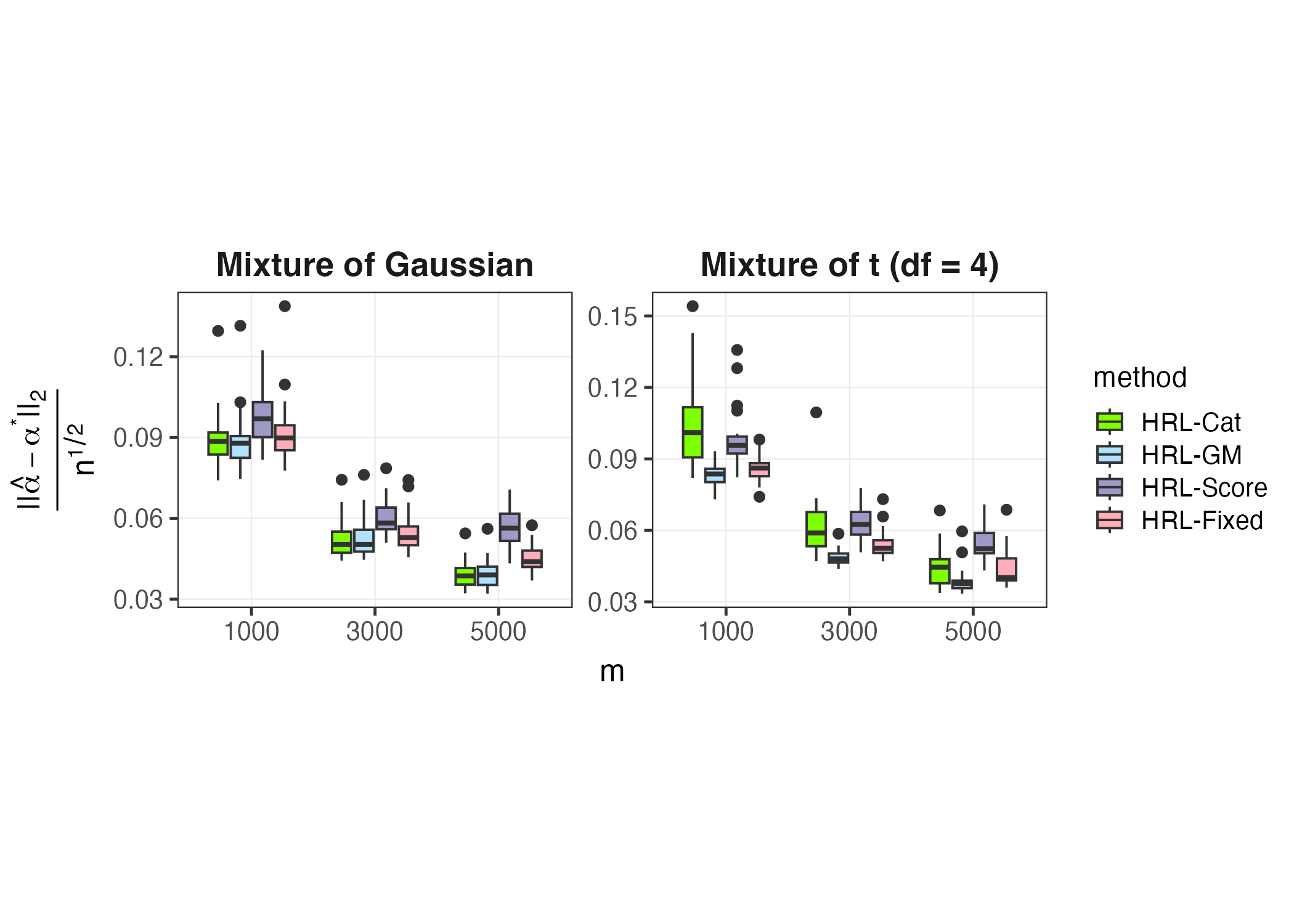}
	\caption{Recovery of node degree parameters $\balpha$ under Gaussian-mixture (left) and t-mixture (right), with $n=180$.}
    \label{fig:alpha}
\end{figure}
We also evaluate latent group recovery using adjusted Rand index (ARI) and normalized mutual information (NMI). For HRL-Cat and HRL-GM, the estimated label of hyperlink \(j\) is obtained by the largest posterior membership probability, using \(\widehat\Delta_{jk}^{(T)}\) and \(\widehat r_{jk}^{(T)}\), respectively. For HRL-Score, we draw \(L=4096\) prior particles from the learned distribution \(p_{\boldsymbol\theta}(\mathbf f)\), compute the approximate posterior mean \(\widehat{\mathbf m}_j\) defined in \eqref{eq:score-estep}, and apply \(K\)-means clustering to \(\{\widehat{\mathbf m}_j\}_{j=1}^m\) with \(K=8\). For HRL-Fixed, we apply \(K\)-means clustering to the estimated hyperlink-specific effects \(\{\widehat{\mathbf f}_j\}_{j=1}^m\), again with \(K=8\). The ARI and NMI values reported in the main text are averaged over 20 repetitions, we report their standard errors in Tables~\ref{tab:sim-cluster-GM} and~\ref{tab:sim-cluster-t}.

\begin{table}[t]
\centering
\small
\caption{Clustering performance for Gaussian-mixture design, mean(standard error), based on 20 repetitions}
\label{tab:sim-cluster-GM}
\setlength{\tabcolsep}{3.6pt}
\begin{tabular}{lcccccc}
\toprule
& \multicolumn{6}{c}{Gaussian mixture} \\
\cmidrule(lr){2-7}
Method
& \multicolumn{2}{c}{1000}
& \multicolumn{2}{c}{3000}
& \multicolumn{2}{c}{5000} \\
\cmidrule(lr){2-3} \cmidrule(lr){4-5} \cmidrule(lr){6-7}
& ARI$\uparrow$ & NMI$\uparrow$
& ARI$\uparrow$ & NMI$\uparrow$
& ARI$\uparrow$ & NMI$\uparrow$ \\
\midrule
HRL-Cat
& 0.958(0.0027) & 0.956(0.0026)
& 0.961(0.0014) & 0.955(0.0014)
& 0.963(0.0010) & 0.956(0.0009) \\

HRL-GM
& 0.962(0.0025) & 0.959(0.0025)
& \underline{0.962(0.0013)} & \underline{0.956(0.0013)}
& 0.964(0.0009) & \underline{0.956(0.0009)} \\

HRL-Score
& \textbf{0.964(0.0015)} & 0.960(0.0015)
& \textbf{0.965(0.0012)} & \textbf{0.959(0.0013)}
& \textbf{0.966(0.0007)} & \textbf{0.959(0.0008)} \\

HRL-Fixed
& \underline{0.962(0.0024)} & \textbf{0.960(0.0024)}
& 0.962(0.0013) & 0.956(0.0013)
& \underline{0.964(0.0010)} & \underline{0.956(0.0009)} \\
\bottomrule
\end{tabular}
\end{table}

\begin{table}[t]
\centering
\small
\caption{Clustering performance for t-mixture design, mean(standard error), based on 20 repetitions}
\label{tab:sim-cluster-t}
\setlength{\tabcolsep}{3.6pt}
\begin{tabular}{lcccccc}
\toprule
& \multicolumn{6}{c}{$t$-mixture, df = 4} \\
\cmidrule(lr){2-7}
Method
& \multicolumn{2}{c}{1000}
& \multicolumn{2}{c}{3000}
& \multicolumn{2}{c}{5000} \\
\cmidrule(lr){2-3} \cmidrule(lr){4-5} \cmidrule(lr){6-7}
& ARI$\uparrow$ & NMI$\uparrow$
& ARI$\uparrow$ & NMI$\uparrow$
& ARI$\uparrow$ & NMI$\uparrow$ \\
\midrule
HRL-Cat
& 0.343(0.0135) & 0.515(0.0105)
& 0.322(0.0157) & 0.488(0.0146)
& 0.357(0.0117) & 0.518(0.0101) \\

HRL-GM
& \underline{0.343(0.0135)} & \underline{0.517(0.0103)}
& \underline{0.344(0.0195)} & \underline{0.510(0.0111)}
& \underline{0.358(0.0118)} & \underline{0.519(0.0101)} \\

HRL-Score
& \textbf{0.948(0.0028)} & \textbf{0.942(0.0030)}
& \textbf{0.954(0.0013)} & \textbf{0.943(0.0014)}
& \textbf{0.951(0.0010)} & \textbf{0.939(0.0012)} \\

HRL-Fixed
& 0.338(0.0136) & 0.501(0.0123)
& 0.314(0.0171) & 0.471(0.0164)
& 0.351(0.0130) & 0.504(0.0117) \\
\bottomrule
\end{tabular}
\end{table}

\subsection{Systematic Stability Study for HRL-Score}\label{app:stability}
In this subsection, we use the heavy-tailed $t_4$-mixture design in Appendix~\ref{append:sim-continuous-random-effects} to study how sensitivity HRL-Score is to initialization, Monte Carlo sample size and score-network hyperparameters. Throughout this subsection, we fix $m=1000$, $n=360$, and use five repetitions for all experiments. The sensitivity study is conducted in a one-factor-at-a-time way, with the baseline having 512 Monte Carlo particles, hidden width $H=128$, learning rate $2\times10^{-3}$, and 400 pretraining steps. Unless explicitly noted below, every fit uses $1000$ training epochs. We report $
\mathrm{Err}_Z=\frac{\|\widehat Z-Z\|_F}{\sqrt{nd}}$,$
\mathrm{Err}_F=\frac{\|\widehat F-F\|_F}{\sqrt{md}}$, ARI and NMI as measurements of parameter recovery results, and report running time as an additional criteria.
\paragraph{Initialization.}
We compare the USVT--GMM initialization proposed in Appendix~\ref{sec:initialization}, the same USVT $(Z^{(0)},\alpha^{(0)})$ with an untrained $N(0,I_d)$ prior, and a fully random $Z^{(0)}$ with degree-based $\alpha^{(0)}$ and an $N(0,I_d)$ prior. The results are shown in Table \ref{tab:score-init}.

\begin{table}[h]
\centering
\scriptsize
\caption{HRL-Score{} sensitivity to initialization. Values are mean (standard error) over five repetitions.}
\label{tab:score-init}
\begin{tabularx}{\textwidth}{Xccccc}
\toprule
Initialization & $\mathrm{Err}_Z\downarrow$ & $\mathrm{Err}_F\downarrow$ & ARI$\uparrow$ & NMI$\uparrow$ & Runtime (s) \\
\midrule
USVT--GMM
& \mc{0.1123}{0.0077} & \mc{0.1384}{0.0079}
& \mc{0.9591}{0.0015} & \mc{0.9529}{0.0020}
& \mc{164.9}{4.5} \\
USVT $(Z,\alpha)$ + $N(0,I_d)$ prior
& \mc{0.1124}{0.0083} & \mc{0.1288}{0.0056}
& \mc{0.9560}{0.0014} & \mc{0.9508}{0.0014}
& \mc{152.7}{3.9} \\
Random $Z$ + degree $\alpha$ + $N(0,I_d)$ prior
& \mc{0.7070}{0.1839} & \mc{0.5733}{0.1354}
& \mc{0.8450}{0.0796} & \mc{0.8745}{0.0515}
& \mc{157.6}{4.5} \\
\bottomrule
\end{tabularx}
\end{table}
The initialization sensitivity study shows that the structured USVT initialization of $(Z,\alpha)$ is the important component. Replace only the pretrained prior by $N(0,I_d)$ changes little, whereas fully random $Z$ produces substantially larger errors, lower clustering accuracy, and much higher variability. The epoch traces also show that the fully random run becomes unstable after roughly the middle of training, while both USVT variants settle to stable low-error trajectories. This shows that our proposed initialization method for $(Z,\alpha)$ is indeed quite useful and important.
\paragraph{Monte Carlo sample size.} We range Monte Carlo sample size in $B\in\{64,128,256,512,1024\}$. As shown in Table \ref{tab:score-particles}, in this design, there is a clear stability threshold in this design: $64$--$256$ particles yield large errors and unstable traces, whereas $512$ and $1024$ are stable. Increasing from $512$ to $1024$ gives only a modest improvement in ARI/NMI and latent recovery while increasing runtime by about $50\%$. We therefore retain $512$ as the default.

\begin{table}[tbp]
\centering
\scriptsize
\caption{HRL-Score{} sensitivity to the number of Monte Carlo particles.}
\label{tab:score-particles}
\begin{tabular}{@{}rccccc@{}}
\toprule
Particles & $\mathrm{Err}_Z\downarrow$ & $\mathrm{Err}_F\downarrow$ & ARI$\uparrow$ & NMI$\uparrow$ & Runtime (s) \\
\midrule
64   & \mc{0.8559}{0.1847} & \mc{0.7584}{0.0263} & \mc{0.5607}{0.0233} & \mc{0.6788}{0.0202} & \mc{94.7}{0.9} \\
128  & \mc{0.7264}{0.3129} & \mc{0.4787}{0.0904} & \mc{0.8539}{0.0285} & \mc{0.8803}{0.0171} & \mc{106.6}{2.0} \\
256  & \mc{0.8919}{0.7597} & \mc{0.3196}{0.1329} & \mc{0.8440}{0.0825} & \mc{0.8708}{0.0584} & \mc{121.0}{2.1} \\
512  & \mc{0.1123}{0.0077} & \mc{0.1384}{0.0079} & \mc{0.9591}{0.0015} & \mc{0.9529}{0.0020} & \mc{164.9}{4.5} \\
1024 & \mc{0.1087}{0.0065} & \mc{0.1108}{0.0070} & \mc{0.9645}{0.0014} & \mc{0.9587}{0.0013} & \mc{248.0}{2.4} \\
\bottomrule
\end{tabular}
\end{table}
\paragraph{Score-network hyperparameters.}
We assess the sensitivity of HRL-Score to pretraining steps, hidden width,
and learning rate. The training trajectories of $\mathrm{Err}_Z$ indicate
that narrower networks require more epochs to stabilize. Accordingly, we
extend training beyond $1000$ epochs for $H=32$ and $H=64$, whereas
$H=128$ stabilizes within $1000$ epochs. Table~\ref{tab:score-hyper}
summarizes the results.

\begin{table}[tbp]
\centering
\scriptsize
\setlength{\tabcolsep}{3.0pt}
\caption{HRL-Score{} sensitivity to pretraining, hidden width, and learning rate. Widths $H=64$ and $H=32$ were extended beyond $1000$ epochs. All other settings use $1000$ epochs.}
\label{tab:score-hyper}
\begin{tabular}{@{}llccccc@{}}
\toprule
Factor & Setting / evaluation window
& $\mathrm{Err}_Z\downarrow$ & $\mathrm{Err}_F\downarrow$
& ARI$\uparrow$ & NMI$\uparrow$ & Runtime (s) \\
\midrule
Pretrain & 200
& \mc{0.1284}{0.0206} & \mc{0.1899}{0.0597}
& \mc{0.9570}{0.0015} & \mc{0.9511}{0.0017} & \mc{165.8}{1.8} \\
& 400
& \mc{0.1123}{0.0077} & \mc{0.1384}{0.0079}
& \mc{0.9591}{0.0015} & \mc{0.9529}{0.0020} & \mc{164.9}{4.5} \\
& 600
& \mc{0.1124}{0.0076} & \mc{0.1358}{0.0071}
& \mc{0.9562}{0.0022} & \mc{0.9516}{0.0017} & \mc{171.7}{2.0} \\
& 800
& \mc{0.1201}{0.0077} & \mc{0.1569}{0.0249}
& \mc{0.9582}{0.0013} & \mc{0.9526}{0.0015} & \mc{163.3}{6.1} \\
\midrule
Width & $H=32$ (epochs 1200--1400)
& \mc{0.1178}{0.0069} & \mc{0.1397}{0.0100}
& \mc{0.9596}{0.0019} & \mc{0.9546}{0.0019} & \mc{160.9}{4.2} \\
& $H=64$ (epochs 1000--1200)
& \mc{0.1075}{0.0045} & \mc{0.1242}{0.0051}
& \mc{0.9592}{0.0021} & \mc{0.9538}{0.0022} & \mc{156.8}{2.8} \\
& $H=128$ (epochs 800--1000)
& \mc{0.1123}{0.0077} & \mc{0.1384}{0.0079}
& \mc{0.9591}{0.0015} & \mc{0.9529}{0.0020} & \mc{164.9}{4.5} \\
\midrule
Learning rate & $5\times10^{-4}$
& \mc{0.1706}{0.0145} & \mc{0.2519}{0.0455}
& \mc{0.9574}{0.0018} & \mc{0.9516}{0.0023} & \mc{161.4}{4.6} \\
& $10^{-3}$
& \mc{0.1200}{0.0067} & \mc{0.1397}{0.0047}
& \mc{0.9560}{0.0031} & \mc{0.9511}{0.0027} & \mc{150.2}{2.1} \\
& $2\times10^{-3}$
& \mc{0.1123}{0.0077} & \mc{0.1384}{0.0079}
& \mc{0.9591}{0.0015} & \mc{0.9529}{0.0020} & \mc{164.9}{4.5} \\
\bottomrule
\end{tabular}
\end{table}

We find that the final performance is insensitive to $200$--$800$ pretraining steps. Widths $32$--$128$ also reach essentially the same ARI/NMI, although narrower networks need more epochs to settle; $H=64$ gives the lowest average latent errors. Learning rates $10^{-3}$ and $2\times10^{-3}$ are both stable, while $5\times10^{-4}$ converges more slowly and gives larger latent errors. Together with the particle study, these results indicate that HRL-Score{} is robust to moderate architecture and pretraining changes once it has a structured initialization and a sufficiently large Monte Carlo sample.

\subsection{Additional Real Data Analysis Results}\label{sec:real-data-details}
This section supplements the real-data experiments in Section~\ref{sec:real-data}. We provide task-specific protocols, implementation details, extended empirical results, and additional baseline comparisons. 

Across all real-data experiments, the latent dimension \(d\) is selected using the procedure in Appendix~\ref{append:select}. For HRL methods, the number of latent components \(K\), when applicable, is selected jointly with \(d\). After selection, \(d\) is fixed for each dataset and used for all methods so that downstream comparisons are based on representations of the same dimension. For non-HRL baselines, \(d\) refers to the dimension of the final learned node embedding or the output dimension of the final representation layer. 

\subsubsection{Implementation Details for the Row-disjoint Cold-start Experiment}
\label{append:missing-value-implementation}

Let \(\bV=(v_{ji})\in\{0,1\}^{m\times n}\) denote the incidence matrix,
where \(v_{ji}=1\) if node \(i\) belongs to hyperlink \(j\).  The aim of
this experiment is to predict the existence of unobserved nodes of previously unseen
hyperlinks, rather than to impute entries of hyperlinks that were already used for model fitting. For each of five repetitions, we
randomly partition the hyperlink indices into
disjoint training, validation, and test sets,
\(\mathcal J_{\mathrm{tr}}\), \(\mathcal J_{\mathrm{val}}\), and
\(\mathcal J_{\mathrm{te}}\), containing \(70\%\), \(10\%\), and \(20\%\)
of the rows, respectively.  All entries of the training hyperlinks are
observed. For every validation or test hyperlink \(j\), we reveal a fraction
\(r_{\mathrm{obs}}\) of its coordinates and predict the remainder, and denote
\(\mathcal O_j\), \(\mathcal H_j=[n]\setminus\mathcal O_j\) for the
revealed and hidden coordinates of hyperlink \(j\), respectively.  The revealed entries are stratified by the binary incidence value. Within a
repetition, every method uses exactly the same row split and the same revealed
and hidden coordinates.

\paragraph{HRL-family methods.}
HRL-Cat, HRL-GM, HRL-Score, and HRL-Fixed are fitted to the complete submatrix
\(\{v_{ji}:j\in\mathcal J_{\mathrm{tr}},i\in[n]\}\).  At validation and test
time, their global parameters are frozen, and a new hyperlink is inferred
using only \(\{v_{ji}:i\in\mathcal O_j\}\). We compute the predicted probability \(\widehat p_{ji}\) for each held-out entry using the posterior predictive procedure described in Appendix~\ref{append:select}, calculating \eqref{eq:prediction}. For HRL-GM, the Gaussian--Hermite approximation uses 20 quadrature nodes and weights. For HRL-Score, hyperparameters are selected according to validation AUPRC, and the selected configuration is then refit on the non-test hyperlinks before evaluating on test hyperlinks.

For HRL-Fixed, the factors learned for the training hyperlinks cannot be used
for a new validation or test hyperlink.  We therefore freeze the fitted global
quantities \(\{\widehat{\bz}_i,\widehat\alpha_i\}_{i=1}^n\) and estimate a new
local factor by maximizing the partial Bernoulli log-likelihood
\[
  \widehat{\bff}_j
  =\arg\max_{\bff\in\mathbb R^K}
  \sum_{i\in\mathcal O_j}
  \left[
    v_{ji}(\bff^\top\widehat{\bz}_i+\widehat\alpha_i)
    -\log\{1+\exp(\bff^\top\widehat{\bz}_i+\widehat\alpha_i)\}
  \right].
\]
We initialize this optimization at the fitted training-hyperlink factor having
the largest revealed-coordinate likelihood and use damped Newton updates with
backtracking.  The final probability is
\[
  \widehat p_{ji}
  =\sigma\!\left(
      \widehat{\bff}_j^\top\widehat{\bz}_i+\widehat\alpha_i
    \right),
  \qquad i\in\mathcal H_j.
\]

\paragraph{Cold-start adaptation of non-HRL baselines.} Here we distinguish models containing a
training-row-specific latent factor (i.e. HRL-Fixed like) from shared encoders that can directly process a partially observed new hyperlink.

NNMF and NeuMF contain two hyperlink-specific \(K\)-dimensional embedding
tables, whereas LightGCN contains one.  For each new validation or test
hyperlink, we freeze all node/item embeddings and network parameters and
optimize only the corresponding new local hyperlink vector(s) on the revealed
coordinates. The resulting model logit for a hidden coordinate is used as its raw score.
MultiVAE has a shared row encoder and therefore requires no local optimization. The revealed positive coordinates are supplied as a partial incidence vector,
and the corresponding decoder logits are used for the hidden coordinates.

EHGNN, AllSet, HypeBoy, and GraphSAGE are adapted as shared inductive encoders.
Only revealed positive incidences are inserted into the partial hypergraph, while revealed zeros are treated as observed nonedges. For these four methods, node and hyperlink representations are connected by a
bilinear incidence decoder.  For a training pair, its logit has the form
\[
  s_{ji}=\bff_j^\top\bW\bz_i+a_i+b_j+c.
\]
The encoder and decoder are trained jointly on the complete training incidence
matrix using binary cross-entropy with logits.  During validation/testing, since a held-out hyperlink has no learned training-row bias, we set it to the mean value of learned training hyperlink bias. 

\paragraph{Hyperparameter selection and probability calibration.}
For every hyperparameter candidate configuration, we train the model on $\cJ_{\text{train}}$ and compute AUPRC following previously described procedure using validation hyperlinks. We then select hyerparameter based on validation AUPRC. After hyperparameter selection and final training, each non-HRL method produces
an uncalibrated validation raw score \(s_{ji}\).  We then fit a one-dimensional
Platt model using only the hidden validation coordinates,
\[
  (\widehat\gamma_0,\widehat\gamma_1)
  =\arg\min_{\gamma_0,\gamma_1}
  \sum_{j\in\mathcal J_{\mathrm{val}}}
  \sum_{i\in\mathcal H_j}
  \mathrm{BCE}\!\left(
    \sigma(\gamma_0+\gamma_1s_{ji}),v_{ji}
  \right).
\]
Then we apply the same fitted calibration map to the hidden test logits,
\[
  \widehat p_{ji}
  =\sigma(\widehat\gamma_0+\widehat\gamma_1s_{ji}),
  \qquad
  j\in\mathcal J_{\mathrm{te}},\ i\in\mathcal H_j.
\]
This calibration step places the raw outputs of different non-HRL baselines on a common probability scale, which is needed for BCE and Brier-score comparisons.

\paragraph{Capacity-matched score models.}
In order to separate the effect of hyperlink random effects from model capacity and inference complexity, in the all-dataset comparison, we deliberately choose the width of HRL-Score so that the total number of parameters contained in HRL-Score is smaller than that in HRL-Fixed. The widths are $H=26$ (MIMIC-III), $H=19$ (Food Recipe), and $H=31$ (MovieLens). Their full retained counts are 6,039, 3,776, and 8,078, respectively, compared with HRL-Fixed's 6,125, 3,837, and 8,183. 

\paragraph{Evaluation and computational measurements.}
For every method, repetition, and reveal fraction, we compute AUROC, AUPRC,
binary cross-entropy, and Brier score over the hidden test coordinates only.
We report means and standard errors across the five repetitions.  

\subsubsection{Node Clustering}\label{append:task-node} 
We evaluate whether the learned node embeddings capture meaningful semantic structure by applying clustering to the fitted representations. For each dataset, we first select \((K,d)\) using the cross-validation procedure described in~\ref{append:select}. For HRL-Cat, HRL-GM, and HRL-Fixed, no additional hyperparameters are required once \((K,d)\) is fixed. We therefore refit these methods on the full hypergraph and use the resulting node embeddings for clustering.

For HRL-Score and the non-HRL baselines, we fix the final embedding dimension to be \(d\) for comparability. Remaining hyperparameters are optimized via a missing-value prediction task (described in Appendix~\ref{append:missing-value-implementation}), using a consistent held-out validation set. We select the configuration yielding the lowest validation Binary Cross-Entropy (BCE) before refitting the models on the full hypergraph.

For methods that naturally output node embeddings, we use these embeddings directly. For NeuMF, which contains both GMF and MLP components, we concatenate the two node-side representations. For methods that return both node embeddings \(\widehat{\bZ}\) and node intercepts \(\widehat{\boldsymbol\alpha}\), such as the HRL methods, we consider several node representations,
\[
    \widehat{\bZ},\qquad
    [\widehat{\bZ},\widehat{\boldsymbol\alpha}],\qquad
    \mathrm{scale}(\widehat{\bZ}),\qquad
    [\mathrm{scale}(\widehat{\bZ}),\mathrm{scale}(\widehat{\boldsymbol\alpha})],
\]
and report the representation that gives the most coherent clustering behavior. MultiVAE is excluded from the node-clustering comparison since it does not naturally provide an explicit node-side embedding.

Given the final node embeddings, we apply both Gaussian mixture model (GMM) clustering and \(K\)-means clustering. For GMM, we use the posterior membership probabilities to identify high-confidence representatives and we define the core members of cluster \(g\) are defined as
\[
    \mathcal C_g^{\mathrm{core}}
    =
    \{i:\Pr(g_i=g\mid \widehat{\bz}_i)\geq 0.9\}.
\]
For \(K\)-means, since posterior probabilities are not available, we define the center members of cluster \(g\) as the nodes assigned to \(g\) that are closest to its centroid in Euclidean distance. These nodes serve as representative examples of the cluster. 

We conduct clustering experiments on both food-recipe and medical-record hypergraphs. Since both datasets do not provide ground-truth node labels, we evaluate clustering performance qualitatively through carefully checking the semantic coherence of the resulting core or center members.

\subsubsection{Medical Record Hypergraph Built from MIMIC-III}\label{sec:mimic}
As discussed in the main text, a motivating real-world application of our framework is the analysis of heterogeneous patient profiles, where multiple clinical conditions, complications, and care-related events may co-occur within the same patient record. To evaluate whether the proposed methods can recover clinically meaningful structures, we consider the Medical Information Mart for Intensive Care (MIMIC-III; \citep{johnson2016mimic}), which contains clinical data for over 45,000 patients admitted to Beth Israel Deaconess Medical Center between 2001 and 2012.

We construct a hypergraph using International Classification of Diseases, Ninth Revision (ICD-9) codes recorded in MIMIC-III. ICD-9 codes are standardized clinical annotations associated with patient records. They are heterogeneous by design. In addition to disease diagnoses, ICD-9 also encode complications, procedure or device status, prior medical history, external causes, and care-related conditions. Consequently, the clusters identified in the following analysis should be interpreted as clinical phenotypes rather than isolated disease categories.

The hypergraph construction and preprocessing follow a systematic filter for clinical richness. For pre-processing, we first select the 300 most frequently recorded ICD-9 codes in the dataset. To focus on clinically rich patient profiles, we then retain the 1000 patient having largest numbers of ICD-9 records. After restricting to these 1000 patients, we further remove selected ICD-9 codes that do not appear in this subset, resulting in a final hypergraph with 1000 patients and 275 distinct ICD-9 codes. In this hypergraph, each node represents an ICD-9 clinical-code concept, and each hyperlink corresponds to a patient profile, connecting all ICD-9 codes recorded for that patient. For the constructed medical-record hypergraph, the cross-validation procedure suggests \(K=7\) and \(d=6\).

\textbf{Row-distinct cold-start experiment}

Here we first report the mean value and standard errors of measurements used in the cold-start experiment in Table \ref{tab:mimic_III}, the results are based on 5 repetitions.

\begin{table}[tbp]
\centering
\small
\setlength{\tabcolsep}{4pt}
\caption{MIMIC-III: cold-start prediction at $r_{\rm obs}=0.2$. Values are mean (standard error) across five repetitions.}
\label{tab:mimic_III}
\begin{tabular}{lcccc}
\toprule
Method & AUROC $\uparrow$ & AUPRC $\uparrow$ & BCE $\downarrow$ & Brier $\downarrow$ \\
\midrule
HRL-Cat & 0.771 (0.0022) & 0.396 (0.0041) & 0.334 (0.0020) & 0.097 (0.0006) \\
HRL-GM & \underline{0.795} (0.0014) & \textbf{0.437} (0.0038) & \underline{0.315} (0.0011) & \textbf{0.093} (0.0004) \\
HRL-Score ($H=26$) & \textbf{0.796} (0.0014) & \underline{0.436} (0.0035) & \textbf{0.314} (0.0012) & \underline{0.093} (0.0004) \\
HRL-Fixed & 0.745 (0.0031) & 0.337 (0.0069) & 0.506 (0.0115) & 0.111 (0.0011) \\
NeuMF & 0.768 (0.0014) & 0.386 (0.0031) & 0.330 (0.0016) & 0.097 (0.0006) \\
AllSet & 0.765 (0.0022) & 0.378 (0.0021) & 0.331 (0.0018) & 0.098 (0.0006) \\
HypeBoy & 0.768 (0.0017) & 0.388 (0.0040) & 0.329 (0.0019) & 0.097 (0.0007) \\
MultiVAE & 0.769 (0.0013) & 0.386 (0.0018) & 0.330 (0.0015) & 0.098 (0.0005) \\
LightGCN & 0.658 (0.0032) & 0.237 (0.0048) & 0.369 (0.0018) & 0.109 (0.0007) \\
EHGNN & 0.769 (0.0015) & 0.377 (0.0052) & 0.347 (0.0017) & 0.104 (0.0006) \\
NNMF & 0.749 (0.0048) & 0.348 (0.0062) & 0.338 (0.0014) & 0.100 (0.0004) \\
GraphSAGE & 0.768 (0.0011) & 0.383 (0.0049) & 0.339 (0.0009) & 0.101 (0.0005) \\
\bottomrule
\end{tabular}
\end{table}

In order to better compare the performance of four HRL-based approaches, we also considered ranging $r_{obs}\in\{0.2,0.5,0.8\}$ in the cold-start prediction task on MIMIC-III dataset. The results are shown in \ref{tab:cold-reveal}.

\begin{table}[tbp]
\centering
\small
\setlength{\tabcolsep}{4pt}
\caption{MIMIC-III reveal-fraction study with 70\% global training rows. Values are mean (standard error) over shared repetitions. Best values are bold and second-best values are underlined.}
\label{tab:cold-reveal}
\begin{tabular}{clcccc}
\toprule
$r_{\rm obs}$ & Method & AUROC $\uparrow$ & AUPRC $\uparrow$ & BCE $\downarrow$ & Brier $\downarrow$ \\
\midrule
0.2 & HRL-Cat & 0.771 (0.0022) & 0.396 (0.0041) & 0.334 (0.0020) & 0.097 (0.0006) \\
 & HRL-GM & \underline{0.795} (0.0014) & \textbf{0.437} (0.0038) & \underline{0.315} (0.0011) & \textbf{0.093} (0.0004) \\
 & HRL-Score & \textbf{0.796} (0.0010) & \underline{0.435} (0.0040) & \textbf{0.314} (0.0010) & \underline{0.093} (0.0004) \\
 & HRL-Fixed & 0.745 (0.0031) & 0.337 (0.0069) & 0.506 (0.0115) & 0.111 (0.0011) \\
\midrule
0.5 & HRL-Cat & 0.785 (0.0020) & 0.413 (0.0034) & 0.323 (0.0014) & 0.094 (0.0005) \\
 & HRL-GM & \underline{0.809} (0.0017) & \textbf{0.460} (0.0048) & \underline{0.305} (0.0009) & \textbf{0.090} (0.0003) \\
 & HRL-Score & \textbf{0.810} (0.0016) & \underline{0.459} (0.0045) & \textbf{0.304} (0.0008) & \underline{0.090} (0.0003) \\
 & HRL-Fixed & 0.791 (0.0016) & 0.421 (0.0054) & 0.379 (0.0059) & 0.095 (0.0001) \\
\midrule
0.8 & HRL-Cat & 0.792 (0.0019) & 0.423 (0.0045) & 0.321 (0.0018) & 0.094 (0.0008) \\
 & HRL-GM & \textbf{0.815} (0.0037) & \textbf{0.472} (0.0057) & \underline{0.304} (0.0023) & \textbf{0.090} (0.0007) \\
 & HRL-Score & \underline{0.815} (0.0033) & \underline{0.471} (0.0053) & \textbf{0.303} (0.0021) & \underline{0.090} (0.0007) \\
 & HRL-Fixed & 0.807 (0.0029) & 0.457 (0.0044) & 0.349 (0.0042) & 0.092 (0.0007) \\
\bottomrule
\end{tabular}
\end{table}

We can clearly observe that the gain is largest when the new hyperlink is weakly observed. At $r_{\mathrm{obs}}=0.2$, HRL-GM and HRL-Score raise AUPRC from $0.336$ HRL-Fixed to $0.437$ and reduce BCE from $0.506$ to $0.315/0.314$. At $r_{\mathrm{obs}}=0.8$, the AUC gap narrows to $0.815$ versus $0.807$. This matches well with how these HRL approaches actually work. The random-effect methods combine the observed coordinates with the learned hyperlink distribution $P_F$, while HRL-Fixed must estimate a new local $f_j$ from the observed coordinates alone. Thus learn the link-level distribution is most useful when the new row contains little row-specific information, while the fixed-effect estimate catches up as more of that row is revealed. Clinically, this corresponds to a new patient with only a partial set of recorded ICD-9 codes: $P_F$ transfers recurrent patient-level contexts learned from previous patients to predict additional unobserved codes.

\textbf{ICD-9 code clustering}

We first obtain node embeddings from both the proposed HRL methods and all baseline methods following the procedure described in Section~\ref{append:task-node}. Given the learned node embeddings, we then apply Gaussian mixture model (GMM) clustering to obtain soft cluster assignments for ICD-9 codes, varying the number of GMM components, denoted as \(G\) from 3 to 15. K-means are also applied, but none of the clustering results are reasonable. We thus mainly focus on analyzing GMM clustering results.

As mentioned in Appendix~\ref{append:task-node}, since the MIMIC-III data do not provide ground-truth labels for ICD-9-code clusters, we evaluate clustering quality qualitatively by manually inspecting the clinical meanings of these representative codes. In particular, we assess whether the codes within each cluster indeed share similar clinical meanings. Here we define the representative codes as those with the highest posterior membership probabilities among the core members. If a cluster contains fewer than 20 core members, we use all of them. Otherwise, we consider only the top 20. 

\textbf{Among the methods and values of \(G\) considered, HRL-Score with \(G=7\) gives the most interpretable and balanced clustering.} For readability, we only report the top representative ICD-9 concepts for each cluster in Table~\ref{tab:mimic-Score-K7} .

\begin{table}[htbp]
\centering
\small
\caption{ICD-9 code cluster learned from HRL-Score with $G=7$}
\label{tab:mimic-Score-K7}
\begin{tabularx}{\textwidth}{
>{\centering\arraybackslash}m{1.0cm}
>{\centering\arraybackslash}m{0.9cm}
>{\centering\arraybackslash}m{3cm}
X
}
\hline
\textbf{Cluster} &  \textbf{Core size} &\textbf{Manual name}
& \textbf{Top 5 representative ICD-9 concepts} \\
\hline

1&  34 & Cardiac disease 
& \diseaselist{
Native coronary atherosclerosis (414.01)\\
PCI status (V45.82)\\
Old myocardial infarction (412)\\
CABG status (V45.81)\\
Atrial fibrillation (427.31)
} \\
\hline

2& 33 & Cardiometabolic comorbidities 
& \diseaselist{
Essential hypertension (401.9)\\
Morbid obesity (278.01)\\
Chronic diastolic heart failure (428.32)\\
Obstructive sleep apnea (327.23)\\
Acute-on-chronic diastolic heart failure (428.33)
} \\
\hline

3 & 63& Postoperative complications
& \diseaselist{
Implant/graft-related surgical complications (E878.2)\\
Post-procedural hematoma (998.12)\\
Post-procedural hemorrhage (998.11)\\
Postoperative infection (998.59)\\
Other procedure-related complications (E878.8)
} \\
\hline

4& 12 & Diabetes-related complications
& \diseaselist{
Type 2 diabetes with neurologic complications (250.60)\\
Diabetic polyneuropathy (357.2)\\
Diabetic retinopathy (362.01)\\
Type 2 diabetes with renal complications (250.40)\\
Nephritis/nephropathy(583.81)
} \\
\hline

5 & 4 & Kidney-related
& \diseaselist{
Chronic kidney disease, unspecified (585.9)\\
Hypertensive CKD, stage I--IV or unspecified (403.90)\\
Stage III CKD (585.3)\\
Acute kidney failure (584.9)
} \\
\hline

6& 31& Liver-related complications
& \diseaselist{
Portal hypertension (572.3)\\
Hepatic encephalopathy (572.2)\\
Liver cirrhosis (571.5)\\
Ascites (789.59)\\
Alcoholic cirrhosis (571.2)
} \\
\hline

7& 21& Severe infection / inflammatory complications
& \diseaselist{
Septic shock (785.52)\\
Severe sepsis (995.92)\\
Septicemia (038.9)\\
Staphylococcal pneumonia / MSSA pneumonia (482.41)\\
Penicillin-resistant infection (V09.0)
} \\
\hline
\end{tabularx}
\end{table}

Generally speaking, the clustering in Table~\ref{tab:mimic-Score-K7} gives a compact and largely interpretable summary of the ICD-9 code space. The first two clusters are both cardiovascular-related, but they capture different aspects of the cardiovascular domain: the first is more directly associated with severe cardiac disease, while the second is closer to a heart-failure-related comorbidity profile. Similarly, although the diabetes-related cluster and the kidney-related cluster are clinically related, they are not redundant. The former contains codes explicitly tied to diabetic complications, whereas the latter consists of more general kidney dysfunction codes without explicit diabetes-specific coding. We interpret the kidney-related cluster cautiously, however, because it contains only four core members.

We also notice that the third cluster is quite large with 63 core members. This broadness is understandable given that ICD-9 coding covers a wide range of clinical annotations rather than pure diseases. Thus, the third cluster is better understood as a broad procedure-related care phenotype. 

We also examined larger values of \(G\) to determine whether a finer partition improves interpretability. Increasing \(G\) can reveal additional local structure, but it also tends to fragment several already interpretable groups. As an example, Table~\ref{tab:mimic-Score-K11} reports the HRL-Score result with \(G=11\).

\begin{table}[htbp]
\centering
\small
\caption{ICD-9 code cluster learned from HRL-Score with $G=11$}
\label{tab:mimic-Score-K11}
\begin{tabularx}{\textwidth}{
>{\centering\arraybackslash}m{0.7cm}
>{\centering\arraybackslash}m{0.9cm}
>{\centering\arraybackslash}m{3.0cm}
X
}
\hline
\textbf{Cluster} 
& \textbf{Core Size} & {\textbf{Manual name}}
& \textbf{Top 5 ICD-9 concepts} \\
\hline

1 & 22
& Liver-related complications
& \diseaselist{
Portal hypertension (572.3)\\
Other ascites (789.59)\\
Hepatic encephalopathy (572.2)\\
Cirrhosis of liver without mention of alcohol (571.5)\\
Alcoholic cirrhosis of liver (571.2)
} \\
\hline

2 & 27
& Cardiometabolic
& \diseaselist{
Unspecified essential hypertension (401.9)\\
Chronic diastolic heart failure (428.32)\\
Morbid obesity (278.01)\\
Obstructive sleep apnea (adult)(pediatric) (327.23)\\
Iatrogenic pneumothorax (512.1)
} \\
\hline

3 & 35
& Cardiac disease
& \diseaselist{
Aortocoronary bypass status (V45.81)\\
Bypass-graft coronary atherosclerosis (414.02)\\
Native coronary atherosclerosis (414.01)\\
Cardiac device/graft complications(996.72)\\
Subendocardial infarction, initial episode of care (410.71)
} \\
\hline

4 & 53
& Malignancy / neurologic complications
& \diseaselist{
Obstructive hydrocephalus (331.4)\\
Intracerebral hemorrhage (431)\\
Secondary malignancy, other sites (198.89)\\
History of lung cancer (V10.11)\\
Brain/spinal cord metastasis (198.3)
} \\
\hline

5 & 22
& Hematologic comorbidity
& \diseaselist{
Acute posthemorrhagic anemia (285.1)\\
Venous peripheral insufficiency, unspecified (459.81)\\
Do not resuscitate status (V49.86)\\
Other and unspecified hyperlipidemia (272.4)\\
Hyperpotassemia (276.7)
} \\
\hline

6 & 8
& Diabetes-related complications
& \diseaselist{
Diabetic retinopathy (362.01)\\
Type 2 diabetes with neurologic complications (250.60)\\
Diabetic polyneuropathy (357.2)\\
Type 2 diabetes with renal complications (250.40)\\
Nephropathy in diseases classified elsewhere (583.81)
} \\
\hline

7 & 20
& Severe infection / inflammatory complications
& \diseaselist{
Septic shock (785.52)\\
Severe sepsis (995.92)\\
Unspecified septicemia (038.9)\\
Staphylococcal pneumonia / MSSA pneumonia (482.41)\\
Tracheostomy status (V44.0)
} \\
\hline

8 & 4
& Dialysis disease
& \diseaselist{
Renal dialysis status (V45.11)\\
Anemia in chronic kidney disease (285.21)\\
End stage renal disease (585.6)\\
Hypertensive CKD, stage V or ESRD (403.91) (403.91)
} \\
\hline

9 & 2
& Heart failure
& \diseaselist{
Atrial fibrillation (427.31)\\
Congestive heart failure, unspecified (428.0)
} \\
\hline

10 & 4
& CKD-related renal dysfunction
& \diseaselist{
Acute kidney failure, unspecified (584.9)\\
Chronic kidney disease, stage III moderate (585.3)\\
Chronic kidney disease, unspecified (585.9)\\
Hypertensive CKD, stage I–IV or unspecified (403.90)
} \\
\hline

11 & 27
& Severe infection / inflammatory complications
& \diseaselist{
Implant/graft-related surgical complication (E878.2)\\
Cardiac complications, not elsewhere (997.1)\\
Acute vascular insufficiency of intestine (557.0)\\
Hemorrhage complicating a procedure (998.11)\\
Hematoma complicating a procedure (998.12)
} \\
\hline
\end{tabularx}
\end{table}

Compared with Table \ref{tab:mimic-Score-K7}, the 11-cluster HRL-Score based clustering provides a finer partition of the ICD-9 code space, but the improvement in interpretability is mixed. Several high-level groups from \(G=7\) remain recognizable, including cardiac, liver-related, diabetes-related, infection-related, and procedure-related components. And the postoperative complication cluster becomes substantially smaller, suggesting that the finer clustering separates out some codes that were previously grouped under a broad procedure-related phenotype. The 11-cluster result also identifies an additional coherent hematologic-related cluster. However, this added resolution also leads to some over-fragmentation. In particular, kidney-related and cardiac-related codes are each further split into small components, but these small clusters do not appear to be clearly separated from their corresponding clusters in Table \ref{tab:mimic-Score-K7} in terms of semantic interpretation. In addition, the fourth cluster appears to mix neurologic and malignancy-related codes, making its clinical interpretation less direct. 

We next compare HRL-Score with embedding-based baseline methods. Among LightGCN, EHGNN, NNMF, NeuMF, and HRL-Fixed, none produced a partition as clean and balanced as the HRL-Score result above. The most semantically reasonable LightGCN solution in the cluster sweep was a coarse \(G=3\) partition, which captured broad super-groups. When we increase the cluster numbers, the new components do not become clearly separated, and the remaining clusters are still heavily mixed. For better comparison, we represent the result based on LightGCN at $G=7$ in Table \ref{tab:mimic_LightGCN}.

\begin{table}[htbp]
\centering
\small
\caption{ICD-9 code clusters learned from LightGCN with $G=7$.}
\label{tab:mimic_LightGCN}
\begin{tabularx}{\textwidth}{
>{\centering\arraybackslash}m{0.8cm}
>{\centering\arraybackslash}m{1.1cm}
>{\centering\arraybackslash}m{3.0cm}
>{\raggedright\arraybackslash}X
}
\hline
\textbf{Cluster} 
& \textbf{Core size} 
& \textbf{Manual name} 
& \textbf{Top 5 representative ICD-9 concepts} \\
\hline

1 & 33 
& Cardiorenal comorbidity
& \diseaselist{
Subendocardial infarction, initial episode of care (410.71)\\
Pleural effusion (511.9)\\
Hypertensive CKD, stage V or ESRD (403.91)\\
Acute kidney failure with tubular necrosis (584.5)\\
Hypovolemia (276.52)
} \\
\hline

2 & 60
& Neurovascular / malignancy-related
& \diseaselist{
Lung contusion (861.21)\\
Thoracic aneurysm (441.2)\\
Alcohol abuse, continuous (305.01)\\
Subarachnoid hemorrhage (430)\\
Cerebral edema (348.5)
} \\
\hline

3 & 18
& Anticoagulation / infection comorbidity
& \diseaselist{
Long-term anticoagulant use (V58.61)\\
Thrombocytopenia (287.5)\\
Aspiration pneumonitis (507.0)\\
Acquired hypothyroidism (244.9)\\
Clostridium difficile infection (008.45)
} \\
\hline

4 & 20
& Liver / fluid-overload complications
& \diseaselist{
Hematemesis (578.0)\\
Urinary complications (997.5)\\
Hepatic encephalopathy (572.2)\\
Alcoholic cirrhosis (571.2)\\
Fluid overload (276.6)
} \\
\hline

5 & 69
& Liver / diabetes mixed residual
& \diseaselist{
Tricuspid valve disease (397.0)\\
Dysthymic disorder (300.4)\\
Chronic hepatitis C (070.54)\\
Acute/subacute liver necrosis (570)\\
Diabetic polyneuropathy (357.2)
} \\
\hline

6 & 7
& Acute cardiorenal comorbidity
& \diseaselist{
Acute kidney failure (584.9)\\
Congestive heart failure (428.0)\\
Atrial fibrillation (427.31)\\
Native coronary atherosclerosis (414.01)\\
Essential hypertension (401.9)
} \\
\hline

7 & 12
& Infection / metabolic comorbidity
& \diseaselist{
Severe sepsis (995.92)\\
Acidosis (276.2)\\
Hyperlipidemia (272.4)\\
Type 2 diabetes without complications (250.00)\\
Pneumonia (486)
} \\
\hline

\end{tabularx}
\end{table}

From Table \ref{tab:mimic_LightGCN}, we can clearly see that the resulting partition obtained by LightGCN at $G=7$ is less clean at the global level. Some clinically related codes are split across multiple clusters, while several clusters combine heterogeneous ICD-9 concepts. In particular, diabetes-related complications and postoperative complications are not recovered as clearly separated phenotypes, and the largest cluster appears to act partly as a mixed residual group without a clear pattern.

Finally, we pay special attention to HRL-Fixed. While HRL-Fixed achieves competitive performance in missing-value prediction and multi-label classification in the main text, its behavior on semantic node clustering is substantially different. Across \(G=3,\ldots,15\), the HRL-Fixed partitions are highly imbalanced. At \(G=3\), the method produces one dominant cluster containing 265 core members, together with two very small clusters. Increasing \(G\) separates several clinically coherent but tiny microclusters, while a large residual group remains. Even at \(G=15\), there is still a residual cluster with 99 core codes, and multiple clusters contain only one or two core members.

This suggests that, for HRL-Fixed, the embedding space captures information useful for predictive tasks but does not organize ICD-9 codes into well-separated semantic groups. In other words, the learned representation appears to encode local predictive associations rather than a globally interpretable geometry for node-level code clustering. 

\subsubsection{Food Recipe Hypergraph}
The \textit{Epicurious -- Recipes with Rating and Nutrition}
dataset\footnote{\href{https://www.kaggle.com/datasets/hugodarwood/epirecipes/data}{https://www.kaggle.com/datasets/hugodarwood/epirecipes/data}} contains 17,736 recipes collected from Epicurious\footnote{\href{https://www.epicurious.com/recipes-menus}{https://www.epicurious.com/recipes-menus}}. We construct a hypergraph by treating each ingredient as a node and each recipe as a hyperedge containing its corresponding ingredients.

For preprocessing, we first merge ingredient names with identical or highly similar meanings, such as ``oat'' and ``oatmeal'', ``chili pepper'' and ``hot pepper''. We also find that in the ingredient list, there are several closely related meat-ingredient variants. For example, the dataset contains both generic pork-related ingredients, such as ``pork'', and more specific ones, such as ``pork chop'', ``pork rib'', and ``pork tenderloin''. We checked the node degrees for these ingredients and found that the generic term has a substantially higher degree, likely because many recipes do not specify the exact cut. We thus merge these variants into a single ``pork'' node. The same procedure is applied to beef- and lamb-related ingredients. After consolidation, we remove ingredients appearing in fewer than five recipes. Then to reduce the impact of incomplete ingredient records, we further retain only recipes containing at least eight ingredients. The resulting hypergraph contains 926 recipes and 249 ingredients. Here for the constructed food-recipe hypergraph, cross-validation procedure suggests choosing $K=6,d=4$.

\textbf{Cold-start experiment}

As a supplement for Table \ref{tab:cold-main} in the main text, here we report standard errors of missing value prediction measurements in Table \ref{tab:food_recipe}, the results are based on 5 repetitions.

\begin{table}[tbp]
\centering
\small
\setlength{\tabcolsep}{4pt}
\caption{Food Recipe: cold-start prediction at $r_{\rm obs}=0.2$. Values are mean (standard error) across five repetitions.}
\label{tab:food_recipe}
\begin{tabular}{lcccc}
\toprule
Method & AUROC $\uparrow$ & AUPRC $\uparrow$ & BCE $\downarrow$ & Brier $\downarrow$ \\
\midrule
HRL-Cat & 0.787 (0.0055) & 0.160 (0.0067) & 0.132 (0.0014) & 0.031 (0.0002) \\
HRL-GM & \textbf{0.827} (0.0036) & \underline{0.194} (0.0093) & \underline{0.121} (0.0010) & \underline{0.030} (0.0002) \\
HRL-Score & \underline{0.826} (0.0034) & \textbf{0.195} (0.0069) & \textbf{0.120} (0.0007) & \textbf{0.029} (0.0001) \\
HRL-Fixed & 0.769 (0.0075) & 0.130 (0.0061) & 0.206 (0.0067) & 0.041 (0.0006) \\
NeuMF & 0.755 (0.0051) & 0.132 (0.0048) & 0.133 (0.0002) & 0.031 (0.0001) \\
AllSet & 0.758 (0.0062) & 0.123 (0.0053) & 0.131 (0.0005) & 0.031 (0.0001) \\
HypeBoy & 0.781 (0.0035) & 0.142 (0.0052) & 0.128 (0.0006) & 0.031 (0.0001) \\
MultiVAE & 0.789 (0.0037) & 0.147 (0.0059) & 0.127 (0.0005) & 0.030 (0.0001) \\
LightGCN & 0.528 (0.0030) & 0.038 (0.0013) & 0.147 (0.0005) & 0.032 (0.0001) \\
EHGNN & 0.774 (0.0051) & 0.141 (0.0061) & 0.129 (0.0007) & 0.031 (0.0001) \\
NNMF & 0.658 (0.0148) & 0.071 (0.0075) & 0.143 (0.0018) & 0.032 (0.0003) \\
GraphSAGE & 0.776 (0.0049) & 0.128 (0.0049) & 0.131 (0.0005) & 0.031 (0.0001) \\
\bottomrule
\end{tabular}
\end{table}

\textbf{Node clustering}

For the node clustering task, we first follow the procedure described in section \ref{append:task-node} to obtain the node embeddings. Intuitively, the majority of the ingredients should belong to only one group, thus we apply K-means clustering on the learned embeddings directly, ranging the number of clusters from 3 to 12 for each embedding result. 

We carefully check the semantic interpretations of those clustering results, and find that using HRL-GM embedding with 5 clusters would give the most meaningful results. To better represent the clustering results, for each cluster, we select the top 10 ingredients that are closest to cluster centers as representative members, the result is shown in Table \ref{tab:recipe-GM}.

\begin{table}[htbp]
\centering
\small
\caption{Ingredient clusters learned from HRL-GM with 5 clusters.}
\label{tab:recipe-GM}
\begin{tabular}{c p{3.2cm} c p{7cm}}
\hline
\textbf{Cluster} & \textbf{Cluster name} & \textbf{Size} & \textbf{Representative ingredients} \\
\hline
1 & Seafood & 22 
& oyster, salmon, cod, snapper, halibut, fish, lobster, crab, scallop, bass \\

2 & Richer flavor profile & 42 
& whiskey, kumquat, cottage cheese, pumpkin, tea, sherry, jam or jelly, bourbon, duck, apple \\

3 & Lighter flavor profile & 68 
& nut, nectarine, egg, cantaloupe, honeydew, butter, hazelnut, orange, amaretto, watermelon \\

4 & Asian/Latin ingredients & 40 
& paprika, lime juice, tofu, coriander, sake, plum, sesame, tomatillo, cumin, bok choy \\

5 & Vegetable-legume savory dishes & 77 
& legume, meat, vegetable, chard, pork, chickpea, parsley, olive, butternut squash, herb \\
\hline
\end{tabular}
\end{table}

Overall, the clustering reveals meaningful culinary co-occurrence patterns rather than strict ingredient taxonomies. The seafood cluster is highly coherent, while the fourth cluster captures a clear Asian/Latin savory profile, grouping ``tofu'' and ``sake'' with ``tomatillo'' and ``cumin''. The two flavor-oriented clusters may appear mixed at first glance, but they correspond to distinct subtypes: one is associated with richer, alcohol-based and cooked flavor profiles (e.g., whiskey, sherry, jam), whereas the other reflects lighter, fresher combinations centered on fruits, dairy, and mild sweetness. This suggests that the model is capturing differences in flavor intensity and culinary usage rather than purely ingredient categories. The last cluster is relatively more heterogeneous. Rather than corresponding to a strict ingredient category, it appears to capture a broad savory cooking pattern, where vegetables and legumes frequently co-occur with herbs and animal proteins. This is consistent with the fact that the clustering is learned from recipe co-occurrence patterns rather than predefined food taxonomies. And in the actual recipes, vegetables and proteins often appear together.

The HRL-GM approach significantly outperforms baseline methods in terms of cluster balance and interpretability. For example, HRL-Fixed has even slightly better performance than our newly proposed three HRL methods on the multi-label classification task. However, it fails to provide a meaningful clustering result. When setting the number of clusters to 5, the resulting clusters with representative ingredients are shown in Table \ref{tab:recipe-Fixed}.

\begin{table}[htbp]
\centering
\small
\caption{Ingredient clusters learned from HRL-Fixed with 5 clusters}
\label{tab:recipe-Fixed}
\begin{tabular}{c p{3.2cm} c p{7cm}}
\hline
\textbf{Cluster} & \textbf{Cluster name} & \textbf{Size} & \textbf{Representative ingredients} \\
\hline
1 & Aromatics & 39
& lime juice, ginger, coriander, seed, chile, lemongrass, mango, lime, paprika, tamarind \\

2 & Nuts  & 17
& walnut, pistachio, pomegranate, date, tofu, chestnut, almond, dried fruit, tree nut, apricot \\

3 & Mixed group & 151
& lentil, phyllo/puff pastry dough, mustard, vinegar, breadcrumbs, ricotta, herb, pork, buttermilk, bread \\

4 & Seafood & 26
& fish, salmon, oyster, crab, cod, halibut, shrimp, snapper, squid, bass \\

5 & Vegetables & 16
& yellow squash, eggplant, broccoli, zucchini, corn, chickpea, plantain, barley, squash, okra \\
\hline
\end{tabular}
\end{table}

Although Table \ref{tab:recipe-Fixed} recovers a reasonably coherent seafood cluster and a vegetable-oriented cluster, it suffers from a ``giant component'' problem, where Cluster 3 collapses 151 heterogeneous ingredients into a single uninformative residual group. We have also tried to increase the number of clusters, but that would result in several over-splitting groups containing only two or three ingredients, while keeping one mixed group with a large size. Therefore, even though this clustering appears to capture some local co-occurrence patterns, it provides a weaker and less interpretable global organization than the HRL-GM based clustering reported in Table~\ref{tab:recipe-GM}.

For the remaining baseline methods, the clustering derived by EHGNN at 5 clusters is the most reasonable one. However, as shown in Table \ref{tab:recipe-EHGNN}, its overall organization is still less coherent than that of HRL-GM in Table~\ref{tab:recipe-GM}.  

\begin{table}[htbp]
\centering
\small
\caption{Ingredient clusters learned from EHGNN with 5 clusters.}
\label{tab:recipe-EHGNN}
\begin{tabular}{c p{3.2cm} c p{7cm}}
\hline
\textbf{Cluster} & \textbf{Cluster name} & \textbf{Size} & \textbf{Representative ingredients} \\
\hline
1 & Mixed fruit, dairy \& protein flavors & 46
& duck, pecan, raisin, pear, honey, orange, brie, swiss cheese, brussels sprout, fig \\

2 & Mixed savory ingredients & 91
& root vegetable, dairy, celery, green bean, citrus, beef, potato, fish, mushroom, shellfish \\

3 & Rich fruit and alcohol flavors & 43
& nutmeg, clove, cantaloupe, blackberry, nectarine, melon, blueberry, brandy, sake, asian pear \\

4 & Dairy, grains \& savory staples & 62
& bulgur, pumpkin, seed, lemongrass, horseradish, chard, mint, phyllo/puff pastry dough, sour cream, bok choy \\

5 & Drinks and spirits & 7
& champagne, sparkling wine, gin, whiskey, anchovy, honeydew, amaretto \\
\hline
\end{tabular}
\end{table}

Compared to the HRL-GM based clustering shown in Table \ref{tab:recipe-GM}, EHGNN tends to blur the boundary between protein-based and non-protein ingredients: for example, seafood is not cleanly separated and partially appears within the mixed cluster. In HRL-GM, by comparison, the seafood cluster is more clearly identified and remains semantically coherent. Apart from that, the first group in EHGNN result still mixes fruits, dairy, meat, and vegetables, making its semantic interpretation less clear. Moreover, the drinks cluster is very small, indicating possible over-fragmentation.

A key observation is that all three presented clusters produce at least one relatively heterogeneous cluster, reflecting the intrinsic challenge of clustering recipe in an unsupervised way based solely on the recipe--ingredient incidence matrix. The key difficulty lies in the mismatch between the intrinsic semantic structure of ingredients and their co-occurrence patterns in recipes. 

On the one hand, ingredients possess well-defined semantic categories (e.g., fruits, dairy products, meats, and vegetables), which suggest a natural clustering structure. On the other hand, recipes are deliberately constructed by combining heterogeneous ingredients from different categories in order to achieve complementary flavors and richer taste profiles. As a result, ingredients that are semantically distinct may frequently co-occur, while ingredients within the same semantic category may not co-occur as often.

This contrasts sharply with co-citation hypergraphs, where, as noted in \citep{ji2022co}, co-citation of two authors typically indicates shared research interests, leading to a stronger alignment between co-occurrence patterns and underlying semantic similarity. In the food-recipe setting, however, co-occurrence reflects culinary composition rather than semantic similarity, which introduces ambiguity in clustering.

Consequently, the natural ingredient taxonomy and the co-occurrence structure induced by recipes often compete with each other, making it difficult for unsupervised methods to recover clean and interpretable clusters from the incidence matrix alone. This explains the observed mixing of ingredient types and the instability of clustering results across different methods.

\subsubsection{Movie Preference Hypergraph}
The MovieLens 1M dataset\footnote{\href{https://grouplens.org/datasets/movielens/}{https://grouplens.org/datasets/movielens/}} contains 1,000,209 ratings from 6,040 users on approximately 3,900 movies. Apart from user rating information, MovieLens dataset also provides movie genre annotations, where the majority of movies contain multiple labels. We use these labels for the multi-label classification task after we learned the node embeddings.

We model the user-movie interaction data as a hypergraph by treating movies as nodes and users as hyperlinks. More specifically, a hyperedge contains a node if the corresponding user assigned a rating of at least 4 stars to that movie, signifying a positive preference. To focus on movies with sufficient interaction information, we first retain the 300 movies with the largest numbers of liked users. We then found that the dataset originally contains 18 different genres, with some associated with very few retained movies. We thus discard those genres appearing in fewer than 30 retained movies and further remove movies with zero degree after filtering. After that, we focus on only the top 1000 hyperlinks with the largest degrees. This preprocessing yields a hypergraph incidence matrix of size $1000 \times 287$, corresponding to 1000 user hyperedges and 287 movie nodes.

For this hypergraph, model selection procedure selects $K=15,d=8$, which are fixed for all following tasks.

\textbf{Cold-start experiment}

Same as we did for MIMIC-III and food-recipe, here we present the cold-start experiment results for movie preference hypergraph, based on 5 repetitions. The result is shown in Table~\ref{tab:movielens}

\begin{table}[tbp]
\centering
\small
\setlength{\tabcolsep}{4pt}
\caption{MovieLens: cold-start prediction at $r_{\rm obs}=0.2$. Each entry reports mean (standard error) over five repetitions.}
\label{tab:movielens}
\begin{tabular}{lcccc}
\toprule
Method & AUROC $\uparrow$ & AUPRC $\uparrow$ & BCE $\downarrow$ & Brier $\downarrow$ \\
\midrule
HRL-Cat & 0.728 (0.0017) & 0.644 (0.0051) & 0.601 (0.0005) & 0.205 (0.0003) \\
HRL-GM & \textbf{0.783} (0.0018) & \textbf{0.712} (0.0049) & \textbf{0.545} (0.0012) & \textbf{0.184} (0.0005) \\
HRL-Score & \underline{0.776} (0.0021) & \underline{0.703} (0.0050) & \underline{0.552} (0.0012) & \underline{0.187} (0.0005) \\
HRL-Fixed & 0.767 (0.0018) & 0.682 (0.0055) & 0.617 (0.0031) & 0.199 (0.0007) \\
NeuMF & 0.754 (0.0034) & 0.664 (0.0083) & 0.576 (0.0023) & 0.196 (0.0011) \\
AllSet & 0.659 (0.0020) & 0.572 (0.0054) & 0.635 (0.0010) & 0.222 (0.0005) \\
HypeBoy & 0.681 (0.0039) & 0.597 (0.0081) & 0.655 (0.0041) & 0.231 (0.0020) \\
MultiVAE & 0.652 (0.0011) & 0.569 (0.0040) & 0.638 (0.0011) & 0.223 (0.0005) \\
LightGCN & 0.698 (0.0019) & 0.610 (0.0062) & 0.612 (0.0005) & 0.212 (0.0002) \\
EHGNN & 0.606 (0.0244) & 0.515 (0.0309) & 0.655 (0.0080) & 0.232 (0.0038) \\
NNMF & 0.693 (0.0049) & 0.615 (0.0086) & 0.615 (0.0022) & 0.213 (0.0010) \\
GraphSAGE & 0.664 (0.0036) & 0.579 (0.0124) & 0.645 (0.0030) & 0.226 (0.0014) \\
\bottomrule
\end{tabular}
\end{table}
\textbf{Multi-label classification}

The multi-label classification task follows the evaluation setting of \cite{zhang2020hypersagnn}. We therefore include Hyper-SAGNN proposed in that paper as an additional baseline. For this task, we consider two kinds of repetitions to check the performance stability. In the first case, we fit all proposed HRL variants and baseline methods only once to obtain final node embeddings, following the procedure described in the first two paragraphs in Appendix~\ref{append:task-node}. Then we randomly split the movie nodes into training and testing sets, with the training proportion varying from $10\%$ to $90\%$. For each training proportion, we consider 100 different train/test split and apply that use the same splits for all methods. For each genre label, we train a separate one-vs-rest logistic regression classifier, using the learned movie embeddings as predictors. Classification performance is evaluated using both Micro-F1 and Macro-F1 scores. The result for this setting is presented in Figure~\ref{fig:movie1}

\begin{figure}[t]
\centering\includegraphics[width=5in]{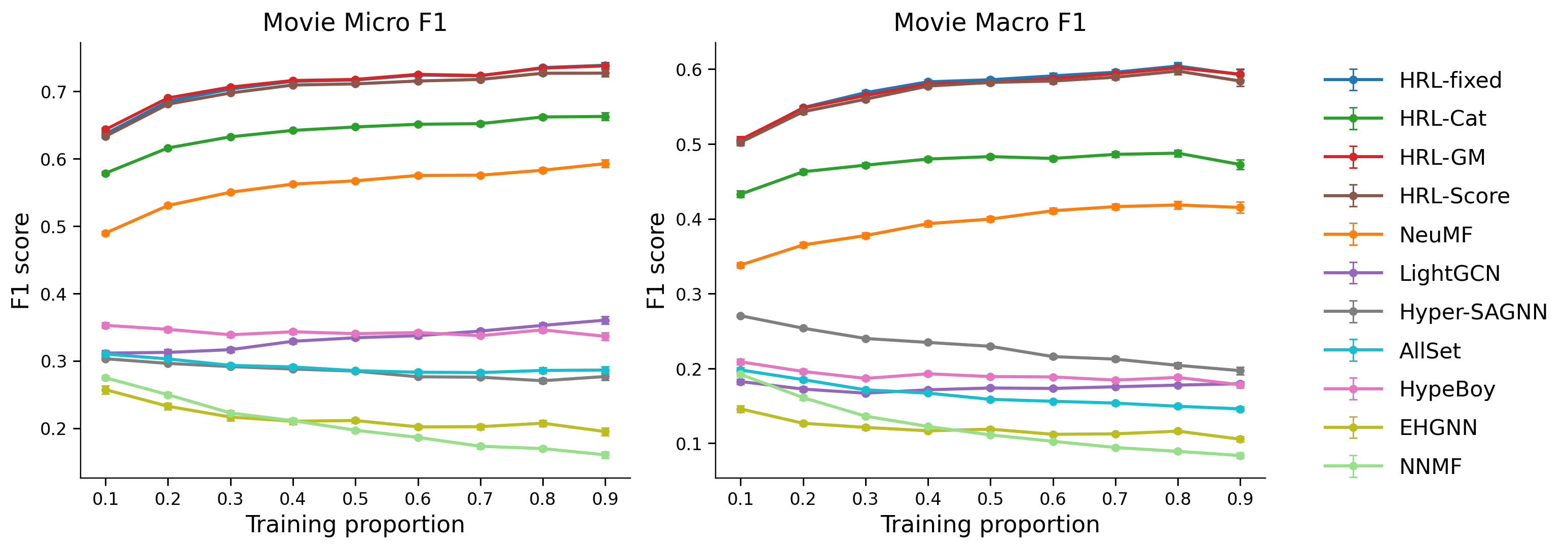}
	\caption{Performance of Multi-label classification on MovieLens datasets: one embedding and 100 training/testing splits}
    \label{fig:movie1}
\end{figure}

Then in the second experimental setting, in order to check the embedding-stability protocol, each hypergraph embedding method was refit five independent times, varying method-specific randomness such as initialization, neural-network training, or diffusion sampling. For each fitted embedding and each training proportion from 0.1 to 0.9, we generate 20 shared training/testing splits. This result is shown in Figure~\ref{fig:movie_small} in the main text.

Across both settings, the HRL-based methods achieve strong classification performance. In particular, HRL-Fixed, HRL-GM, and HRL-Score form the leading group in most cases, with HRL-Fixed often giving a small improvement over the other HRL variants. The advantage of HRL-Fixed is modest, but it is consistent with the cross-validation results. For both the sub-hypergraph and the full hypergraph, a relatively large $K$ is selected, which indicates that the movie preference hypergraph may contain a relatively complicated user-level preference patterns. Since HRL-Fixed can be viewed as a highly flexible special case of HRL-Cat, where each hyperedge is allowed to have its own latent edge effect, it may better capture such heterogeneous user preferences. We note that the above comparison follows the model-selection protocol in Appendix~\ref{append:select}. For HRL-Score, as well as for all baseline methods, hyperparameters are tuned during the representation-learning stage using the missing-value prediction criterion, rather than the downstream genre labels. This is intentional, since we are conducting the multi-label classification task as an external evaluation of the learned node embeddings. It is conceivable that for HRL-Score, its classification performance could potentially improve if its hyperparameters were selected directly to maximize multi-label classification accuracy. 

\subsection{Data and Asset Licenses.}\label{append:licenses}
All real-world datasets used in this paper are existing datasets obtained from their
original public or credentialed sources. We use these datasets only for academic
research evaluation. We do not redistribute the raw datasets.

\begin{table}[t]
\centering
\small
\caption{Sources, versions, and license or usage terms of the real-world datasets.}
\label{tab:data-licenses}
\begin{tabular}{
  >{\RaggedRight\arraybackslash}p{0.15\linewidth}
  >{\RaggedRight\arraybackslash}p{0.25\linewidth}
  >{\RaggedRight\arraybackslash}p{0.27\linewidth}
  >{\RaggedRight\arraybackslash}p{0.20\linewidth}
}
\toprule
Dataset & Source and version & License / terms of use & Use in this paper \\
\midrule
MIMIC-III
&
MIMIC-III Clinical Database, version 1.4, hosted on PhysioNet. We cite both the
PhysioNet database record and the original Scientific Data publication.
&
PhysioNet Credentialed Health Data License 1.5.0 and Credentialed Health Data Use
Agreement 1.5.0. Access is restricted to credentialed users who complete the required
training and sign the data use agreement.
&
Used to construct an ICD-9 code co-occurrence hypergraph. We only use it for method evaluation, we do not redistribute the raw data.
\\

Epicurious recipes
&
Kaggle dataset ``Epicurious -- Recipes with Rating and Nutrition'' by Hugo Darwood;
downloaded from \url{https://www.kaggle.com/datasets/hugodarwood/epirecipes/data}.
&
The Kaggle page lists the license as unknown. We therefore treat the dataset as a
third-party asset with unspecified licensing terms, use it only for academic evaluation,
and do not redistribute the raw dataset.
&
Used to construct a recipe--ingredient hypergraph. We only use it to evaluate the proposed method. We do not re-distribute the raw data.
\\
MovieLens 1M
&
GroupLens MovieLens 1M Dataset, released 02/2003; downloaded from
\url{https://grouplens.org/datasets/movielens/1m/}.
&
GroupLens usage terms as stated in the MovieLens 1M README. The dataset may be used
for research with acknowledgment, but may not be redistributed or used for commercial
or revenue-bearing purposes without permission.
&
Used to construct the movie-preference hypergraph, used for model performance evaluation. We do not redistribute the raw data.
\\
\bottomrule
\end{tabular}
\end{table}

\subsection{Direct HRL Comparisons and Computational Scaling}
\label{app:comparison}
In this subsection, we compare the four HRL specifications through parameter counts,
computational scaling, and cold-start prediction, and summarize their
practical tradeoffs.

\paragraph{Parameterization and computational cost.}
Let $m$, $n$, $d$, and $K$ denote the numbers of training hyperlinks,
nodes, latent dimensions, and mixture components, respectively, and let
$s_j=|e_j|$ be hyperlink cardinality. All methods operate on the
$m\times n$ incidence matrix without enumerating $s_j$-node subsets or
constructing order-$s_j$ tensors.

Table~\ref{tab:complexity} separates retained parameters, counted before
normalization constraints, from fitting memory. HRL-Cat and HRL-GM replace
the $md$ hyperlink coordinates of HRL-Fixed with $Kd$ component coordinates
and $K-1$ weights. HRL-GM fixes its common covariance. HRL-Cat factorizes an unrestricted $K\times n$
probability table after fitting, so its retained count differs from the
number of coordinates optimized during fitting.

For HRL-Score, $B_{\mathrm{MC}}$ denotes the number of prior particles and
$C_{\mathrm{score}}$ the total network-training and prior-sampling cost
per outer epoch. With hidden width $H$ and time-feature dimension $T$,
the time encoder $T\to H\to H$ and score network
$(d+H)\to H\to H\to d$ contain
\[
 P_{\btheta}=3H^2+(T+2d+4)H+d
\]
parameters, including biases. For example, on MIMIC-III, $n=275$, $d=6$, $T=64$, and $H=26$
give $P_{\btheta}=4114$ and $n(d+1)+P_{\btheta}=6039$ retained parameters. Posterior moments, importance weights, optimizer
states, and exponential-moving-average copies contribute to memory but are
not additional statistical parameters.

\begin{table}[!htbp]
\centering
\caption{Retained representation size and conservative dense-operation counts for one fitting iteration or epoch. GM uses a fixed common covariance. Working arrays exclude the common incidence matrix; initialization and final normalization/factorization are excluded from per-iteration work.}
\label{tab:complexity}
\footnotesize
\setlength{\tabcolsep}{3pt}
\renewcommand{\arraystretch}{1.35}
\resizebox{\linewidth}{!}{%
\begin{tabular}{@{}l *{3}{>{\(\scriptstyle}c<{\)}}@{}}
\toprule
Method
& \multicolumn{1}{c}{Retained parameters}
& \multicolumn{1}{c}{Dense work}
& \multicolumn{1}{c@{}}{Principal working arrays} \\
\midrule
HRL-Cat
& n(d+1)+K(d+1)-1
& O(mnK)
& O(mK+Kn+nd) \\
HRL-GM
& n(d+1)+K(d+1)-1
& O(mnKd^2+(mK+n)d^3)
& O(mnK+mKd^2+nd) \\
HRL-Score
& n(d+1)+P_{\btheta}
& O(mnd^2+mB_{\mathrm{MC}}d^2+(m+n)d^3+C_{\mathrm{score}})
& O(mn+mB_{\mathrm{MC}}+nd+(m+B_{\mathrm{MC}})d^2+P_{\btheta}) \\
HRL-Fixed
& n(d+1)+md
& O(mnd)
& O(nd+md) \\
\bottomrule
\end{tabular}%
}
\end{table}
\paragraph{Scaling with the number of nodes.} In order to show the impact of node number in a more direct way, we measured scaling under the mixture-Gaussian and mixture-$t$
simulation designs from the main text.
We fix $m=3000$ and vary $n\in\{90,180,360\}$ under the Gaussian-mixture
and heavy-tailed $t_4$-mixture designs. Tables~\ref{tab:cost_gaussian}
and~\ref{tab:cost_t} report normalized, sign-aligned representation error
$\mathrm{Err}_Z=\|\widehat{\bZ}-\bZ^\star\|_F/\sqrt{nd}$ and normalized
mutual information (NMI) for hyperlink types, while computational cost are reported in Table \ref{tab:cost_gaussian} and Table \ref{tab:cost_t}.

\begin{table}[!htbp]
\centering
\small
\setlength{\tabcolsep}{4pt}
\renewcommand{\arraystretch}{1.08}
\caption{Gaussian-mixture design with $m=3000$: representation recovery and computational cost. Accuracy values are means across simulation seeds. Bold and underlined entries mark the best and second-best reported accuracy values within each node count.}
\label{tab:cost_gaussian}
\textit{(a) Representation recovery}
\par\smallskip
\begin{tabular}{@{}l*{6}{c}@{}}
\toprule
& \multicolumn{2}{c}{$n=90$} & \multicolumn{2}{c}{$n=180$} & \multicolumn{2}{c}{$n=360$} \\
\cmidrule(lr){2-3}\cmidrule(lr){4-5}\cmidrule(lr){6-7}
Method & $\mathrm{Err}_Z\downarrow$ & NMI$\uparrow$ & $\mathrm{Err}_Z\downarrow$ & NMI$\uparrow$ & $\mathrm{Err}_Z\downarrow$ & NMI$\uparrow$ \\
\midrule
HRL-Cat & 0.0808 & 0.8824 & 0.0783 & 0.9482 & 0.0784 & 0.9744 \\
HRL-GM & \textbf{0.0612} & \textbf{0.8984} & \textbf{0.0530} & \textbf{0.9573} & \textbf{0.0518} & \underline{0.9818} \\
HRL-Score & \underline{0.0672} & 0.8878 & \underline{0.0689} & 0.9476 & 0.0894 & 0.9694 \\
HRL-Fixed & 0.1566 & \underline{0.8976} & 0.0810 & \underline{0.9569} & \underline{0.0606} & \textbf{0.9819} \\
\bottomrule
\end{tabular}
\par\medskip
\textit{(b) Computational cost}
\par\smallskip
\resizebox{\linewidth}{!}{%
\begin{tabular}{@{}l*{9}{c}@{}}
\toprule
& \multicolumn{3}{c}{$n=90$} & \multicolumn{3}{c}{$n=180$} & \multicolumn{3}{c}{$n=360$} \\
\cmidrule(lr){2-4}\cmidrule(lr){5-7}\cmidrule(lr){8-10}
Method & Time (s) & s/iter & RSS (MB) & Time (s) & s/iter & RSS (MB) & Time (s) & s/iter & RSS (MB) \\
\midrule
HRL-Cat & 20.33 & 0.1874 & 507.6 & 10.26 & 0.2919 & 505.0 & 12.29 & 0.3368 & 638.9 \\
HRL-GM & 396.79 & 0.9681 & 1541.4 & 458.67 & 1.3015 & 2208.7 & 1059.54 & 2.1191 & 2030.2 \\
HRL-Score & 54.10 & 0.0902 & 532.4 & 76.00 & 0.0950 & 660.2 & 146.40 & 0.1220 & 485.7 \\
HRL-Fixed & 12.84 & 0.0128 & 469.0 & 25.04 & 0.0250 & 625.7 & 50.35 & 0.0504 & 438.9 \\
\bottomrule
\end{tabular}%
}
\par\smallskip
\begin{minipage}{\linewidth}\footnotesize
\textit{Note:} Time is total fitting time; s/iter is time per iteration or epoch; RSS is peak resident-set size. HRL-Cat timings use separate sequential reruns.
\end{minipage}
\end{table}

\begin{table}[t]
\centering
\small
\setlength{\tabcolsep}{4pt}
\renewcommand{\arraystretch}{1.08}
\caption{Heavy-tailed $t_4$-mixture design with $m=3000$: representation recovery and computational cost. Accuracy values are means across simulation seeds. Bold and underlined entries mark the best and second-best reported accuracy values within each node count.}
\label{tab:cost_t}
\textit{(a) Representation recovery}
\par\smallskip
\begin{tabular}{@{}l*{6}{c}@{}}
\toprule
& \multicolumn{2}{c}{$n=90$} & \multicolumn{2}{c}{$n=180$} & \multicolumn{2}{c}{$n=360$} \\
\cmidrule(lr){2-3}\cmidrule(lr){4-5}\cmidrule(lr){6-7}
Method & $\mathrm{Err}_Z\downarrow$ & NMI$\uparrow$ & $\mathrm{Err}_Z\downarrow$ & NMI$\uparrow$ & $\mathrm{Err}_Z\downarrow$ & NMI$\uparrow$ \\
\midrule
HRL-Cat & 0.4124 & \underline{0.4143} & 0.4152 & 0.4330 & 0.4042 & 0.4396 \\
HRL-GM & 0.3930 & 0.4141 & 0.3818 & 0.4398 & \underline{0.3759} & 0.4042 \\
HRL-Score & \textbf{0.0702} & \textbf{0.8769} & \textbf{0.0737} & \textbf{0.9321} & \textbf{0.0862} & \textbf{0.9486} \\
HRL-Fixed & \underline{0.3899} & 0.3920 & \underline{0.3813} & \underline{0.4523} & 0.3771 & \underline{0.4722} \\
\bottomrule
\end{tabular}
\par\medskip
\textit{(b) Computational cost}
\par\smallskip
\resizebox{\linewidth}{!}{%
\begin{tabular}{@{}l*{9}{c}@{}}
\toprule
& \multicolumn{3}{c}{$n=90$} & \multicolumn{3}{c}{$n=180$} & \multicolumn{3}{c}{$n=360$} \\
\cmidrule(lr){2-4}\cmidrule(lr){5-7}\cmidrule(lr){8-10}
Method & Time (s) & s/iter & RSS (MB) & Time (s) & s/iter & RSS (MB) & Time (s) & s/iter & RSS (MB) \\
\midrule
HRL-Cat & 24.75 & 0.1796 & 494.6 & 37.84 & 0.2488 & 580.5 & 62.67 & 0.3419 & 657.5 \\
HRL-GM & 462.45 & 0.9679 & 1752.1 & 647.67 & 1.5237 & 1525.0 & 848.64 & 1.9332 & 2443.9 \\
HRL-Score & 53.00 & 0.0757 & 527.7 & 88.70 & 0.1109 & 660.3 & 186.10 & 0.1329 & 538.1 \\
HRL-Fixed & 13.36 & 0.0134 & 477.5 & 25.34 & 0.0253 & 629.6 & 52.40 & 0.0524 & 428.7 \\
\bottomrule
\end{tabular}%
}
\par\smallskip
\begin{minipage}{\linewidth}\footnotesize
\textit{Note:} Time is total fitting time; s/iter is time per iteration or epoch; RSS is peak resident-set size. HRL-Cat timings use separate sequential reruns.
\end{minipage}
\end{table}

As shown in the tables, per-iteration time increases with $n$, whereas total time also depends
on convergence. HRL-GM achieves the lowest Gaussian-design error at each
node count, while HRL-Cat and HRL-Fixed fit faster. Under the heavy-tailed mixture-$t$ design,
HRL-Score is the only method that maintains low representation error and
high NMI, illustrating a case where flexibility in $P_F$ is worth the additional computation.
HRL-GM has the highest peak memory because it stores variational quantities indexed by hyperlink, node, and component.

\paragraph{Cold-start performance under parameter budgets.} Apart from simulation work, we use the cold-start experiment to study the impact of hyperlink size. In order to change node size, we vary the global training fraction from 0.6 to 0.8, keeping validation fraction 0.1 and $r_{\rm obs}=0.2$. With $d=6$ and $K=7$, HRL-Score uses the
largest integer width satisfying $P_{\btheta}<md$, giving
$H=23$, $26$, and $28$. Its hyperparameters are selected by validation
AUPRC on the first repetition at each fraction and reused thereafter. The results are shown in Table \ref{tab:train-sweep-revised}.

\begin{table}[tbp]
\centering
\small
\setlength{\tabcolsep}{4pt}
\caption{MIMIC-III row-disjoint cold-start prediction at $r_{\mathrm{obs}}=0.2$ as the global-training portion varies. Values are mean (standard error).}
\label{tab:train-sweep-revised}
\begin{tabular}{clrcccc}
\toprule
Train & Method & Params & AUPRC $\uparrow$ & BCE $\downarrow$ & Time (s) & RSS (MB) \\
\midrule
0.6 & HRL-Cat & 1,973 & 0.391 (0.0033) & 0.336 (0.0019) & 6.7 (1.2) & 388.0 (12.1) \\
 & HRL-GM & 1,973 & \textbf{0.434} (0.0025) & \textbf{0.315} (0.0007) & 366.1 (13.7) & 1278.5 (70.2) \\
 & HRL-Score & 5,358 & \underline{0.429} (0.0035) & \underline{0.316} (0.0018) & 29.9 (1.6) & 530.0 (6.8) \\
 & HRL-Fixed & 5,525 & 0.332 (0.0075) & 0.514 (0.0068) & 45.0 (1.4) & 492.3 (7.3) \\
\midrule
0.7 & HRL-Cat & 1,973 & 0.396 (0.0041) & 0.334 (0.0020) & 13.6 (2.7) & 402.6 (13.4) \\
 & HRL-GM & 1,973 & \underline{0.437} (0.0038) & \underline{0.315} (0.0011) & 332.7 (32.3) & 1408.2 (62.6) \\
 & HRL-Score & 6,039 & \textbf{0.437} (0.0034) & \textbf{0.314} (0.0012) & 97.7 (6.4) & 504.0 (35.7) \\
 & HRL-Fixed & 6,125 & 0.337 (0.0069) & 0.506 (0.0115) & 44.1 (1.1) & 574.3 (9.5) \\
\midrule
0.8 & HRL-Cat & 1,973 & 0.402 (0.0031) & 0.331 (0.0016) & 12.5 (1.3) & 401.2 (7.5) \\
 & HRL-GM & 1,973 & \textbf{0.441} (0.0032) & \textbf{0.313} (0.0012) & 457.5 (37.2) & 1385.2 (37.0) \\
 & HRL-Score & 6,523 & \underline{0.435} (0.0031) & \underline{0.314} (0.0010) & 46.2 (1.3) & 566.3 (8.0) \\
 & HRL-Fixed & 6,725 & 0.341 (0.0033) & 0.519 (0.0094) & 53.4 (0.8) & 560.4 (8.3) \\
\bottomrule
\end{tabular}
\end{table}
The conclusions are consistent across training portions. HRL-GM and
HRL-Score achieve the highest AUPRC and lowest BCE despite fewer retained
parameters. Even HRL-Cat improves on HRL-Fixed, supporting the value of
learning a population distribution for cold-start prediction. HRL-Cat
and HRL-GM share the smallest count (1973), with HRL-Cat fitting fastest. We also find that HRL-Score approaches HRL-GM accuracy with substantially lower peak memory.

\paragraph{Practical choice and fitting stability.} We summarize the resulting tradeoffs in Table~\ref{tab:model-guide}.

\begin{table}[H]
\centering
\small
\setlength{\tabcolsep}{4pt}
\renewcommand{\arraystretch}{1.15}
\caption{Practical considerations for choosing an HRL specification, informed by the reported accuracy, cost, and stability results.}
\label{tab:model-guide}
\begin{tabularx}{\linewidth}{@{}l>{\raggedright\arraybackslash}X>{\raggedright\arraybackslash}X>{\raggedright\arraybackslash}X@{}}
\toprule
Method & Random-effect structure & Advantage & Disadvantages \\
\midrule
HRL-Cat
& Discrete hyperlink types
& Fast fitting and a compact, interpretable representation
& Restrictive with within-type variation \\
HRL-GM
& Mixture Gaussian
& Continuous/moderately multimodal heterogeneity; a strong default	
& High working memory \\
HRL-Score
& Flexible learned distribution $\PP_{\bF}$
& Admits non-Gaussian flexibility
& Needs structured initialization and enough particles\\
HRL-Fixed
& Unrestricted effects for training hyperlinks
& Dense/transductive rows or no pooling benefit
& Parameters grow with 
; no prior for new rows\\
\bottomrule
\end{tabularx}
\end{table}
We recommend selecting $d,K$ and the model family by row-disjoint validation BCE/Brier, then applying a one-standard-error rule to choose the simplest statistically indistinguishable model. For HRL-Score, we additionally recommend particle-size stability checks and USVT-based initialization. On current evidence, HRL-GM is a strong default, while HRL-Score is most valuable when non-Gaussian flexibility is empirically supported.

\subsection{Detailed Introduction of Baseline Models}\label{sec:baseline-details}
\subsubsection{MultiVAE.}
MultiVAE~\citep{liang2018variational} was originally proposed for collaborative
filtering with implicit feedback. In the original recommendation setting, one
observes a sparse interaction vector for each user, recording whether the user
watched, clicked, listened to, or purchased each item. The goal is to rank
unseen items according to the user's latent preference. In our hypergraph setting, we adapt MultiVAE by treating each hyperlink as a ``user'' and each node as an ``item''.

MultiVAE assumes that each interaction vector is generated from a latent
Gaussian representation through a multinomial likelihood. Specifically, for the
$j$-th hyperlink, let
\[
    \bh_j \sim N(0,I_K),
\]
where $K$ is the latent dimension. Conditional on $\bh_j$, the incidence vector
$\bv_j$ is modeled as
\[
    \bv_j \sim \mathrm{Mult}(N_j,\pi(\bh_j)),
\]
where $N_j=\sum_{i=1}^n v_{ji}$ is the size of hyperlink $j$, and
$\pi(\bh_j)\in\mathbb{R}^n$ is a probability vector over the $n$ nodes. The
probability vector is obtained by first passing $\bh_j$ through a neural decoder
$f_\theta(\cdot)$ and then applying the softmax function:
\[
    \pi(\bh_j)
    =
    \mathrm{softmax}\{f_\theta(\bh_j)\},
    \qquad
    f_\theta(\bh_j)\in\mathbb{R}^n .
\]
Since the posterior distribution $p(\bh_j\mid \bv_j)$ is intractable, MultiVAE
introduces an encoder network
\[
    q_\phi(\bh_j\mid \bv_j)
    =
    N\!\left(
        \mu_\phi(\bv_j),
        \operatorname{diag}\{\sigma_\phi^2(\bv_j)\}
    \right),
\]
and trains the encoder and decoder by maximizing the variational objective
\[
    \mathcal{L}_\beta(\bv_j;\theta,\phi)
    =
    \mathbb{E}_{q_\phi(\bh_j\mid \bv_j)}
    \left[
        \log p_\theta(\bv_j\mid \bh_j)
    \right]
    -
    \beta\,
    \mathrm{KL}\!\left(
        q_\phi(\bh_j\mid \bv_j)\,\|\,p(\bh_j)
    \right),
\]
Although MultiVAE provides a flexible nonlinear baseline for sparse
implicit-feedback data, its modeling assumptions are less tailored to general
hypergraph incidence data. The multinomial likelihood forces all nodes to compete for a fixed amount of probability mass. This is natural for
recommendation, where the objective is to rank items for a user, but it is less natural for hypergraphs, where multiple nodes jointly form a hyperlink and the
hyperlink size $N_j$ is itself an important structural feature. Apart from that, MultiVAE
learns hyperlink-level latent variables through an encoder--decoder
architecture, but it does not explicitly assign each node an interpretable
embedding. As a result, node-level heterogeneity and hyperlink--node
compatibility are only implicitly absorbed into the neural decoder. This limits
interpretability and makes the learned representation less directly useful for
downstream node-level tasks such as node clustering or node classification.

\subsubsection{NeuMF}\label{sec:NeuMF}
NeuMF~\citep{he2017neural} is also a neural collaborative filtering method originally proposed for implicit-feedback recommendation. Its main idea is to combine the strengths of Generalized Matrix Factorization (GMF), which captures linear latent interactions, and a multi-layer perceptron (MLP), which captures nonlinear user--item interaction patterns.

The GMF component models the interaction between hyperlink $j$ and node $i$ through the elementwise product of their embeddings $
    \bphi^{\mathrm{GMF}}_{ji}
    =
    \bff_j^{G} \odot \bz_i^{G}$, while the MLP component uses a separate set of hyperlink and node embeddings. It first concatenates the two embeddings and then passes the concatenated vector through multiple nonlinear layers:
\[
    \bt_1 =
    \begin{bmatrix}
        \bff_j^M \\
        \bz_i^M
    \end{bmatrix},
    \qquad
    \bt_{\ell}
    =
    a_{\ell}\!\left(\bW_{\ell}^{\top}\bt_{\ell-1}+\bb_{\ell}\right),
    \quad \ell=2,\ldots,L,
\]
where $\bW_\ell$ and $\bb_\ell$ are neural-network parameters, and $a_\ell(\cdot)$ is a nonlinear activation function. 

NeuMF then fuses the GMF and MLP representations and maps them to an incidence probability through a logistic link, i.e.
\[
    \widehat v_{ji}
    =
    \sigma\!\left(
        \bh^\top
        \begin{bmatrix}
            \bphi^{\mathrm{GMF}}_{ji} \\
            \bphi^{\mathrm{MLP}}_{ji}
        \end{bmatrix}
    \right),
\]
The model is trained by minimizing the binary cross-entropy loss
\[
    \mathcal{L}_{\mathrm{BCE}}
    =
    -
    \sum_{j=1}^{m}
    \sum_{i=1}^{n}
    \left\{
        v_{ji}\log \widehat v_{ji}
        +
        (1-v_{ji})\log(1-\widehat v_{ji})
    \right\},
\]

In our comparison tasks, we use the concatenation of the learned GMF embedding and MLP embedding as final embeddings.

\subsubsection{LightGCN}
LightGCN~\citep{he2020lightgcn} was originally proposed for graph-based collaborative filtering. 
It simplifies graph convolutional networks by removing feature transformations and nonlinear activations, retaining only the core operation of neighborhood aggregation. 
In our hypergraph setting, we adapt LightGCN by converting the incidence matrix into a bipartite incidence graph, whose two node types are hyperlinks and nodes. An edge is placed between hyperlink $j$ and node $i$ if and only if node $i$ belongs to hyperlink $j$.

Let $\be_j^{(0)}$ and $\be_i^{(0)}$ denote the initial embeddings of hyperlink $j$ and node $i$, respectively. LightGCN leverages the edge-node interaction graph to propagate embeddings through normalized neighborhood aggregation shown below
\[
    \be_j^{(k+1)}
    =
    \sum_{i\in \mathcal{N}_j}
    \frac{1}
    {\sqrt{|\mathcal{N}_j|}\sqrt{|\mathcal{N}_i|}}
    \be_i^{(k)},
\]
\[
    \be_i^{(k+1)}
    =
    \sum_{j\in \mathcal{N}_i}
    \frac{1}
    {\sqrt{|\mathcal{N}_i|}\sqrt{|\mathcal{N}_j|}}
    \be_j^{(k)}.
\]
Here $\mathcal{N}_j$ denotes the set of nodes contained in hyperlink $j$, and $\mathcal{N}_i$ denotes the set of hyperlinks containing node $i$. After $K$ propagation layers, LightGCN combines embeddings from different layers to obtain the final representations
\[
    \be_j = \sum_{k=0}^{K}\alpha_k \be_j^{(k)},
    \qquad
    \be_i = \sum_{k=0}^{K}\alpha_k \be_i^{(k)}.
\]
The incidence score between hyperlink $j$ and node $i$ is then computed by an inner product,
$\widehat y_{ji} = \be_j^\top \be_i$, and the model is trained with Bayesian Personalized Ranking (BPR) loss:
\[
    \mathcal{L}_{\mathrm{BPR}}
    =
    -
    \sum_{j}
    \sum_{i\in \mathcal{N}_j}
    \sum_{\ell\notin \mathcal{N}_j}
    \log \sigma(\widehat y_{ji}-\widehat y_{j\ell})
    +
    \lambda \|\bE^{(0)}\|_2^2.
\]
which encourages observed hyperlink--node incidences to receive higher scores than unobserved ones.

\subsubsection{EHGNN}
EHGNN~\citep{jo2021edge} originally focus on a bunch of graphs, where each graph can be viewed as a tuple containing  $\bG_i=(\bE_i,\bM_i,\bX_i)$. Here $\bE_i$ represents observed edge feature, $\bX_i$ represents observed node feature and $\bM_i$ is the incidence 0-1 matrix. Their method is motivated mainly by the molecular setting, where each molecular is a graph, $X_i$ is the atom type of atom $i$, $E_j$ is the bond type of bond $j$.

In our case, only one hypergraph is available, and there are neither edge features nor node features. In all the tasks, we adopt the main framework of EHGNN, first build the dual hypergraph $\bH^*=(\epsilon,\bB',\bV)$, and then use an EHGNN-style encoder to learn hyperlink embeddings:$$\bff_j=\text{EHGNN}(j)$$
And we use a node encoder on the original hypergraph to learn node embeddings:$$\bz_i=\text{NodeEncoder}(i)$$
Here we use the hyperlink degrees as the one-dimensional edge attribute and use the number of hyperlinks containing node $i$ as one-dimensional node attribute.

To train the network, we add an incidence decoder linking $z_j,h_i$ to $B_{ij}$, e.g. $p_{ij}=\sigma(\bh_i^{\top}\bW\bz_j)$ and using  binary cross-entropy loss.

\subsubsection{NNMF}
Neural Network Matrix Factorization (NNMF) was proposed for matrix completion
and relational data modeling~\citep{dziugaite2015neural}. The original problem
is to predict missing entries in a partially observed matrix, such as user--movie
ratings in collaborative filtering or links in relational arrays. Classical
matrix factorization assumes that the entry $X_{nm}$ can be approximated by the
inner product of row and column latent features:
\[
    \widehat X_{nm} = U_n^\top V_m.
\]
NNMF generalizes this idea by replacing the fixed inner product with a learned
neural-network function.

Specifically, NNMF assigns latent feature vectors to each row and each column.
Let $\bU_n$ and $\bV_m$ denote row and column latent feature vectors. The model may
also include additional latent feature matrices $\bU'_n$ and $\bV'_m$, whose columns
are combined through elementwise products. The predicted entry is modeled as
\[
    \widehat X_{nm}
    =
    f_\theta\!\left(
        \bU_n,
        \bV_m,
        \bU'_{n,1}\circ \bV'_{m,1},
        \ldots,
        \bU'_{n,D'}\circ \bV'_{m,D'}
    \right),
\]
where $f_\theta$ is a feed-forward neural network and $\circ$ denotes the
elementwise product. The parameters are estimated by minimizing a regularized
reconstruction loss over observed entries:
\[
    \sum_{(n,m)\in \mathcal{J}}
    \left(
        X_{nm}-\widehat X_{nm}
    \right)^2
    +
    \lambda
    \left(
        \sum_n \|\bU_n\|_2^2
        +
        \sum_m \|\bV_m\|_2^2
        +
        \sum_n \|\bU'_n\|_F^2
        +
        \sum_m \|\bV'_m\|_F^2
    \right),
\]
where $\mathcal{J}$ is the set of observed entries. Training alternates between
optimizing the neural-network parameters with latent features fixed and
optimizing the latent features with the network fixed.

The basic idea of NNMF is to preserve the latent-feature representation of
matrix factorization while learning a flexible nonlinear function from row and
column embeddings to matrix entries. In our hypergraph setting, the incidence
matrix $Y$ is treated as a partially observed relational matrix, with hyperlinks
as rows and nodes as columns. NNMF then predicts whether node $i$ belongs to
hyperlink $j$. This gives a flexible nonlinear matrix-completion baseline, but
it does not explicitly exploit the combinatorial nature of hyperlinks or model
hyperlink-level sparsity.
\subsubsection{HRL-Fixed}
As we have mentioned in the main text, HRL-Fixed~\citet{wu2024general} can be viewed as a special case of HRL-Cat. Apart from node embeddings $\{\bz_i\}$, hyperlink embeddings $\{\bff_i\}$ and the node degree intercepts $\{\alpha_j\}$, HRL-Fixed additionally consider an order-adjusting parameter $\beta_{m,n}$ to control the overall sparsity and typical order of hyperlinks. Conditional on the embeddings and
parameters, HRL-Fixed assumes 
\[
    v_{ji} \mid \bff_j,\bz_i,\alpha_i,\beta_{m,n}
    \sim
    \mathrm{Bernoulli}(p_{ji}),
\]
where
\[
    p_{ji}
    =
    \sigma\!\left(
        \beta_{m,n}
        +
        \alpha_i
        +
        \bff_j^\top \bz_i
    \right).
\]

The model parameters are estimated by maximizing the likelihood under
identifiability constraints, which is achieved by adopting the projected gradient ascent algorithm.

\subsubsection{AllSet}\label{Allset}
The central idea of AllSet~\citep{chien2022allset} is to formulate hypergraph layer as a composition of two learnable multiset functions. One aggregates node representations into hyperlink representations, and the other propagates hyperlink representations back to nodes. 

Let \(\bX\in\mathbb R^{n\times d_0}\) denote the input node-feature matrix. (In our application case, we are not considering using any external node attributes. We thus add an additional encoder layer, treating node-ID as input, and the resulting $\bX\in\RR^{n,d_0}$ learned embedding as node-feature matrix. Equivalently, this can be viewed as applying a learnable linear embedding layer to the one-hot node identifiers $\bE_{\mathrm{ID}} = \bI_n \bW_{\mathrm{ID}},
 \bW_{\mathrm{ID}}\in\mathbb R^{n\times d_0}$. And we tune $d_0\in\{8,16,32,64\}$ as an additional hyper-parameter.)

With the node-feature matrix, starting from \(\bh_i^{(0)}=\bx_i\), an AllSet layer first constructs hyperlink representations through a node-to-hyperlink multiset function,
\[
    \bu_j^{(\ell)}
    =
    \phi_{V\to E}^{(\ell)}
    \left(
    \{\bh_i^{(\ell)}: i\in e_j\}
    \right),
\]
and then updates node representations through a hyperlink-to-node multiset function,
\[
    \bh_i^{(\ell+1)}
    =
    \phi_{E\to V}^{(\ell)}
    \left(
    \{\bu_j^{(\ell)}: j\in \mathcal E_i\}
    \right),
\]
where \(\mathcal E_i=\{j:i\in e_j\}\) is the set of hyperedges containing node \(i\). After \(L\) layers, we use
\[
    \widehat{\bz}_i = \bh_i^{(L)}
\]
are the learned node embedding. And edge-embedding can be obatined by set 
\[
    \widehat{\bff}_j = \bu_j^{(L)},
\]

\subsubsection{HypeBoy}
HypeBoy~\citep{kim2024hypeboy} is a generative self-supervised representation learning method, originally proposed for hyperlink filling task. Given a subset of nodes from a hyperlink, the model predicts the missing node that completes the hyperlink. This task is closely related to our missing-value prediction setting, since both require evaluating whether a candidate node is compatible with a partially observed hyperlink.

However, HypeBoy also requires an existing node feature matrix $\bX$. (In our tasks, we construct $\bX$ in exactly the same way as described in \ref{Allset}.) With $\bX$, HypeBoy first applies a hypergraph neural encoder to obtain node embeddings,
\[
    \bZ=f_\theta(\bX,\mathcal E)
\]

In the original HypeBoy objective, for a hyperlink \(e_j\) and a held-out member \(i\in e_j\), the query set is
\[
    Q_{j,-i}=e_j\setminus\{i\}.
\]
The model represents this query set by summing the node embeddings in the query and passing the result through a set projection head:
\[
    \widehat{\bq}_{j,-i}
    =
    h_{\mathrm{set}}
    \left(
    \sum_{u\in Q_{j,-i}} \widehat{\bz}_u
    \right).
\]
Candidate nodes are also passed through a node projection head \(h_{\mathrm{node}}\). The hyperedge-filling score is then obtained based on the cosine similarity between the projected candidate-node representation and the projected query-set representation, followed by a softmax over candidate nodes.

The final learned representation from HypeBoy is the node embedding \(\widehat{\bz}_i\).
where \(e_j^{\mathrm{train}}\) contains only the observed training members of hyperedge \(j\).

\end{document}